\documentclass[twoside,11pt]{article}

\usepackage{jmlr2e}
\usepackage{amsmath}
\usepackage{xparse}
\usepackage[dvipsnames]{xcolor}

\usepackage{tikz}
\usepackage{bbold}
\usepackage{caption,subcaption}
\usepackage{enumerate}
\usepackage{multirow}
\usepackage{url}
\usepackage{rotating}

\usepackage{afterpage}
\usepackage{adjustbox}
\usepackage{lscape}
\usepackage{cprotect}
\usepackage{pifont} 
\usepackage{float}

\usepackage[ruled,vlined,linesnumbered]{algorithm2e}

\usepackage[hang,flushmargin]{footmisc}

\newcommand{\rot}[1]{\multicolumn{1}{c}{\adjustbox{angle=30,lap=\width-1em}{#1}}}

\makeatletter
\newcommand{\algorithmfootnote}[2][\footnotesize]{%
  \let\old@algocf@finish\@algocf@finish
  \def\@algocf@finish{\old@algocf@finish
    \leavevmode\rlap{\begin{minipage}{\linewidth}
    #1#2
    \end{minipage}}%
  }%
}
\makeatother

\def\xscale{1}
\def\yscale{1}
\def\nodescale{1}
 
 \usetikzlibrary{decorations.markings}
\tikzset{->-/.style={decoration={
  markings,
  mark=at position .5 with {\arrow{>}}},postaction={decorate}}}

\newcommand{\xmark}{\ding{55}}%

\DeclareMathOperator*{\degr}{deg}
\DeclareMathOperator*{\height}{\mathcal{H}}
\DeclareMathOperator*{\roottree}{\mathcal{R}}
\DeclareMathOperator*{\child}{\mathcal{C}}
\DeclareMathOperator*{\leaves}{\mathcal{L}}
\DeclareMathOperator*{\fami}{\mathcal{F}} 
\DeclareMathOperator*{\des}{\mathcal{D}} 
\DeclareMathOperator*{\parents}{\mathcal{P}} 
\DeclareMathOperator*{\subtrees}{\mathcal{S}}
\DeclareMathOperator*{\red}{\mathfrak{R}^\ast}

\renewcommand*{\O}{\mathcal{O}} 
\newcommand{\data}[2][]{\mathcal{X}_\text{\normalfont #2}^\text{#1}} 

\newcommand{\origin}{o} 
\newcommand{\num}{N}
\DeclareMathOperator*{\presence}{\varphi}
\DeclareMathOperator*{\smstp}{f^\ast}
\DeclareMathOperator*{\stree}{\mathbb{T}_\mathcal{F_T}}
\DeclareMathOperator*{\sdag}{\mathbb{D}_\mathcal{F}}
\DeclareMathOperator*{\ntm}{\sigma} 

\usepackage{lastpage}
\jmlrheading{}{}{}{}{}{}{Romain Aza\"is and Florian Ingels}
\ShortHeadings{The Weight Function in the Subtree Kernel is Decisive}{Aza\"is and Ingels}

\firstpageno{1}

\begin{document}

\title{The Weight Function in the Subtree Kernel is Decisive}

\author{\name Romain Aza\"is \email romain.azais@inria.fr \\
       \addr Laboratoire Reproduction et D\'eveloppement des Plantes\\
       Univ Lyon, ENS de Lyon, UCB Lyon 1, CNRS, INRAE, Inria, F-69342, Lyon, France
       \AND
       \name Florian Ingels \email florian.ingels@inria.fr \\
       \addr Laboratoire Reproduction et D\'eveloppement des Plantes\\
       Univ Lyon, ENS de Lyon, UCB Lyon 1, CNRS, INRAE, Inria, F-69342, Lyon, France}

\editor{}

\maketitle

\begin{abstract}
Tree data are ubiquitous because they model a large variety of situations, e.g., the architecture of plants, the secondary structure of RNA, or the hierarchy of XML files. Nevertheless, the analysis of these non-Euclidean data is difficult per se. In this paper, we focus on the subtree kernel that is a convolution kernel for tree data introduced by Vishwanathan and Smola in the early 2000's. More precisely, we investigate the influence of the weight function from a theoretical perspective and in real data applications. We establish on a 2-classes stochastic model that the performance of the subtree kernel is improved when the weight of leaves vanishes, which motivates the definition of a new weight function, learned from the data and not fixed by the user as usually done. To this end, we define a unified framework for computing the subtree kernel from ordered or unordered trees, that is particularly suitable for tuning parameters. We show through eight real data classification problems the great efficiency of our approach, in particular for small data sets, which also states the high importance of the weight function. Finally, a visualization tool of the significant features is derived.
\end{abstract}


\section{Introduction}

\subsection{Analysis of tree data}
Tree data naturally appear in a wide range of scientific fields, from RNA secondary structures in biology \citep{LE1989461} to \verb+XML+ files \citep{10.1007/978-3-540-30116-5_15} in computer science through dendrimers \citep{PhysRevB.65.155116} in chemistry and physics. Consequently, the statistical analysis of tree data is of great interest. Nevertheless, investigating these data is difficult due to the intrinsic non-Euclidean nature of trees.

Several approaches have been considered in the literature to deal with this kind of data: edit distances between unordered or ordered trees \citep[see][and the references therein]{B2005}, coding processes for ordered trees \citep{Shen2014FunctionalDA}, with a special focus on conditioned Galton-Watson trees \citep{AGH19,KPDRV14}. One can also mention the approach developed in \citep{wang2007}. In the present paper, we focus on kernel methods, a complementary family of techniques that are well-adapted to non-Euclidean data.

Kernel methods consists in mapping the original data into a (inner product) feature space. Choosing the proper feature space and finding out the mapping might be very difficult. Furthermore, the curse of dimensionality takes place and the feature space may be extremely big, therefore impossible to use. Fortunately, a wide range of prediction algorithms do not need to access that feature space, but only the inner product between elements of the feature space. Building a function, called a kernel, that simulates an inner product in an implicit feature space, frees us from constructing a mapping. Indeed, $K:\mathcal{X}^2\to\mathbb{R}$ is said to be a kernel function on $\mathcal{X}$ if, for any $(x_1,\dots,x_n)\in\mathcal{X}^n$, the Gram matrix $\left[K(x_i,x_j)\right]_{1\leq i,j\leq n}$ is positive semidefinite. By virtue of Mercer's theorem \citeyearpar{mercer1909xvi}, there exists a (inner product) feature space $\mathcal{Y}$ and a mapping $\varphi:\mathcal{X}\to\mathcal{Y}$ such that, for any $(x,y)\in\mathcal{X}^2$, $K(x,y) = \langle\varphi(x) , \varphi(y)\rangle_{\mathcal{Y}}$. This technique is known as the kernel trick. Algorithms that can use kernels include Support Vector Machines (SVM), Principal Components Analyses (PCA) and many others. We refer the reader to the books \citep{cristianini2000introduction,Scholkopf:2001:LKS:559923,shawe2004kernel} and the references therein for more detailed explanations of theory and applications of kernels.

To use kernel-based algorithms with tree data, one needs to design kernel functions adapted to trees. Convolution kernels, introduced by Haussler \citeyearpar{haussler1999convolution}, measure the similarity between two complex combinatorial objects based on the similarity of their substructures. Based on this strategy, many authors have developed convolution kernels for trees, among them the subset tree kernel \citep{collins2002convolution}, the subtree kernel \citep{vishwanathan2002fast} and the subpath kernel \citep{kimura2011subpath}. A recent state-of-the-art on kernels for trees can be found in the thesis of Da San Martino \citeyearpar{da2009kernel}, as well as original contributions on related topics. In this article, we focus on the subtree kernel as defined by \cite{vishwanathan2002fast}. In this introduction, we develop some concepts on trees in Subsection~\ref{subsec:ordered}. They are required to deal with the precise definition of the subtree kernel in Subsection~\ref{ss:kernel} as well as the aim of the paper presented in Subsection~\ref{ss:aim}.

\subsection{Unordered and ordered rooted trees}\label{subsec:ordered}

\paragraph{Rooted trees} A rooted tree $T$ is a connected graph with no cycle such that there exists a unique vertex $\roottree(T)$, called the root, which has no parent, and any vertex different from the root has exactly one parent. The leaves $\leaves(T)$ are all the vertices without children. The height of a vertex $v$ of a tree $T$ can be recursively defined as $\height(v)=0$ if $v$ is a leaf of $T$ and
$$\height(v)=1+\max_{w\in\child(v)}\height(w)$$
otherwise. The height $\height(T)$ of the tree $T$ is defined as the height of its root, i.e., $\height(T)=\height(\roottree(T))$.
The outdegree of $T$ is the maximal branching factor that can be found in $T$, that is
$$\degr(T)=\max_{v\in T}\#\child(v),$$
where $\child(v)$ denotes the set of children of $v$.
For any vertex $v$ of $T$, the subtree $T[v]$ rooted in $v$ is the tree composed of $v$ and all its descendants $\des(v)$. $\subtrees(T)$ denotes the set of subtrees of $T$.

\paragraph{Unordered trees} Rooted trees are said unordered if the order between the sibling vertices of any vertex is not significant. The precise definition of unordered rooted trees, or simply unordered trees, is obtained from the following equivalence relation: two trees $T_1$ and $T_2$ are isomorphic (as unordered trees) if there exists a one-to-one correspondence $\Phi$ from the set of vertices of $T_1$ into the set of vertices of $T_2$ such that, if $w$ is a child of $v$ in $T_1$, then $\Phi(w)$ is a child of $\Phi(v)$ in $T_2$. The set of unordered trees is the quotient set of rooted trees by this equivalence relation.

\paragraph{Ordered trees} In ordered rooted trees, or simply ordered trees, the set of children of any vertex is ordered. As before, ordered trees can be defined as a quotient set if one adds the concept of order to the equivalence relation: two trees $T_1$ and $T_2$ are isomorphic (as ordered trees) if there exists a one-to-one correspondence $\Phi$ from the set of vertices of $T_1$ into the set of vertices of $T_2$ such that, if $w$ is the $r^{\text{th}}$ child of $v$ in $T_1$, then $\Phi(w)$ is the $r^\text{th}$ child of $\Phi(v)$ in $T_2$.

In the whole paper, $\mathcal{T}^\ast$ denotes the set of $\ast$-trees with $\ast \in \{\text{ordered},\text{unordered}\}$.

\subsection{Subtree kernel}\label{ss:kernel}

The subtree kernel has been introduced by \cite{vishwanathan2002fast} as a convolution kernel on trees for which the similarity between two trees is measured through the similarity of their subtrees. A subtree kernel $K$ on $\ast$-trees is defined as,
\begin{equation}\label{eq:defK}
\forall\,T_1,T_2\in\mathcal{T}^\ast,~K(T_1,T_2) = \sum_{\tau\in\mathcal{T}^\ast} w_\tau\,\kappa\left( \num_{\tau}(T_1) , \num_\tau(T_2) \right) ,
\end{equation}
where $w_\tau$ is the weight associated to $\tau$, $\num_{\tau}(T)$ counts the number of subtrees of $T$ that are isomorphic (as $\ast$-trees) to $\tau$ and $\kappa$ is a kernel function on $\mathbb{N}$, $\mathbb{Z}$ or $\mathbb{R}$ \citep[see][Section 2.3 for some classic examples]{Scholkopf:2001:LKS:559923}. Assuming $\kappa(0,n) = \kappa(n,0) = 0$, the formula $\eqref{eq:defK}$ of $K$ becomes
$$
	K(T_1,T_2) = \sum_{\tau\in\subtrees(T_1)\cap\subtrees(T_2)} w_\tau\,\kappa\left( \num_{\tau}(T_1) , \num_\tau(T_2) \right),
$$
making the sum finite. Indeed, all the subtrees $\tau\in\mathcal{T}^\ast\setminus\subtrees(T_1)\cap\subtrees(T_2)$ do not count in the sum \eqref{eq:defK}. In this paper, as for \cite{vishwanathan2002fast}, we assume that $\kappa(n,m) = nm$, then
\begin{equation}\label{eq:K:subtrees}
	K(T_1,T_2) = \sum_{\tau\in\subtrees(T_1)\cap\subtrees(T_2)} w_\tau\,\num_{\tau}(T_1) \num_\tau(T_2).
\end{equation}
which is the subtree kernel as introduced by \cite{vishwanathan2002fast}.

The weight function $\tau\mapsto w_\tau$ is the only parameter to be tuned. In the literature, the weight is always assumed to be a function of a quantity measuring the ``size'' of $\tau$, in particular its height $\height(\tau)$. Then $w_\tau$ is taken as an exponential decay of this quantity, $w_\tau = \lambda^{\height(\tau)}$ for some $\lambda\in[0,1]$ \citep{aiolli2006fast,collins2002convolution,da2009kernel,kimura2011subpath,vishwanathan2002fast}. This choice can be justified in the following manner. If a subtree $\tau$ is counted in the kernel, then all its subtrees are also counted. Then an exponential decay counterbalances the exponential growth of the number of subtrees.

In the literature, two algorithms have been proposed to compute the subtree kernel for ordered trees. The approach of \citep{vishwanathan2002fast} is based on string representations of trees, while the authors of \citep{aiolli2006fast, da2009kernel} extensively use DAG reduction of tree data, an algorithm that achieves lossly compression of trees. To the best of our knowledge, the case of unordered trees has only been considered through the arbitrary choice of a sibling order.

\subsection{Aim of the paper}\label{ss:aim}

The aim of the present paper is threefold:
\begin{enumerate}

\item We investigate the theoretical properties of the subtree kernel on a $2$-classes model of random trees in Section~\ref{sec:theory}. More precisely, we provide a lower-bound for the contrast of the kernel in Proposition~\ref{prop:ker:sep}. Indeed, the higher the contrast, the less data are required to achieve a given performance in prediction \citep[see][for general similarity functions and Corollary~\ref{cor:ker:size} for the subtree kernel]{DBLP:journals/ml/BalcanBS08}. We exploit this result to show in Subsection~\ref{ss:weight:leaves} that the contrast of the subtree kernel is improved if the weight of leaves vanishes. The relevance of the model is discussed in Remark~\ref{rem:model}.

\item We rely on \cite{aiolli2006fast,da2009kernel} on ordered trees to develop in Section~\ref{sec:dag} a unified framework based on DAG reduction for computing the subtree kernel from ordered or unordered trees, with or without labels on their vertices. Subsection~\ref{ss:dag:tree} is devoted to DAG reduction of unordered then ordered trees. DAG reduction of a forest is introduced in Subsection~\ref{ss:dag:forest}. Then, the subtree kernel is computed from the annotated DAG reduction of the data set is Subsection~\ref{ss:dag:data set}. We notice in Remark~\ref{rem:compl:sub} that DAG reduction of the data set is costly but makes possible super-fast repeated computations of the kernel, which is particularly adapted for tuning parameters. This is the main advantage of the DAG computation of the subtree kernel compared to the algorithm based on string representations \citep{vishwanathan2002fast}. Our method allows the implementation of any weighting function, while the recursive computation of the subtree kernel proposed by \cite[Chapter 6]{da2009kernel} also uses DAG reduction of tree data but makes an extensive use of the exponential form of the weight (combining equations (3.12) and (6.2) from \cite{da2009kernel}). We also investigate the theoretical complexities of the different steps of the DAG computation for both ordered and unordered trees (see Proposition~\ref{prop:dag:recompression:complexity} and Remark~\ref{rem:compl:sub}). This kind of question has been tackled in the literature only for ordered trees and from a numerical perspective \citep[Section~4]{aiolli2006fast}.

\item As aforementioned, we show in the context of a stochastic model that the performance of the subtree kernel is improved when the weight of leaves is $0$ (see Section~\ref{sec:theory}). Relying on this (see also Remark~\ref{rem:link:model:discr} on the possible generalization of this result), we define in Section~\ref{sec:weight} a new weight function, called discriminance, that is not a function of the size of the argument as in the literature, but is learned from the data. The learning step of the discriminance weight function strongly relies on the DAG computation of the subtree kernel presented above because it allows the enumeration of all the subtrees composing the data set without redundancies. We explore in Section~\ref{sec:data} the relevance of this new weighting scheme across several data sets, notably on the difficult prediction problem of the language of a Wikipedia article from its structure in Subsection~\ref{ss:wikipedia}. Beyond very good classification results, we show that the methodology developed in the paper can be used to extract the significant features of the problem and provide a visualization at a glance of the data set. In addition, we remark that the average discriminance weight decreases exponentially as a function of the height (except for leaves). Thus, the discriminance weight can be interpreted as the second order of the exponential weight introduced in the literature. Application to real-world data sets in Subsections~\ref{ss:markup}, \ref{ss:biol} and \ref{ss:logml} shows that the discriminance weight is particularly relevant for small databases when the classification problem is rather difficult, as depicted in Fig.\,\ref{fig:improvement}.

\end{enumerate}

Finally, concluding remarks are presented in Section~\ref{s:concl}. Technical proofs have been deferred into Appendices~\ref{app:proof:ker:sep} and \ref{app:proof:dag:recompression}.


\section{Theoretical study}
\label{sec:theory}

In this section, we define a stochastic model of $2$-classes tree data.
From this ideal data set, we prove the efficiency of the subtree kernel and derive the sufficient size of the training data set to get a classifier with a given prediction error.
We also state on this simple model that the weight of leaves should always be $0$.
We emphasize that this study is valid for both ordered and unordered trees.

\subsection{Two trees as different as possible}

Our goal is to build a $2$-classes data set of random trees. To this end, we first define two typical trees $T_0$ and $T_1$ that are as different as possible in terms of subtree kernel.

Let $T_0$ and $T_1$ be two trees that fulfill the following conditions:
\begin{enumerate}
\item $\forall\,i\in\{0,1\}$, $\forall\,u,v\in T_i \setminus \leaves(T_i)$, if $u\neq v$ then $T_i[u]\neq T_i[v]$, i.e, two subtrees of $T_i$ are not isomorphic (except leaves).
\item $\forall\,u\in T_0 \setminus \leaves(T_0)$, $\forall\,v\in T_1 \setminus \leaves(T_1)$, $T_0[u]\neq T_1[v]$, i.e., any subtree of $T_0$ is not isomorphic to a subtree of $T_1$ (except leaves).
\end{enumerate}

These two assumptions ensure that the trees $T_0$ and $T_1$ are as different as possible. Indeed, it is easy to see that
$$ K(T_0,T_1) = w_\bullet \#\leaves(T_0)\#\leaves(T_1),$$
which is the minimal value of the kernel and where $\omega_\bullet$ is the weight of leaves. We refer to Fig.\,\ref{fig:trees:ass12} for an example of trees that satisfy these conditions.

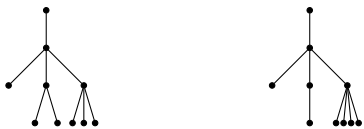
\begin{figure}[h]
\centering
\begin{tikzpicture}[xscale=0.1,yscale=0.1]
\tikzstyle{fleche}=[-,>=latex]
\tikzstyle{noeud}=[fill=black,circle,draw,scale=0.2]
\tikzstyle{etiquette}=[midway]

\def\DistanceInterNiveaux{5}
\def\DistanceInterFeuilles{5}

\def\NiveauA{(-0)*\DistanceInterNiveaux}
\def\NiveauB{(-1)*\DistanceInterNiveaux}
\def\NiveauC{(-2)*\DistanceInterNiveaux}
\def\NiveauD{(-3)*\DistanceInterNiveaux}
\def\InterFeuilles{(1)*\DistanceInterFeuilles}

\node[noeud] (A) at ({(0)*\InterFeuilles},{\NiveauA}) {};
\node[noeud] (B) at ({(0)*\InterFeuilles},{\NiveauB}) {};
\node[noeud] (C) at ({(-1)*\InterFeuilles},{\NiveauC}) {};
\node[noeud] (D) at ({(0)*\InterFeuilles},{\NiveauC}) {};
\node[noeud] (E) at ({(1)*\InterFeuilles},{\NiveauC}) {};
\node[noeud] (F) at ({(-0.3)*\InterFeuilles},{\NiveauD}) {};
\node[noeud] (G) at ({(0.3)*\InterFeuilles},{\NiveauD}) {};
\node[noeud] (H) at ({(0.7)*\InterFeuilles},{\NiveauD}) {};
\node[noeud] (I) at ({(1)*\InterFeuilles},{\NiveauD}) {};
\node[noeud] (J) at ({(1.3)*\InterFeuilles},{\NiveauD}) {};

\draw[fleche] (A)--(B) node[etiquette] {};
\draw[fleche] (B)--(C) node[etiquette] {};
\draw[fleche] (B)--(D) node[etiquette] {};
\draw[fleche] (B)--(E) node[etiquette] {};

\draw[fleche] (D)--(F) node[etiquette] {};
\draw[fleche] (D)--(G) node[etiquette] {};

\draw[fleche] (E)--(H) node[etiquette] {};
\draw[fleche] (E)--(I) node[etiquette] {};
\draw[fleche] (E)--(J) node[etiquette] {};
\end{tikzpicture}
\hspace{2cm}
\begin{tikzpicture}[xscale=0.1,yscale=0.1]
\tikzstyle{fleche}=[-,>=latex]
\tikzstyle{noeud}=[fill=black,circle,draw,scale=0.2]
\tikzstyle{etiquette}=[midway]

\def\DistanceInterNiveaux{5}
\def\DistanceInterFeuilles{5}

\def\NiveauA{(-0)*\DistanceInterNiveaux}
\def\NiveauB{(-1)*\DistanceInterNiveaux}
\def\NiveauC{(-2)*\DistanceInterNiveaux}
\def\NiveauD{(-3)*\DistanceInterNiveaux}
\def\InterFeuilles{(1)*\DistanceInterFeuilles}

\node[noeud] (A) at ({(0)*\InterFeuilles},{\NiveauA}) {};
\node[noeud] (B) at ({(0)*\InterFeuilles},{\NiveauB}) {};
\node[noeud] (C) at ({(-1)*\InterFeuilles},{\NiveauC}) {};
\node[noeud] (D) at ({(0)*\InterFeuilles},{\NiveauC}) {};
\node[noeud] (E) at ({(1)*\InterFeuilles},{\NiveauC}) {};
\node[noeud] (F) at ({(0)*\InterFeuilles},{\NiveauD}) {};
\node[noeud] (H) at ({(0.7)*\InterFeuilles},{\NiveauD}) {};
\node[noeud] (I) at ({(0.9)*\InterFeuilles},{\NiveauD}) {};
\node[noeud] (J) at ({(1.1)*\InterFeuilles},{\NiveauD}) {};
\node[noeud] (K) at ({(1.3)*\InterFeuilles},{\NiveauD}) {};

\draw[fleche] (A)--(B) node[etiquette] {};
\draw[fleche] (B)--(C) node[etiquette] {};
\draw[fleche] (B)--(D) node[etiquette] {};
\draw[fleche] (B)--(E) node[etiquette] {};

\draw[fleche] (D)--(F) node[etiquette] {};

\draw[fleche] (E)--(H) node[etiquette] {};
\draw[fleche] (E)--(I) node[etiquette] {};
\draw[fleche] (E)--(J) node[etiquette] {};
\draw[fleche] (E)--(K) node[etiquette] {};
\end{tikzpicture}
\caption{Two trees $T_0$ and $T_1$ that fulfill conditions 1 and 2.}
\label{fig:trees:ass12}
\end{figure}

Trees of class $i$ will be obtained as random editions of $T_i$. In the sequel, $T_i(v\mapsto\tau)$ denotes the tree $T_i$ in which the subtree rooted at $u$ has been replaced by $\tau$.
These random edits will tend to make trees of class $0$ closer to trees of class $1$. To this end, we introduce the following additional assumption. Let $(\tau_h)$ a sequence of trees such that $\height(\tau_h) = h$.
\begin{enumerate}
\setcounter{enumi}{2}
\item Let $u\in T_0$ and $v\in T_1$. We consider the edited trees $T_0' = T_0(u\mapsto\tau_{\height(u)})$ and $T_1'=T_1(v\mapsto\tau_{\height(v)})$. Then,
$\forall\,u'\in T_0'\setminus \left(\tau_{\height(u)} \cup \leaves(T_0')\right)$, $\forall\,v'\in T_1'\setminus \left(\tau_{\height(v)}  \cup\leaves(T_1')\right)$, $T_0'[u']\neq T_1'[v']$.
\end{enumerate}

In other words, if one replaces subtrees of $T_0$ and $T_1$ by subtrees of the same height, then any subtree of $T_0$ is not isomorphic to a subtree of $T_1$ (except the new subtrees and leaves). This means that the similarity between random edits of $T_0$ and $T_1$ will come only from the new subtrees and not from collateral modifications. We refer to Fig.\,\ref{fig:trees:ass123} for an example of trees that satisfy these conditions.

\begin{figure}[h]
\centering
\begin{tikzpicture}[xscale=0.1,yscale=0.1]
\tikzstyle{fleche}=[-,>=latex]
\tikzstyle{noeud}=[fill=black,circle,draw,scale=0.2]
\tikzstyle{etiquette}=[midway]

\def\DistanceInterNiveaux{5}
\def\DistanceInterFeuilles{5}

\def\NiveauA{(-0)*\DistanceInterNiveaux}
\def\NiveauB{(-1)*\DistanceInterNiveaux}
\def\NiveauC{(-2)*\DistanceInterNiveaux}
\def\NiveauD{(-3)*\DistanceInterNiveaux}
\def\InterFeuilles{(1)*\DistanceInterFeuilles}

\node[noeud] (A) at ({(0)*\InterFeuilles},{\NiveauA}) {};
\node[noeud] (B) at ({(0)*\InterFeuilles},{\NiveauB}) {};
\node[noeud] (C) at ({(-1)*\InterFeuilles},{\NiveauC}) {};
\node[noeud] (D) at ({(0)*\InterFeuilles},{\NiveauC}) {};
\node[noeud] (E) at ({(1)*\InterFeuilles},{\NiveauC}) {};
\node[noeud] (F) at ({(-0.3)*\InterFeuilles},{\NiveauD}) {};
\node[noeud] (G) at ({(0.3)*\InterFeuilles},{\NiveauD}) {};
\node[noeud] (H) at ({(0.7)*\InterFeuilles},{\NiveauD}) {};
\node[noeud] (I) at ({(1)*\InterFeuilles},{\NiveauD}) {};
\node[noeud] (J) at ({(1.3)*\InterFeuilles},{\NiveauD}) {};

\draw[fleche] (A)--(B) node[etiquette] {};
\draw[fleche] (B)--(C) node[etiquette] {};
\draw[fleche] (B)--(D) node[etiquette] {};
\draw[fleche] (B)--(E) node[etiquette] {};

\draw[fleche] (D)--(F) node[etiquette] {};
\draw[fleche] (D)--(G) node[etiquette] {};

\draw[fleche] (E)--(H) node[etiquette] {};
\draw[fleche] (E)--(I) node[etiquette] {};
\draw[fleche] (E)--(J) node[etiquette] {};
\end{tikzpicture}
\hspace{2cm}
\begin{tikzpicture}[xscale=0.1,yscale=0.1]
\tikzstyle{fleche}=[-,>=latex]
\tikzstyle{noeud}=[fill=black,circle,draw,scale=0.2]
\tikzstyle{etiquette}=[midway]

\def\DistanceInterNiveaux{5}
\def\DistanceInterFeuilles{5}

\def\NiveauA{(-0)*\DistanceInterNiveaux}
\def\NiveauB{(-1)*\DistanceInterNiveaux}
\def\NiveauC{(-2)*\DistanceInterNiveaux}
\def\NiveauD{(-3)*\DistanceInterNiveaux}
\def\InterFeuilles{(1)*\DistanceInterFeuilles}

\node[noeud] (A) at ({(0)*\InterFeuilles},{\NiveauA}) {};

\node[noeud] (AA) at ({(-1)*\InterFeuilles},{\NiveauB}) {};

\node[noeud] (B) at ({(0)*\InterFeuilles},{\NiveauB}) {};
\node[noeud] (C) at ({(-1)*\InterFeuilles},{\NiveauC}) {};
\node[noeud] (D) at ({(0)*\InterFeuilles},{\NiveauC}) {};
\node[noeud] (E) at ({(1)*\InterFeuilles},{\NiveauC}) {};
\node[noeud] (F) at ({(0)*\InterFeuilles},{\NiveauD}) {};
\node[noeud] (H) at ({(0.7)*\InterFeuilles},{\NiveauD}) {};
\node[noeud] (I) at ({(0.9)*\InterFeuilles},{\NiveauD}) {};
\node[noeud] (J) at ({(1.1)*\InterFeuilles},{\NiveauD}) {};
\node[noeud] (K) at ({(1.3)*\InterFeuilles},{\NiveauD}) {};

\draw[fleche] (A)--(AA) node[etiquette] {};
\draw[fleche] (A)--(B) node[etiquette] {};
\draw[fleche] (B)--(C) node[etiquette] {};
\draw[fleche] (B)--(D) node[etiquette] {};
\draw[fleche] (B)--(E) node[etiquette] {};

\draw[fleche] (D)--(F) node[etiquette] {};

\draw[fleche] (E)--(H) node[etiquette] {};
\draw[fleche] (E)--(I) node[etiquette] {};
\draw[fleche] (E)--(J) node[etiquette] {};
\draw[fleche] (E)--(K) node[etiquette] {};
\end{tikzpicture}
\caption{Two trees $T_0$ and $T_1$ that fulfill conditions 1, 2 and 3.}
\label{fig:trees:ass123}
\end{figure}
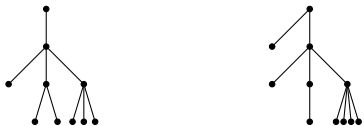

\subsection{A stochastic model of $2$-classes tree data}\label{ss:stomodel}

From now on, we assume that, for any $h>0$, $\tau_h$ is not a subtree of $T_0$ nor $T_1$. For the sake of simplicity, $T_0$ and $T_1$ have the same height $H$. In addition, if $u\in T_i$ then $T_i^u$ denotes $T_i(u\mapsto\tau_{\height(u)})$.

The stochastic model of $2$-classes tree data that we consider is defined from the binomial distribution $P_\rho = \mathcal{B}(H,\rho/H)$ on support $\{0,\dots,H\}$ with mean $\overline{P}_\rho = \rho$. The parameter $\rho \in [0,H]$ is fixed. In the data set, class $i$ is composed of random trees $T_i^u$, where the vertex $u$ has been picked uniformly at random among vertices of height $h$ in $T_i$, where $h$ follows $P_\rho$.
Furthermore, the considered training data set is well-balanced in the sense that it contains the same number of data of each class.

Intuitively, when $\rho$ increases, the trees are more degraded and thus two trees of different class are closer. $\rho$ somehow measures the similarity between the two classes. In other words, the larger $\rho$, the more difficult is the supervised classification problem.

\begin{remark}\label{rem:model}
The structure of a markup document such as an \verb+HTML+ page can be described by a tree (see Subsection~\ref{ss:prelim} and Fig.\,\ref{fig:htmltree} for more details). In this context, the tree $T_i$, $i\in\{0,1\}$, can be seen as a model of the structure of a webpage template. By assumption, the two templates of interest are as different as possible. However, they are completed in a similar manner, for example to present the same content in two different layouts. Edition of the templates is modeled by random edit operations. They tend to bring trees from different templates closer.
\end{remark}

\subsection{Theoretical guarantees on the subtree kernel}

\cite{DBLP:journals/ml/BalcanBS08} have introduced a theory that describes the effectiveness of a given kernel in terms of similarity-based properties. A similarity function over $\mathcal{X}$ is a pairwise function $K:\mathcal{X}^2\to[-1,1]$ \citep[Definition 1]{DBLP:journals/ml/BalcanBS08}. It is said $(\epsilon,\gamma)$-strongly good \citep[Definition 4]{DBLP:journals/ml/BalcanBS08} if, with probability at most $1-\epsilon$,
$$ \mathbb{E}_{x',y}[ K(x,x') - K(x,y) ] \geq \gamma,$$
where $\text{label}(x) = \text{label}(x')\neq \text{label}(y)$. From this definition, the authors derive the following simple classifier: the class of a new data $x$ is predicted by $1$ if $x$ is more similar on average to points in class $1$ than to points in class $0$, and $0$ otherwise. In addition, they prove \citep[Theorem 1]{DBLP:journals/ml/BalcanBS08} that a well-balanced training data set of size $32/\gamma^2\log(2/\delta)$ is sufficient so that, with probability at least $1-\delta$, the above algorithm applied to an $(\epsilon,\gamma)$-strongly good similarity function produces a classifier with error at most $\epsilon+\delta$.

We aim to prove comparable results for the subtree kernel that is not a similarity function. To this end, we focus for $i\in\{0,1\}$ on
\begin{equation}\label{eq:delta:def}
\Delta_x^i = \mathbb{E}_{u,v}[K(T_i^x,T_i^u) - K(T_i^x,T_{1-i}^v)]. 
\end{equation}

We emphasize that the two following results (Proposition~\ref{prop:ker:sep} and Corollary~\ref{cor:ker:size}) assume that the weight of leaves $\omega_\bullet$ is $0$. For the sake of readability, we introduce the following notations, for any $0\leq h\leq H$ and $i\in\{0,1\}$,

\begin{eqnarray*}
K_{i,h} &=& \max_{\{u\in T_i\,:\,\height(u) = h\}} K(T_i[u],T_i[u]),\\
C_{i,h} &=& \frac{K(T_i,T_i) - K_{i,h}}{\#\leaves(T_i)},\\
G_\rho(h) &=& 1 -\sum_{k=h+1}^H P_\rho(k). 
\end{eqnarray*}

The following results are expressed in terms of a parameter $0\leq h<H$. The statement is then true with probability $G_\rho(h)$. This is equivalent to state a result that is true with probability $1-\epsilon$, for any $\epsilon>0$.

\begin{proposition}
\label{prop:ker:sep}
If $w_{T_i}>0$ then $\Delta_x^i = 0$ if and only if $x=\roottree(T_i)$. In addition, if $\rho>H/2$, for any $0\leq h<H$, with probability $G_\rho(h)$, one has
\begin{equation}\label{eq:delta:min} \Delta_x^i \geq P_\rho(0) C_{i,h}.\end{equation}
\end{proposition}
\begin{proof} The proof lies in Appendix~\ref{app:proof:ker:sep}.\end{proof}

This result shows that the two classes can be well-separated by the subtree kernel. The only data that can not be separated are the trees completely edited. In addition, the lower-bound in \eqref{eq:delta:min} is of order $H\exp(-\rho)$ (up to a multiplicative constant).

\begin{corollary}
\label{cor:ker:size}
For any $0\leq h\leq H$, a well-balanced training data set of size
$$\frac{2\max_i K(T_i,T_i)^2}{\min_i C_{i,h}^2}\frac{\exp(2\rho)}{H^2} \log\left(\frac{2}{\delta}\right)$$
is sufficient so that, with probability at least $1-\delta$, the aforementioned classification algorithm produces a classifier with error at most $1-G_\rho(h)+\delta$.
\end{corollary}
\begin{proof} The proof is based on the demonstration of \cite[Theorem 1]{DBLP:journals/ml/BalcanBS08}. However, in our setting, the kernel $K$ is bounded by $\max_i K(T_i,T_i)$ and not by $1$. Consequently, by Hoeffding bounds, the sufficient size of the training data set if of order
\begin{equation}
\label{eq:data setsize}
2 \log\left(\frac{2}{\delta}\right)\frac{\max_iK(T_i,T_i)^2}{\gamma^2},
\end{equation}
where $\gamma$ can be read in Proposition~\ref{prop:ker:sep}, $\gamma = P_\rho(0) C_{i,h} \geq P_\rho(0) \min_i C_{i,h}$. The coefficient $2$ lies because we consider here the total size of the data set and not only the number of examples of each class. Together with $P_\rho(0)\sim H\exp(-\rho)$, we obtain the expected result. \end{proof}

\subsection{Weight of leaves}
\label{ss:weight:leaves}

Here $K^+$ is the subtree kernel obtained from the weights used in the computation of $K$ together with a positive weight on leaves, $w_\bullet>0$. We aim to show that $K^+$ separates the two classes less than $K$. $\Delta_x^{+,i}$ denotes the conditional expectation \eqref{eq:delta:def} computed from $K^+$.

\begin{proposition}
\label{prop:deltaplus}
For any $x\in T_i$,
$$
\Delta_x^{+,i} = \Delta_x^i + w_\bullet \#\leaves(T_i[x]) D_{i,1-i},
$$
where $D_{i,1-i} = \mathbb{E}_{u,v}[\#\leaves(T_i^u) - \#\leaves(T_{1-i}^v)]$.
\end{proposition}
\begin{proof} We have the following decomposition, for any trees $T_1$ and $T_2$,
$$K^+(T_1,T_2) = K(T_1,T_2) + w_\bullet\#\leaves(T_1)\#\leaves(T_2),$$
in light of the formula \eqref{eq:K:subtrees} of $K$. Thus, with \eqref{eq:delta:def},
\begin{eqnarray*}
\Delta_x^{+,i} &=& \mathbb{E}_{u,v}\left[K(T_i^x,T_i^u)+w_\bullet\#\leaves(T_i^x)\#\leaves(T_i^u) - K(T_i^x,T_{1-i}^v)-w_\bullet\#\leaves(T_i^x)\#\leaves(T_{1-i}^v)\right] \\
&=& \Delta_x^i + \mathbb{E}_{u,v}\left[w_\bullet \#\leaves(T_i^x)(\#\leaves(T_i^u)-\#\leaves(T_i^v))\right],
\end{eqnarray*}
which ends the proof.
\end{proof}

The sufficient number of data provided in Corollary~\ref{cor:ker:size} is obtained \eqref{eq:data setsize} through the square ratio of $\max_i K(T_i,T_i)$ over $\min_i \Delta_x^i$. First, it should be noticed that $K^+(T_i,T_i)>K(T_i,T_i)$. In addition, by virtue of Proposition~\ref{prop:deltaplus}, either $\Delta_x^{+,0}\leq \Delta_x^0$ or $\Delta_x^{+,1}\leq \Delta_x^1$ (and the inequality is strict if trees of classes $0$ and $1$ have not the same number of leaves on average). Consequently,
$$\min_i \Delta_x^{+,i} \leq \min_i \Delta_x^i,$$
and thus the sufficient number of data mentioned above is minimum for $\omega_\bullet = 0$.

\begin{remark}\label{rem:link:model:discr}
The results stated in this section establish that the subtree kernel is more efficient when the weight of leaves is $0$. It should be placed in perspective with the exponential weighting scheme of the literature \citep{aiolli2006fast,collins2002convolution,da2009kernel,kimura2011subpath,vishwanathan2002fast} for which the weight of leaves is maximal. We conjecture that the accuracy of the subtree kernel should be in general improved by imposing a null weight to any subtree present in two different classes. This can not be established from the model for which the only such subtrees are the leaves. Relying on this, one of the objectives of the sequel of the paper is to develop a learning method for the weight function that improves in practice the classification results (see Sections~\ref{sec:weight} and \ref{sec:data}).
\end{remark}


\section{DAG computation of the subtree kernel} \label{sec:dag}

In this section, we define DAG reduction, an algorithm that achieves both compression of data and enumeration of all subtrees of a tree without redundancies. DAG reduction of a tree is presented in Subsection~\ref{ss:dag:tree}, while Subsection~\ref{ss:dag:forest} is devoted to the compression of a forest. In Subsection~\ref{ss:dag:data set}, we state that the subtree kernel can be computed from the DAG reduction of data set of trees.

\subsection{DAG reduction of a tree} \label{ss:dag:tree}

Trees can present internal repetitions in their structure. Eliminating these structural redundancies defines a reduction of the initial data that can result in a Directed Acyclic Graph (DAG).
In particular, beginning with \cite{Sutherland:1963:SMG:1461551.1461591}, DAG representations of trees are also much used in computer graphics
where the process of condensing a tree into a graph is called object instancing \citep{Hart:1991:EAR:127719.122728}. DAG reduction can be computed upon unordered or ordered trees. We begin with the case of unordered trees.

\paragraph{Unordered trees} We consider the equivalence relation ``existence of an unordered tree isomorphism'' on the set of the
subtrees of a tree $T$: $Q(T)=(V,E)$ denotes the quotient graph obtained from $T$ using this
equivalence relation. $V$ is the set of equivalence classes on the subtrees of $T$, while $E$ is a set of
pairs of equivalence classes $(C_1,C_2)$ such that $\roottree(C_2)\in\child(\roottree(C_1))$ up to an isomorphism.
The graph $Q(T)$ is a DAG \citep[Proposition 1]{GODI} that is a connected directed graph without path from any vertex $v$ to itself.
Let $(C_1,C_2)$ be an edge of the DAG $Q(T)$. We define $L(C_1,C_2)$ as the number of occurrences of a
tree of $C_2$ just below the root of any tree of $C_1$. The tree reduction of $T$ is defined as the quotient graph $Q(T)$ augmented with labels $L(C_1,C_2)$
on its edges. We refer to Fig.\,\ref{fig:dag:reduction:unordered} for an example of DAG reduction of an unordered tree. Two different algorithms that allow the computation of the DAG reduction of an unordered tree but that share the same time-complexity in $\O(\#T^2\degr(T)\log(\degr(T)))$ are presented by \cite{GODI}.

\paragraph{Ordered trees} In the case of ordered trees, it is required to preserve the order of the children in the DAG reduction. As for unordered trees, we consider the quotient graph $Q(T)=(V,E)$ obtained from $T$ using the equivalence relation between ordered trees. $V$ is the set of equivalence classes on the subtrees of $T$. Here, the edges of the graph are ordered as follows. $(C_1,C_2)$ is the $r^\text{th}$ edge between $C_1$ and $C_2$ if $\roottree(C_2)$ is the $r^\text{th}$ child of $\roottree(C_1)$ up to an isomorphism. We obtain a DAG with ordered edges that compresses the initial tree $T$. An example of DAG reduction of an ordered tree is presented in Fig.\,\ref{fig:dag:reduction:ordered}. Polynomial algorithms have been developed to allow the computation of a DAG, with complexities ranging in $\O(\#T^2)$ to $\O(\#T)$ for ordered trees \citep{Downey:1980:VCS:322217.322228}.

\begin{figure}[h]
\centering
\begin{subfigure}[t]{0.49\textwidth}
\def\xscale{0.7}
\def\yscale{0.7}
\def\nodescale{0.7}
\begin{tikzpicture}[xscale=\xscale,yscale=\yscale]

\usetikzlibrary{decorations.markings}
\tikzset{->-/.style={decoration={
  markings,
  mark=at position .5 with {\arrow{>}}},postaction={decorate}}}

\tikzstyle{fleche}=[-,>=latex]
\tikzstyle{noeudcyan}=[draw,circle,fill=cyan,scale=\nodescale]
\tikzstyle{noeudyellow}=[draw,circle,fill=yellow,scale=\nodescale]
\tikzstyle{noeudmagenta}=[draw,circle,fill=magenta,scale=\nodescale]
\tikzstyle{noeudblue}=[draw,circle,fill=blue,scale=\nodescale]
\tikzstyle{noeudorange}=[draw,circle,fill=orange,scale=\nodescale]
\def\localnodescalea{1}
\node[noeudorange,scale=\localnodescalea] (4829600456) at ({0},{0}) {};
\node[noeudyellow,scale=\localnodescalea] (4819830824) at ({-1.5},{-1.0}) {};
\node[noeudmagenta,scale=\localnodescalea] (4701563536) at ({-2.5},{-2}) {};
\node[noeudcyan,scale=\localnodescalea] (4827613056) at ({-3.0},{-3}) {};
\node[noeudblue,scale=\localnodescalea] (4827655696) at ({-3.0},{-4}) {};
\node[noeudblue,scale=\localnodescalea] (4827720728) at ({-2.0},{-3}) {};
\node[noeudmagenta,scale=\localnodescalea] (4820693120) at ({-0.5},{-2}) {};
\node[noeudcyan,scale=\localnodescalea] (4820693512) at ({-1.0},{-3}) {};
\node[noeudblue,scale=\localnodescalea] (4820694240) at ({-1.0},{-4}) {};
\node[noeudblue,scale=\localnodescalea] (4820693904) at ({0.0},{-3}) {};
\node[noeudblue,scale=\localnodescalea] (4831129440) at ({1.0},{-1.0}) {};
\node[noeudmagenta,scale=\localnodescalea] (4831522376) at ({2.5},{-1.0}) {};
\node[noeudblue,scale=\localnodescalea] (4831529560) at ({2.0},{-2}) {};
\node[noeudcyan,scale=\localnodescalea] (4831169840) at ({3.0},{-2}) {};
\node[noeudblue,scale=\localnodescalea] (4831144032) at ({3.0},{-3}) {};
\draw[fleche] (4829600456)--(4819830824) {};
\draw[fleche] (4829600456)--(4831129440) {};
\draw[fleche] (4829600456)--(4831522376) {};
\draw[fleche] (4819830824)--(4701563536) {};
\draw[fleche] (4819830824)--(4820693120) {};
\draw[fleche] (4701563536)--(4827613056) {};
\draw[fleche] (4701563536)--(4827720728) {};
\draw[fleche] (4827613056)--(4827655696) {};
\draw[fleche] (4820693120)--(4820693512) {};
\draw[fleche] (4820693120)--(4820693904) {};
\draw[fleche] (4820693512)--(4820694240) {};
\draw[fleche] (4831522376)--(4831529560) {};
\draw[fleche] (4831522376)--(4831169840) {};
\draw[fleche] (4831169840)--(4831144032) {};

\def\xshift{6}
\def\yshift{-4}

\tikzstyle{arc}=[->-,>=latex]

\tikzstyle{fleche}=[-,>=latex]
\tikzstyle{noeud}=[draw,circle,fill=blue,scale=\nodescale*1]
\node[noeud] (0) at ({-0.5+\xshift},{0+\yshift}) {};
\tikzstyle{noeud}=[draw,circle,fill=cyan,scale=\nodescale*1]
\node[noeud] (1) at ({-0.5+\xshift},{1+\yshift}) {};
\tikzstyle{noeud}=[draw,circle,fill=magenta,scale=\nodescale*1]
\node[noeud] (2) at ({-0.5+\xshift},{2+\yshift}) {};
\tikzstyle{noeud}=[draw,circle,fill=yellow,scale=\nodescale*1]
\node[noeud] (3) at ({-0.5+\xshift},{3+\yshift}) {};
\tikzstyle{noeud}=[draw,circle,fill=orange,scale=\nodescale*1]
\node[noeud] (4) at ({-0.5+\xshift},{4+\yshift}) {};
\draw[arc] (1)--(0) {};
\draw[arc] (2) to [bend left=45](0) {};
\draw[arc] (2)--(1) {};
\draw[arc] (3)--(2) {} node [right,midway] {2};
\draw[arc] (4) to [bend left=45](0) {};
\draw[arc] (4) to [bend right=45](2) {};
\draw[arc] (4)--(3) {};
\end{tikzpicture}
\caption{Unordered\label{fig:dag:reduction:unordered}}
\end{subfigure}
\begin{subfigure}[t]{0.49\textwidth}
\def\xscale{0.7}
\def\yscale{0.7}
\def\nodescale{0.7}
\begin{tikzpicture}[xscale=\xscale,yscale=\yscale]

\usetikzlibrary{decorations.markings}
\tikzset{->-/.style={decoration={
  markings,
  mark=at position .5 with {\arrow{>}}},postaction={decorate}}}

\def\xshift{6}
\def\yshift{-4}

\tikzstyle{fleche}=[-,>=latex]
\tikzstyle{noeudcyan}=[draw,circle,fill=cyan,scale=\nodescale]
\tikzstyle{noeudyellow}=[draw,circle,fill=yellow,scale=\nodescale]
\tikzstyle{noeudmagenta}=[draw,circle,fill=magenta,scale=\nodescale]
\tikzstyle{noeudblue}=[draw,circle,fill=blue,scale=\nodescale]
\tikzstyle{noeudorange}=[draw,circle,fill=violet,scale=\nodescale]
\tikzstyle{noeudbrown}=[draw,circle,fill=brown,scale=\nodescale]
\def\localnodescalea{1}
\node[noeudbrown,scale=\localnodescalea] (4829600456) at ({0},{0}) {};
\node[noeudyellow,scale=\localnodescalea] (4819830824) at ({-1.5},{-1.0}) {};
\node[noeudmagenta,scale=\localnodescalea] (4701563536) at ({-2.5},{-2}) {};
\node[noeudcyan,scale=\localnodescalea] (4827613056) at ({-3.0},{-3}) {};
\node[noeudblue,scale=\localnodescalea] (4827655696) at ({-3.0},{-4}) {};
\node[noeudblue,scale=\localnodescalea] (4827720728) at ({-2.0},{-3}) {};
\node[noeudmagenta,scale=\localnodescalea] (4820693120) at ({-0.5},{-2}) {};
\node[noeudcyan,scale=\localnodescalea] (4820693512) at ({-1.0},{-3}) {};
\node[noeudblue,scale=\localnodescalea] (4820694240) at ({-1.0},{-4}) {};
\node[noeudblue,scale=\localnodescalea] (4820693904) at ({0.0},{-3}) {};
\node[noeudblue,scale=\localnodescalea] (4831129440) at ({1.0},{-1.0}) {};
\node[noeudorange,scale=\localnodescalea] (4831522376) at ({2.5},{-1.0}) {};
\node[noeudblue,scale=\localnodescalea] (4831529560) at ({2.0},{-2}) {};
\node[noeudcyan,scale=\localnodescalea] (4831169840) at ({3.0},{-2}) {};
\node[noeudblue,scale=\localnodescalea] (4831144032) at ({3.0},{-3}) {};
\draw[fleche] (4829600456)--(4819830824) {};
\draw[fleche] (4829600456)--(4831129440) {};
\draw[fleche] (4829600456)--(4831522376) {};
\draw[fleche] (4819830824)--(4701563536) {};
\draw[fleche] (4819830824)--(4820693120) {};
\draw[fleche] (4701563536)--(4827613056) {};
\draw[fleche] (4701563536)--(4827720728) {};
\draw[fleche] (4827613056)--(4827655696) {};
\draw[fleche] (4820693120)--(4820693512) {};
\draw[fleche] (4820693120)--(4820693904) {};
\draw[fleche] (4820693512)--(4820694240) {};
\draw[fleche] (4831522376)--(4831529560) {};
\draw[fleche] (4831522376)--(4831169840) {};
\draw[fleche] (4831169840)--(4831144032) {};

\tikzstyle{arc}=[->-,>=latex]
\tikzstyle{noeud}=[draw,circle,fill=blue,scale=\nodescale*1]
\node[noeud] (0) at ({-0.5+\xshift},{0+\yshift}) {};
\tikzstyle{noeud}=[draw,circle,fill=cyan,scale=\nodescale*1]
\node[noeud] (1) at ({-0.5+\xshift},{1+\yshift}) {};
\tikzstyle{noeud}=[draw,circle,fill=magenta,scale=\nodescale*1]
\node[noeud] (2) at ({-1.5+\xshift},{2+\yshift}) {};
\tikzstyle{noeud}=[draw,circle,fill=violet,scale=\nodescale*1]
\node[noeud] (3) at ({0.5+\xshift},{2+\yshift}) {};
\tikzstyle{noeud}=[draw,circle,fill=yellow,scale=\nodescale*1]
\node[noeud] (4) at ({-1.5+\xshift},{3+\yshift}) {};
\tikzstyle{noeud}=[draw,circle,fill=brown,scale=\nodescale*1]
\node[noeud] (5) at ({-0.5+\xshift},{4+\yshift}) {};
\draw[arc] (1)--(0) {};
\draw[arc] (2) to [bend right=90](1) {};
\draw[arc] (2)--(0) {};
\draw[arc] (3)--(0) {};
\draw[arc] (3) to [bend left=90](1) {};
\draw[arc] (4) to [bend right=45](2) {};
\draw[arc] (4) to [bend left=45](2) {};
\draw[arc] (5)--(4) {} ;
\draw[arc] (5) to [in=0, out=-90] (0) {} ;
\draw[arc] (5)--(3) {} ;
\end{tikzpicture}
\caption{Ordered\label{fig:dag:reduction:ordered}}
\end{subfigure}
\caption{A tree (left) and its DAG reduction (right) seen (\subref{fig:dag:reduction:unordered}) as an unordered tree and (\subref{fig:dag:reduction:ordered}) as an ordered tree. In each figure, roots of isomorphic subtrees are displayed with the same color, which is reproduced on the corresponding vertex of the DAG. Note that the subtree on the left is colored differently in the two cases, whether the order of its children is relevant or not. If no label is specified on an edge (in the unordered case), it is equal to $1$. }
\label{fig:dag:reduction:both}
\end{figure}
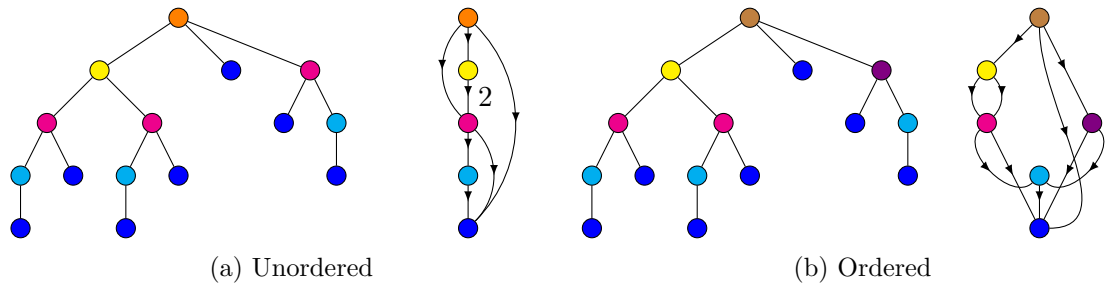

In this paper, $\red(T)$ denotes the DAG reduction of $T$ as $\ast$-tree, $\ast\in\{\text{ordered},\text{unordered}\}$. It is crucial to notice that the function $\red$ is a one-to-one correspondence, which means that DAG reduction is a lossless compression algorithm. In other words, $T$ can be reconstructed from $\red(T)$ and $(\red)^{-1}$ stands for the inverse function.

The DAG structure inherits of some properties of trees. For a vertex $\nu$ in a DAG $D$, we will denote by $\child(\nu)$ ($\parents(\nu)$, respectively) the set of children (parents, respectively) of $\nu$. $\height(\nu)$ and $\deg(\nu)$ are inherited as well. Similarly to trees, we denote by $D[\nu]$ the subDAG rooted in $\nu$ composed of $\nu$ and all its descendants in $D$.

\subsection{DAG reduction of a forest} \label{ss:dag:forest}

Let $\stree$ be the super-tree obtained from a forest of $\ast$-trees $\mathcal{F_T}=(T_1,\dots,T_N)$ by placing in this order each $T_i$ as a subtree of an artificial root. We define the DAG reduction of the forest $\mathcal{F_T}$ as $\red(\mathcal{F_T}) = \red(\stree)$.

However, if the forest $\mathcal{F_T}$ is stocked as a forest of compressed DAGs, that is, $\mathcal{F_D}=(D_1,\dots,D_N)$ (with $D_i= \red(T_i)$), it would be superfluous to decompress all trees before reducing the super-tree. So, one would rather compute $\red(\mathcal{F_T})$ directly from $\mathcal{F_D}$. From now on, we consider only forests of DAGs that we will denote unambiguously $\mathcal{F}$. In this context, $\red(\mathcal{F})$ stands for the DAG reduction of the forest of trees $((\red)^{-1}(D_1),\dots,\allowbreak (\red)^{-1}(D_N))$. We define the degree of the forest as $\degr(\mathcal{F})=\max_{i=1}^N \degr(D_i)$.

Computing $\red(\mathcal{F})$ from $(D_1,\dots,D_N)$ is in two steps: (i) we construct  a super-DAG $\sdag$ from $\mathcal{F}=(D_1,\dots,D_N)$ by placing in this order each $D_i$ as a subDAG of an artificial root (with time-complexity $\O(\degr(\mathcal{F})\sum_{i=1}^N \#D_i)$), and (ii) we recompress $\sdag$ using Algorithm~\ref{algo:dag:recompression}. Fig.\,\ref{fig:dag:recompression} illustrates step by step Algorithm \ref{algo:dag:recompression} on a forest of two trees seen as unordered then ordered trees.

\pagebreak 

\begin{algorithm}[H]
\caption{\textsc{DagRecompression}}\label{algo:dag:recompression}
\algorithmfootnote{It should be noticed that $\text{Im}\,f$ (that appears line 3) depends on $\ast$. Indeed, if $\ast = \text{ordered}$, $\text{Im}\,f$ is the set of all \emph{lists} of children; otherwise, $\text{Im}\,f$ is the set of all \emph{multisets} of children.}
\KwData{$\sdag$ the superdag obtained from a forest of DAG reductions of $\ast$-trees, $\ast \in \{\text{ordered},\text{unordered}\}$}
\KwResult{$\red(\mathcal{F})$}
Construct, within one exploration of $\sdag$, the mapping $h\mapsto \sdag^h$ where $\sdag^h$ is the set of vertices of $\sdag$ at height $h$\\
\For{$h$ \textup{\textbf{from}} $0$ \textup{\textbf{to}} $\height(\sdag)-1$}{
	Let $\ntm(h)= \big\lbrace f^{-1}(\lbrace S\rbrace) : S \in \text{Im } f, \# f^{-1}(\lbrace S \rbrace)\geq 2\big\rbrace$ be the set of vertices to be merged at height $h$, where $f: \nu\in \sdag^h \mapsto \child(\nu)$\\
	\eIf{$\ntm(h)=\emptyset$}{
		 Exit algorithm\;}
	{		 
		\For{$M$ \textup{\textbf{in}} $\ntm(h)$}{
			Choose one element $\nu_M$ in $M$ to remain in $\sdag$\\
			Denote by $\delta_M$ the other elements of $M$
		}
		\For{$\nu$ \textup{\textbf{in}} $\sdag$ \textup{\textbf{such that}} $\mathcal{H}(\nu)>h$}{
			\For{$\mu$ \textup{\textbf{in}} $\child(\nu)$ \textup{\textbf{such that}} $\exists\,M\in\ntm(h)$, $\delta_M\owns\mu$}{
				Delete $\mu$ from $\child(\nu)$\\
				Add $\nu_M$ in $\child(\nu)$
			}
		}
	
		\For{$M \in \ntm(h)$}{
			Delete $\delta_M$ from $\sdag$
		}
	}
}
\Return $\sdag$
\end{algorithm}

\afterpage{%
\begin{landscape}
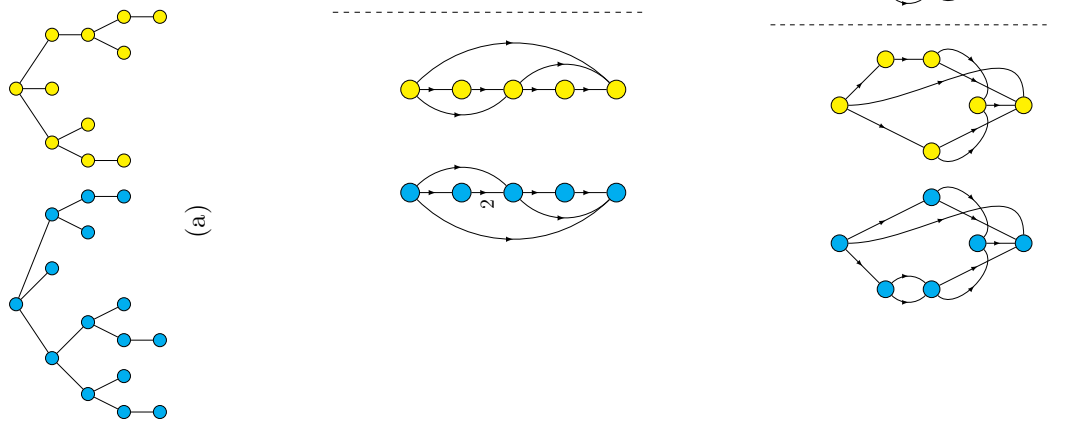
\begin{figure}[p]
\centering
\begin{minipage}[c][0.21\textwidth]{0.38\textheight}
    \begin{subfigure}[b]{\textwidth} 
        \centering 
\adjustbox{max width=\textwidth}{%
\def\xscale{0.7}
\def\yscale{0.7}
\def\nodescale{0.7}
\begin{tikzpicture}[xscale=\xscale,yscale=\yscale]

\def\xshift{6}
\def\yshift{0}
\tikzstyle{fleche}=[-,>=latex]
\tikzstyle{noeudcyan}=[draw,circle,fill=cyan,scale=\nodescale]
\def\localnodescalea{1}
\node[noeudcyan,scale=\localnodescalea] (4829600456) at ({0},{0}) {};
\node[noeudcyan,scale=\localnodescalea] (4819830824) at ({-1.5},{-1.0}) {};
\node[noeudcyan,scale=\localnodescalea] (4701563536) at ({-2.5},{-2}) {};
\node[noeudcyan,scale=\localnodescalea] (4827613056) at ({-3.0},{-3}) {};
\node[noeudcyan,scale=\localnodescalea] (4827655696) at ({-3.0},{-4}) {};
\node[noeudcyan,scale=\localnodescalea] (4827720728) at ({-2.0},{-3}) {};
\node[noeudcyan,scale=\localnodescalea] (4820693120) at ({-0.5},{-2}) {};
\node[noeudcyan,scale=\localnodescalea] (4820693512) at ({-1.0},{-3}) {};
\node[noeudcyan,scale=\localnodescalea] (4820694240) at ({-1.0},{-4}) {};
\node[noeudcyan,scale=\localnodescalea] (4820693904) at ({0.0},{-3}) {};
\node[noeudcyan,scale=\localnodescalea] (4831129440) at ({1.0},{-1.0}) {};
\node[noeudcyan,scale=\localnodescalea] (4831522376) at ({2.5},{-1.0}) {};
\node[noeudcyan,scale=\localnodescalea] (4831529560) at ({2.0},{-2}) {};
\node[noeudcyan,scale=\localnodescalea] (4831169840) at ({3.0},{-2}) {};
\node[noeudcyan,scale=\localnodescalea] (4831144032) at ({3.0},{-3}) {};
\draw[fleche] (4829600456)--(4819830824) {};
\draw[fleche] (4829600456)--(4831129440) {};
\draw[fleche] (4829600456)--(4831522376) {};
\draw[fleche] (4819830824)--(4701563536) {};
\draw[fleche] (4819830824)--(4820693120) {};
\draw[fleche] (4701563536)--(4827613056) {};
\draw[fleche] (4701563536)--(4827720728) {};
\draw[fleche] (4827613056)--(4827655696) {};
\draw[fleche] (4820693120)--(4820693512) {};
\draw[fleche] (4820693120)--(4820693904) {};
\draw[fleche] (4820693512)--(4820694240) {};
\draw[fleche] (4831522376)--(4831529560) {};
\draw[fleche] (4831522376)--(4831169840) {};
\draw[fleche] (4831169840)--(4831144032) {};

\tikzstyle{noeudyellow}=[draw,circle,fill=yellow,scale=\nodescale]
\def\localnodescalea{1}
\node[noeudyellow,scale=\localnodescalea] (4902318712) at ({0+\xshift},{0+\yshift}) {};
\node[noeudyellow,scale=\localnodescalea] (4784466072) at ({-1.5+\xshift},{-1.0+\yshift}) {};
\node[noeudyellow,scale=\localnodescalea] (4785037040) at ({-2.0+\xshift},{-2+\yshift}) {};
\node[noeudyellow,scale=\localnodescalea] (4785037264) at ({-2.0+\xshift},{-3+\yshift}) {};
\node[noeudyellow,scale=\localnodescalea] (4785049952) at ({-1.0+\xshift},{-2+\yshift}) {};
\node[noeudyellow,scale=\localnodescalea] (4902372072) at ({0.0+\xshift},{-1.0+\yshift}) {};
\node[noeudyellow,scale=\localnodescalea] (4785119016) at ({1.5+\xshift},{-1.0+\yshift}) {};
\node[noeudyellow,scale=\localnodescalea] (4785033456) at ({1.5+\xshift},{-2+\yshift}) {};
\node[noeudyellow,scale=\localnodescalea] (4785051744) at ({1.0+\xshift},{-3+\yshift}) {};
\node[noeudyellow,scale=\localnodescalea] (4861623488) at ({2.0+\xshift},{-3+\yshift}) {};
\node[noeudyellow,scale=\localnodescalea] (4531112424) at ({2.0+\xshift},{-4+\yshift}) {};
\draw[fleche] (4902318712)--(4784466072) {};
\draw[fleche] (4902318712)--(4902372072) {};
\draw[fleche] (4902318712)--(4785119016) {};
\draw[fleche] (4784466072)--(4785037040) {};
\draw[fleche] (4784466072)--(4785049952) {};
\draw[fleche] (4785037040)--(4785037264) {};
\draw[fleche] (4785119016)--(4785033456) {};
\draw[fleche] (4785033456)--(4785051744) {};
\draw[fleche] (4785033456)--(4861623488) {};
\draw[fleche] (4861623488)--(4531112424) {};
\end{tikzpicture}
}
                \caption{\label{fig:dag:recompression:trees}}
    \end{subfigure}

\end{minipage}
\hfill
\begin{minipage}[c][0.21\textwidth]{0.85\textheight}

\caption{An illustration step by step of the Algorithm \ref{algo:dag:recompression} with (\subref{fig:dag:recompression:trees}) two trees $T_1$ (in \textcolor{cyan}{cyan}) and $T_2$ (in yellow), seen as (\subref{fig:dag:recompression:unordered}) unordered or (\subref{fig:dag:recompression:ordered}) ordered trees. One can observe the DAGs (left) and the execution of the algorithm (right). At each step 1, 2 and 3, we examine vertices at height (0,1,2) and merge those which have same children. At step 4, we can not find any vertex to merge and we stop. Note that in (\subref{fig:dag:recompression:ordered}) at step 3, we find two pairs of vertices to be merged : we are not restricted to one pair per height. Merged vertices are colored in \textcolor{red}{red}. The artificial root is colored in \textbf{black}.}
\label{fig:dag:recompression}
\end{minipage}
\newline
\begin{minipage}[b][0.75\textwidth]{1.25\textheight}
    \begin{subfigure}[b]{\textwidth}
        \centering 
\adjustbox{totalheight=\textwidth, max width=1.15\textheight}{%
 \usetikzlibrary{decorations.markings}
\tikzset{->-/.style={decoration={
  markings,
  mark=at position .5 with {\arrow{>}}},postaction={decorate}}}

\def\xscale{1}
\def\yscale{1}
\def\nodescale{1}
\begin{tikzpicture}[xscale=\xscale,yscale=\yscale]
\def\ecart{0.5}

\def\xshift{5}
\def\yshift{0}

\tikzstyle{arc}=[->-,>=latex]
\tikzstyle{noeud}=[draw,circle,fill=cyan,scale=\nodescale*1]
\node[noeud] (0) at ({-1.0-\ecart},{0}) {};
\tikzstyle{noeud}=[draw,circle,fill=cyan,scale=\nodescale*1]
\node[noeud] (1) at ({-1.0-\ecart},{1}) {};
\tikzstyle{noeud}=[draw,circle,fill=cyan,scale=\nodescale*1]
\node[noeud] (2) at ({-1.0-\ecart},{2}) {};
\tikzstyle{noeud}=[draw,circle,fill=cyan,scale=\nodescale*1]
\node[noeud] (3) at ({-1.0-\ecart},{3}) {};
\tikzstyle{noeud}=[draw,circle,fill=cyan,scale=\nodescale*1]
\node[noeud] (4) at ({-1.0-\ecart},{4}) {};
\tikzstyle{noeud}=[draw,circle,fill=yellow,scale=\nodescale*1]
\node[noeud] (5) at ({0.0+\ecart},{0}) {};
\tikzstyle{noeud}=[draw,circle,fill=yellow,scale=\nodescale*1]
\node[noeud] (6) at ({0.0+\ecart},{1}) {};
\tikzstyle{noeud}=[draw,circle,fill=yellow,scale=\nodescale*1]
\node[noeud] (7) at ({0.0+\ecart},{2}) {};
\tikzstyle{noeud}=[draw,circle,fill=yellow,scale=\nodescale*1]
\node[noeud] (8) at ({0.0+\ecart},{3}) {};
\tikzstyle{noeud}=[draw,circle,fill=yellow,scale=\nodescale*1]
\node[noeud] (9) at ({0.0+\ecart},{4}) {};
\tikzstyle{noeud}=[draw,circle,fill=black,scale=\nodescale*1]
\node[noeud] (11) at ({-0.5},{5}) {};

\node[draw,circle] () at (-2,5) {1};

\draw[arc] (1)--(0) {};
\draw[arc] (2) to [bend right=45](0) {};
\draw[arc] (2)--(1) {};
\draw[arc] (3)--(2) {} node[midway, left] {2};
\draw[arc] (4) to [bend right=45](0) {};
\draw[arc] (4) to [bend left=45](2) {};
\draw[arc] (4)--(3) {};
\draw[arc] (6)--(5) {};
\draw[arc] (7) to [bend left=45](5) {};
\draw[arc] (7)--(6) {};
\draw[arc] (8)--(7) {};
\draw[arc] (9) to [bend left=45](5) {};
\draw[arc] (9) to [bend right=45](7) {};
\draw[arc] (9)--(8) {};
\draw[arc] (11)--(4) {};
\draw[arc] (11)--(9) {};

\tikzstyle{lien}=[-,>=latex,dashed,red]
\tikzstyle{noeud}=[draw,circle,dashed,red,scale=1.8*\nodescale]
\node[noeud] (40) at (-1-\ecart,0) {};
\node[noeud] (41) at (0+\ecart,0) {};
\draw[lien] (40)--(41) {} node[midway,above,red]{$=$};

\tikzstyle{arc}=[->-,>=latex]
\tikzstyle{noeud}=[draw,circle,fill=red,scale=\nodescale*1]
\node[noeud] (0) at ({-0.5+\xshift},{0}) {};
\tikzstyle{noeud}=[draw,circle,fill=cyan,scale=\nodescale*1]
\node[noeud] (1) at ({-1.0+\xshift-\ecart},{1}) {};
\tikzstyle{noeud}=[draw,circle,fill=cyan,scale=\nodescale*1]
\node[noeud] (2) at ({-1.0+\xshift-\ecart},{2}) {};
\tikzstyle{noeud}=[draw,circle,fill=cyan,scale=\nodescale*1]
\node[noeud] (3) at ({-1.0+\xshift-\ecart},{3}) {};
\tikzstyle{noeud}=[draw,circle,fill=cyan,scale=\nodescale*1]
\node[noeud] (4) at ({-1.0+\xshift-\ecart},{4}) {};
\tikzstyle{noeud}=[draw,circle,fill=yellow,scale=\nodescale*1]
\node[noeud] (6) at ({0.0+\xshift+\ecart},{1}) {};
\tikzstyle{noeud}=[draw,circle,fill=yellow,scale=\nodescale*1]
\node[noeud] (7) at ({0.0+\xshift+\ecart},{2}) {};
\tikzstyle{noeud}=[draw,circle,fill=yellow,scale=\nodescale*1]
\node[noeud] (8) at ({0.0+\xshift+\ecart},{3}) {};
\tikzstyle{noeud}=[draw,circle,fill=yellow,scale=\nodescale*1]
\node[noeud] (9) at ({0.0+\xshift+\ecart},{4}) {};
\tikzstyle{noeud}=[draw,circle,fill=black,scale=\nodescale*1]
\node[noeud] (11) at ({-0.5+\xshift},{5}) {};

\node[draw,circle] () at (-2+\xshift,5) {2};
\draw[arc] (1)--(0) {};
\draw[arc] (2) to [bend right=90](0) {};
\draw[arc] (2)--(1) {};
\draw[arc] (3)--(2) {} node[midway, left] {2};
\draw[arc] (4) to [bend right=120](0) {};
\draw[arc] (4) to [bend left=45](2) {};
\draw[arc] (4)--(3) {};
\draw[arc] (6)--(0) {};
\draw[arc] (7) to [bend left=90](0) {};
\draw[arc] (7)--(6) {};
\draw[arc] (8)--(7) {};
\draw[arc] (9) to [bend left=120](0) {};
\draw[arc] (9) to [bend right=45](7) {};
\draw[arc] (9)--(8) {};
\draw[arc] (11)--(4) {};
\draw[arc] (11)--(9) {};

\tikzstyle{lien}=[-,>=latex,dashed,red]
\tikzstyle{noeud}=[draw,circle,dashed,red,scale=1.8*\nodescale]
\node[noeud] (40) at (-1+\xshift-\ecart,1) {};
\node[noeud] (41) at (0+\xshift+\ecart,1) {};
\draw[lien] (40)--(41) {} node[midway,above,red]{$=$};

\tikzstyle{arc}=[->-,>=latex]
\tikzstyle{noeud}=[draw,circle,fill=red,scale=\nodescale*1]
\node[noeud] (0) at ({-0.5+2*\xshift},{0+\yshift}) {};
\tikzstyle{noeud}=[draw,circle,fill=red,scale=\nodescale*1]
\node[noeud] (1) at ({-0.5+2*\xshift},{1+\yshift}) {};
\tikzstyle{noeud}=[draw,circle,fill=cyan,scale=\nodescale*1]
\node[noeud] (2) at ({-1.0+2*\xshift-\ecart},{2+\yshift}) {};
\tikzstyle{noeud}=[draw,circle,fill=cyan,scale=\nodescale*1]
\node[noeud] (3) at ({-1.0+2*\xshift-\ecart},{3+\yshift}) {};
\tikzstyle{noeud}=[draw,circle,fill=cyan,scale=\nodescale*1]
\node[noeud] (4) at ({-1.0+2*\xshift-\ecart},{4+\yshift}) {};
\tikzstyle{noeud}=[draw,circle,fill=yellow,scale=\nodescale*1]
\node[noeud] (7) at ({0.0+2*\xshift+\ecart},{2+\yshift}) {};
\tikzstyle{noeud}=[draw,circle,fill=yellow,scale=\nodescale*1]
\node[noeud] (8) at ({0.0+2*\xshift+\ecart},{3+\yshift}) {};
\tikzstyle{noeud}=[draw,circle,fill=yellow,scale=\nodescale*1]
\node[noeud] (9) at ({0.0+2*\xshift+\ecart},{4+\yshift}) {};
\tikzstyle{noeud}=[draw,circle,fill=black,scale=\nodescale*1]
\node[noeud] (11) at ({-0.5+2*\xshift},{5+\yshift}) {};

\node[draw,circle] () at (-2+2*\xshift,5+\yshift) {3};

\draw[arc] (1)--(0) {};
\draw[arc] (2) to [bend right=45](0) {};
\draw[arc] (2)--(1) {};
\draw[arc] (3)--(2) {} node[midway, left] {2};
\draw[arc] (4) to [bend right=90](0) {};
\draw[arc] (4) to [bend left=45](2) {};
\draw[arc] (4)--(3) {};
\draw[arc] (7) to [bend left=45](0) {};
\draw[arc] (7)--(1) {};
\draw[arc] (8)--(7) {};
\draw[arc] (9) to [bend left=90](0) {};
\draw[arc] (9) to [bend right=45](7) {};
\draw[arc] (9)--(8) {};
\draw[arc] (11)--(4) {};
\draw[arc] (11)--(9) {};

\tikzstyle{lien}=[-,>=latex,dashed,red]
\tikzstyle{noeud}=[draw,circle,dashed,red,scale=1.8*\nodescale]
\node[noeud] (40) at (-1+2*\xshift-\ecart,2+\yshift) {};
\node[noeud] (41) at (0+2*\xshift+\ecart,2+\yshift) {};
\draw[lien] (40)--(41) {} node[midway,above,red]{$=$};

\tikzstyle{arc}=[->-,>=latex]
\tikzstyle{noeud}=[draw,circle,fill=red,scale=\nodescale*1]
\node[noeud] (0) at ({-0.5+2*\xshift+\xshift},{0+\yshift}) {};
\tikzstyle{noeud}=[draw,circle,fill=red,scale=\nodescale*1]
\node[noeud] (1) at ({-0.5+2*\xshift+\xshift},{1+\yshift}) {};
\tikzstyle{noeud}=[draw,circle,fill=red,scale=\nodescale*1]
\node[noeud] (2) at ({-0.5+2*\xshift+\xshift},{2+\yshift}) {};
\tikzstyle{noeud}=[draw,circle,fill=cyan,scale=\nodescale*1]
\node[noeud] (3) at ({-1.0+2*\xshift+\xshift-\ecart},{3+\yshift}) {};
\tikzstyle{noeud}=[draw,circle,fill=cyan,scale=\nodescale*1]
\node[noeud] (4) at ({-1.0+2*\xshift+\xshift-\ecart},{4+\yshift}) {};
\tikzstyle{noeud}=[draw,circle,fill=yellow,scale=\nodescale*1]
\node[noeud] (8) at ({0.0+2*\xshift+\xshift+\ecart},{3+\yshift}) {};
\tikzstyle{noeud}=[draw,circle,fill=yellow,scale=\nodescale*1]
\node[noeud] (9) at ({0.0+2*\xshift+\xshift+\ecart},{4+\yshift}) {};
\tikzstyle{noeud}=[draw,circle,fill=black,scale=\nodescale*1]
\node[noeud] (11) at ({-0.5+2*\xshift+\xshift},{5+\yshift}) {};

\node[draw,circle] () at (-2+2*\xshift+\xshift,5+\yshift) {4};

\draw[arc] (1)--(0) {};
\draw[arc] (2) to [bend right=45](0) {};
\draw[arc] (2)--(1) {};
\draw[arc] (3)--(2) {} node[midway, above right] {2};
\draw[arc] (4) to [bend right=90](0) {};
\draw[arc] (4) to [bend right=90](2) {};
\draw[arc] (4)--(3) {};
\draw[arc] (8)--(2) {};
\draw[arc] (9) to [bend left=90](0) {};
\draw[arc] (9) to [bend left=90](2) {};
\draw[arc] (9)--(8) {};
\draw[arc] (11)--(4) {};
\draw[arc] (11)--(9) {};

\tikzstyle{lien}=[-,>=latex,dashed,blue]
\tikzstyle{noeud}=[draw,circle,dashed,blue,scale=1.8*\nodescale]
\node[noeud] (40) at (-1+2*\xshift+\xshift-\ecart,3+\yshift) {};
\node[noeud] (41) at (0+2*\xshift+\xshift+\ecart,3+\yshift) {};
\draw[lien] (40)--(41) {} node[midway,above,blue]{$\neq$};

\def\xshift{-4}
\def\yshift{0}

\tikzstyle{arc}=[->-,>=latex]
\tikzstyle{noeud}=[draw,circle,fill=cyan,scale=\nodescale*1]
\node[noeud] (0) at ({-0.5+1.5*\xshift},{0}) {};
\tikzstyle{noeud}=[draw,circle,fill=cyan,scale=\nodescale*1]
\node[noeud] (1) at ({-0.5+1.5*\xshift},{1}) {};
\tikzstyle{noeud}=[draw,circle,fill=cyan,scale=\nodescale*1]
\node[noeud] (2) at ({-0.5+1.5*\xshift},{2}) {};
\tikzstyle{noeud}=[draw,circle,fill=cyan,scale=\nodescale*1]
\node[noeud] (3) at ({-0.5+1.5*\xshift},{3}) {};
\tikzstyle{noeud}=[draw,circle,fill=cyan,scale=\nodescale*1]
\node[noeud] (4) at ({-0.5+1.5*\xshift},{4}) {};
\draw[arc] (1)--(0) {};
\draw[arc] (2) to [bend right=45](0) {};
\draw[arc] (2)--(1) {};
\draw[arc] (3)--(2) {} node[midway, left] {2};
\draw[arc] (4) to [bend right=45](0) {};
\draw[arc] (4) to [bend left=45](2) {};
\draw[arc] (4)--(3) {};

\tikzstyle{noeud}=[draw,circle,fill=yellow,scale=\nodescale*1]
\node[noeud] (0) at ({-0.5+\xshift},{0+\yshift}) {};
\tikzstyle{noeud}=[draw,circle,fill=yellow,scale=\nodescale*1]
\node[noeud] (1) at ({-0.5+\xshift},{1+\yshift}) {};
\tikzstyle{noeud}=[draw,circle,fill=yellow,scale=\nodescale*1]
\node[noeud] (2) at ({-0.5+\xshift},{2+\yshift}) {};
\tikzstyle{noeud}=[draw,circle,fill=yellow,scale=\nodescale*1]
\node[noeud] (3) at ({-0.5+\xshift},{3+\yshift}) {};
\tikzstyle{noeud}=[draw,circle,fill=yellow,scale=\nodescale*1]
\node[noeud] (4) at ({-0.5+\xshift},{4+\yshift}) {};
\draw[arc] (1)--(0) {};
\draw[arc] (2) to [bend left=45](0) {};
\draw[arc] (2)--(1) {};
\draw[arc] (3)--(2) {};
\draw[arc] (4) to [bend left=45](0) {};
\draw[arc] (4) to [bend right=45](2) {};
\draw[arc] (4)--(3) {};

\draw[dashed,>=latex] (1+\xshift,5.5)--(1+\xshift,-0.5) {};
\end{tikzpicture}}
 \caption{\label{fig:dag:recompression:unordered}}
    \end{subfigure}\\[\baselineskip]
    \begin{subfigure}[b]{\textwidth}
        \centering
\adjustbox{totalheight=\textwidth, max width = 1.25\textheight}{%

 \usetikzlibrary{decorations.markings}
\tikzset{->-/.style={decoration={
  markings,
  mark=at position .5 with {\arrow{>}}},postaction={decorate}}}
\def\xscale{1}
\def\yscale{1}
\def\nodescale{1}
\begin{tikzpicture}[xscale=\xscale,yscale=\yscale]
\def\ecart{0.5}
\tikzstyle{arc}=[->-,>=latex]

\def\xshift{6}
\def\yshift{0}

\tikzstyle{noeud}=[draw,circle,fill=cyan,scale=\nodescale*1]
\node[noeud] (0) at ({-1.5-\ecart},{0}) {};
\tikzstyle{noeud}=[draw,circle,fill=cyan,scale=\nodescale*1]
\node[noeud] (1) at ({-1.5-\ecart},{1}) {};
\tikzstyle{noeud}=[draw,circle,fill=cyan,scale=\nodescale*1]
\node[noeud] (2) at ({-2.5-\ecart},{2}) {};
\tikzstyle{noeud}=[draw,circle,fill=cyan,scale=\nodescale*1]
\node[noeud] (3) at ({-0.5-\ecart},{2}) {};
\tikzstyle{noeud}=[draw,circle,fill=cyan,scale=\nodescale*1]
\node[noeud] (4) at ({-2.5-\ecart},{3}) {};
\tikzstyle{noeud}=[draw,circle,fill=cyan,scale=\nodescale*1]
\node[noeud] (5) at ({-1.5-\ecart},{4}) {};
\tikzstyle{noeud}=[draw,circle,fill=yellow,scale=\nodescale*1]
\node[noeud] (6) at ({0.5+\ecart},{0}) {};
\tikzstyle{noeud}=[draw,circle,fill=yellow,scale=\nodescale*1]
\node[noeud] (7) at ({0.5+\ecart},{1}) {};
\tikzstyle{noeud}=[draw,circle,fill=yellow,scale=\nodescale*1]
\node[noeud] (8) at ({-0.5+\ecart},{2}) {};
\tikzstyle{noeud}=[draw,circle,fill=yellow,scale=\nodescale*1]
\node[noeud] (9) at ({1.5+\ecart},{2}) {};
\tikzstyle{noeud}=[draw,circle,fill=yellow,scale=\nodescale*1]
\node[noeud] (10) at ({1.5+\ecart},{3}) {};
\tikzstyle{noeud}=[draw,circle,fill=yellow,scale=\nodescale*1]
\node[noeud] (11) at ({0.5+\ecart},{4}) {};
\tikzstyle{noeud}=[draw,circle,fill=black,scale=\nodescale*1]
\node[noeud] (13) at ({-0.5},{5}) {};

\node[draw,circle] () at (-2,5) {1};

\draw[arc] (1)--(0) {};
\draw[arc] (2) to [bend right=90](1) {};
\draw[arc] (2)--(0) {};
\draw[arc] (3) to [bend left=90](1) {};
\draw[arc] (3)--(0) {};
\draw[arc] (4) to [bend right=45](2) {};
\draw[arc] (4) to [bend left=45](2) {};
\draw[arc] (5)--(4) {};
\draw[arc] (5) to [in=0, out=-90](0) {};
\draw[arc] (5)--(3) {};
\draw[arc] (7)--(6) {};
\draw[arc] (8) to [bend right=90](7) {};
\draw[arc] (8)--(6) {};
\draw[arc] (9)--(6) {};
\draw[arc] (9) to [bend left=90](7) {};
\draw[arc] (10)--(9) {};
\draw[arc] (11)--(8) {};
\draw[arc] (11) to [in=0,out=-90](6) {};
\draw[arc] (11)--(10) {};
\draw[arc] (13)--(5) {};
\draw[arc] (13)--(11) {};

\tikzstyle{lien}=[-,>=latex,dashed,red]
\tikzstyle{noeud}=[draw,circle,dashed,red,scale=1.8*\nodescale]
\node[noeud] (40) at (-1.5-\ecart,0) {};
\node[noeud] (41) at (0.5+\ecart,0) {};
\draw[lien] (40)--(41) {} node[midway,above,red]{$=$};

\tikzstyle{arc}=[->-,>=latex]

\tikzstyle{noeud}=[draw,circle,fill=red,scale=\nodescale*1]
\node[noeud] (0) at ({-0.5+\xshift},{0}) {};
\tikzstyle{noeud}=[draw,circle,fill=cyan,scale=\nodescale*1]
\node[noeud] (1) at ({-1.5+\xshift-\ecart},{1}) {};
\tikzstyle{noeud}=[draw,circle,fill=cyan,scale=\nodescale*1]
\node[noeud] (2) at ({-2.5+\xshift-\ecart},{2}) {};
\tikzstyle{noeud}=[draw,circle,fill=cyan,scale=\nodescale*1]
\node[noeud] (3) at ({-0.5+\xshift-\ecart},{2}) {};
\tikzstyle{noeud}=[draw,circle,fill=cyan,scale=\nodescale*1]
\node[noeud] (4) at ({-2.5+\xshift-\ecart},{3}) {};
\tikzstyle{noeud}=[draw,circle,fill=cyan,scale=\nodescale*1]
\node[noeud] (5) at ({-1.5+\xshift-\ecart},{4}) {};
\tikzstyle{noeud}=[draw,circle,fill=yellow,scale=\nodescale*1]
\node[noeud] (7) at ({0.5+\xshift+\ecart},{1}) {};
\tikzstyle{noeud}=[draw,circle,fill=yellow,scale=\nodescale*1]
\node[noeud] (8) at ({-0.5+\xshift+\ecart},{2}) {};
\tikzstyle{noeud}=[draw,circle,fill=yellow,scale=\nodescale*1]
\node[noeud] (9) at ({1.5+\xshift+\ecart},{2}) {};
\tikzstyle{noeud}=[draw,circle,fill=yellow,scale=\nodescale*1]
\node[noeud] (10) at ({1.5+\xshift+\ecart},{3}) {};
\tikzstyle{noeud}=[draw,circle,fill=yellow,scale=\nodescale*1]
\node[noeud] (11) at ({0.5+\xshift+\ecart},{4}) {};
\tikzstyle{noeud}=[draw,circle,fill=black,scale=\nodescale*1]
\node[noeud] (13) at ({-0.5+\xshift},{5}) {};

\node[draw,circle] () at (-2+\xshift,5) {2};

\draw[arc] (1) to [bend right=35] (0) {};
\draw[arc] (2) to [bend right=90](1) {};
\draw[arc] (2) to [bend right=45](0) {};
\draw[arc] (3) to [bend left=90](1) {};
\draw[arc] (3) to [bend right=45] (0) {};
\draw[arc] (4) to [bend right=45](2) {};
\draw[arc] (4) to [bend left=45](2) {};
\draw[arc] (5)--(4) {};
\draw[arc] (5) to [in=150, out=-90](0) {};
\draw[arc] (5)--(3) {};
\draw[arc] (7) to [bend left = 35] (0) {};
\draw[arc] (8) to [bend right=90](7) {};
\draw[arc] (8) to [bend left = 45](0) {};
\draw[arc] (9) to [bend left = 45] (0) {};
\draw[arc] (9) to [bend left=90](7) {};
\draw[arc] (10)--(9) {};
\draw[arc] (11)--(8) {};
\draw[arc] (11) to [in=30,out=-90](0) {};
\draw[arc] (11)--(10) {};
\draw[arc] (13)--(5) {};
\draw[arc] (13)--(11) {};

\tikzstyle{lien}=[-,>=latex,dashed,red]
\tikzstyle{noeud}=[draw,circle,dashed,red,scale=1.8*\nodescale]
\node[noeud] (40) at (-1.5+\xshift-\ecart,1) {};
\node[noeud] (41) at (0.5+\xshift+\ecart,1) {};
\draw[lien] (40)--(41) {} node[midway,above,red]{$=$};

\tikzstyle{noeud}=[draw,circle,fill=red,scale=\nodescale*1]
\node[noeud] (0) at ({-0.5+2*\xshift},{0+\yshift}) {};
\tikzstyle{noeud}=[draw,circle,fill=red,scale=\nodescale*1]
\node[noeud] (1) at ({-0.5+2*\xshift},{1+\yshift}) {};
\tikzstyle{noeud}=[draw,circle,fill=cyan,scale=\nodescale*1]
\node[noeud] (2) at ({-2.5+2*\xshift-\ecart},{2+\yshift}) {};
\tikzstyle{noeud}=[draw,circle,fill=cyan,scale=\nodescale*1]
\node[noeud] (3) at ({-0.5+2*\xshift-\ecart},{2+\yshift}) {};
\tikzstyle{noeud}=[draw,circle,fill=cyan,scale=\nodescale*1]
\node[noeud] (4) at ({-2.5+2*\xshift-\ecart},{3+\yshift}) {};
\tikzstyle{noeud}=[draw,circle,fill=cyan,scale=\nodescale*1]
\node[noeud] (5) at ({-1.5+2*\xshift-\ecart},{4+\yshift}) {};
\tikzstyle{noeud}=[draw,circle,fill=yellow,scale=\nodescale*1]
\node[noeud] (8) at ({-0.5+2*\xshift+\ecart},{2+\yshift}) {};
\tikzstyle{noeud}=[draw,circle,fill=yellow,scale=\nodescale*1]
\node[noeud] (9) at ({1.5+2*\xshift+\ecart},{2+\yshift}) {};
\tikzstyle{noeud}=[draw,circle,fill=yellow,scale=\nodescale*1]
\node[noeud] (10) at ({1.5+2*\xshift+\ecart},{3+\yshift}) {};
\tikzstyle{noeud}=[draw,circle,fill=yellow,scale=\nodescale*1]
\node[noeud] (11) at ({0.5+2*\xshift+\ecart},{4+\yshift}) {};
\tikzstyle{noeud}=[draw,circle,fill=black,scale=\nodescale*1]
\node[noeud] (13) at ({-0.5+2*\xshift},{5+\yshift}) {};

\node[draw,circle] () at (-2+2*\xshift,5+\yshift) {3};

\draw[arc] (1)--(0) {};
\draw[arc] (2) to [bend right=120](1) {};
\draw[arc] (2) to [bend right=45](0) {};
\draw[arc] (3) to [bend left=90](1) {};
\draw[arc] (3) to [bend right=45] (0) {};
\draw[arc] (4) to [bend right=45](2) {};
\draw[arc] (4) to [bend left=45](2) {};
\draw[arc] (5)--(4) {};
\draw[arc] (5) to [in=150, out=-90](0) {};
\draw[arc] (5)--(3) {};
\draw[arc] (8) to [bend right=90](1) {};
\draw[arc] (8) to [bend left = 45](0) {};
\draw[arc] (9) to [bend left = 45] (0) {};
\draw[arc] (9) to [bend left=120](1) {};
\draw[arc] (10)--(9) {};
\draw[arc] (11)--(8) {};
\draw[arc] (11) to [in=30,out=-90](0) {};
\draw[arc] (11)--(10) {};
\draw[arc] (13)--(5) {};
\draw[arc] (13)--(11) {};

\tikzstyle{lien}=[-,>=latex,dashed,red]
\tikzstyle{noeud}=[draw,circle,dashed,red,scale=1.8*\nodescale]
\node[noeud] (40) at (-2.5+2*\xshift-\ecart,2+\yshift) {};
\node[noeud] (41) at (-0.5+2*\xshift+\ecart,2+\yshift) {};
\draw[lien] (40) to [bend right=45](41) {} ;
\node[red] () at (-1.5+2*\xshift-\ecart,1+\yshift) {$=$};

\tikzstyle{lien}=[-,>=latex,dashed,violet]
\tikzstyle{noeud}=[draw,circle,dashed,violet,scale=1.8*\nodescale]
\node[noeud] (42) at (-0.5+2*\xshift-\ecart,2+\yshift) {};
\node[noeud] (43) at (1.5+2*\xshift+\ecart,2+\yshift) {};
\draw[lien] (42) to [bend left=45](43) {};
\node[violet] () at (-0.5+2*\xshift,2.8+\yshift) {$=$};

\tikzstyle{noeud}=[draw,circle,fill=red,scale=\nodescale*1]
\node[noeud] (0) at ({-0.5+2*\xshift+\xshift},{0+\yshift}) {};
\tikzstyle{noeud}=[draw,circle,fill=red,scale=\nodescale*1]
\node[noeud] (1) at ({-0.5+2*\xshift+\xshift},{1+\yshift}) {};
\tikzstyle{noeud}=[draw,circle,fill=red,scale=\nodescale*1]
\node[noeud] (2) at ({-1+2*\xshift+\xshift-\ecart},{2+\yshift}) {};
\tikzstyle{noeud}=[draw,circle,fill=red,scale=\nodescale*1]
\node[noeud] (3) at ({0+2*\xshift+\xshift+\ecart},{2+\yshift}) {};
\tikzstyle{noeud}=[draw,circle,fill=cyan,scale=\nodescale*1]
\node[noeud] (4) at ({-2.5+2*\xshift+\xshift-\ecart},{3+\yshift}) {};
\tikzstyle{noeud}=[draw,circle,fill=cyan,scale=\nodescale*1]
\node[noeud] (5) at ({-1.5+2*\xshift+\xshift-\ecart},{4+\yshift}) {};
\tikzstyle{noeud}=[draw,circle,fill=yellow,scale=\nodescale*1]
\node[noeud] (10) at ({1.5+2*\xshift+\xshift+\ecart},{3+\yshift}) {};
\tikzstyle{noeud}=[draw,circle,fill=yellow,scale=\nodescale*1]
\node[noeud] (11) at ({0.5+2*\xshift+\xshift+\ecart},{4+\yshift}) {};
\tikzstyle{noeud}=[draw,circle,fill=black,scale=\nodescale*1]
\node[noeud] (13) at ({-0.5+2*\xshift+\xshift},{5+\yshift}) {};

\node[draw,circle] () at (-2+2*\xshift+\xshift,5+\yshift) {4};

\draw[arc] (1)--(0) {};
\draw[arc] (2) to [bend right=90](1) {};
\draw[arc] (2)--(0) {};
\draw[arc] (3) to [bend left=90](1) {};
\draw[arc] (3)--(0) {};
\draw[arc] (4) to [bend right=45](2) {};
\draw[arc] (4) to [bend left=45](2) {};
\draw[arc] (5)--(4) {};
\draw[arc] (5) to [in=180, out=-90](0) {};
\draw[arc] (5)--(3) {};
\draw[arc] (10)--(3) {};
\draw[arc] (11)--(2) {};
\draw[arc] (11) to [in=0,out=-90](0) {};
\draw[arc] (11)--(10) {};
\draw[arc] (13)--(5) {};
\draw[arc] (13)--(11) {};

\tikzstyle{lien}=[-,>=latex,dashed,blue]
\tikzstyle{noeud}=[draw,circle,dashed,blue,scale=1.8*\nodescale]
\node[noeud] (40) at (-2.5+2*\xshift+\xshift-\ecart,3+\yshift) {};
\node[noeud] (41) at (1.5+2*\xshift+\xshift+\ecart,3+\yshift) {};
\draw[lien] (40)--(41) {} node[midway,above,blue]{$\neq$};

\def\xshift{-5}

\tikzstyle{noeud}=[draw,circle,fill=cyan,scale=\nodescale*1]
\node[noeud] (0) at ({-1+1.5*\xshift},{0}) {};
\tikzstyle{noeud}=[draw,circle,fill=cyan,scale=\nodescale*1]
\node[noeud] (1) at ({-1+1.5*\xshift},{1}) {};
\tikzstyle{noeud}=[draw,circle,fill=cyan,scale=\nodescale*1]
\node[noeud] (2) at ({-2+1.5*\xshift},{2}) {};
\tikzstyle{noeud}=[draw,circle,fill=cyan,scale=\nodescale*1]
\node[noeud] (3) at ({0+1.5*\xshift},{2}) {};
\tikzstyle{noeud}=[draw,circle,fill=cyan,scale=\nodescale*1]
\node[noeud] (4) at ({-2+1.5*\xshift},{3}) {};
\tikzstyle{noeud}=[draw,circle,fill=cyan,scale=\nodescale*1]
\node[noeud] (5) at ({-1+1.5*\xshift},{4}) {};
\draw[arc] (1)--(0) {};
\draw[arc] (2) to [bend right=90] (1) {};
\draw[arc] (2)--(0) {};
\draw[arc] (3)--(0) {};
\draw[arc] (3) to [bend left=90](1) {};
\draw[arc] (4) to [bend right=45](2) {};
\draw[arc] (4) to [bend left=45](2) {};
\draw[arc] (5)--(4) {};
\draw[arc] (5) to [in=0, out=-90](0) {};
\draw[arc] (5)--(3) {};

\def\yshift{0}

\tikzstyle{noeud}=[draw,circle,fill=yellow,scale=\nodescale*1]
\node[noeud] (0) at ({-0.5+\xshift},{0+\yshift}) {};
\tikzstyle{noeud}=[draw,circle,fill=yellow,scale=\nodescale*1]
\node[noeud] (1) at ({-0.5+\xshift},{1+\yshift}) {};
\tikzstyle{noeud}=[draw,circle,fill=yellow,scale=\nodescale*1]
\node[noeud] (2) at ({-1.5+\xshift},{2+\yshift}) {};
\tikzstyle{noeud}=[draw,circle,fill=yellow,scale=\nodescale*1]
\node[noeud] (3) at ({0.5+\xshift},{2+\yshift}) {};
\tikzstyle{noeud}=[draw,circle,fill=yellow,scale=\nodescale*1]
\node[noeud] (4) at ({0.5+\xshift},{3+\yshift}) {};
\tikzstyle{noeud}=[draw,circle,fill=yellow,scale=\nodescale*1]
\node[noeud] (5) at ({-0.5+\xshift},{4+\yshift}) {};
\draw[arc] (1)--(0) {};
\draw[arc] (2) to [bend right=90](1) {};
\draw[arc] (2)--(0) {};
\draw[arc] (3)--(0) {};
\draw[arc] (3) to [bend left=90](1) {};
\draw[arc] (4)--(3) {};
\draw[arc] (5)--(2) {};
\draw[arc] (5) to [in=0,out=-90](0) {};
\draw[arc] (5)--(4) {};

\draw[dashed,>=latex] (1.25+\xshift,5.5)--(1.25+\xshift,-0.5) {};
\end{tikzpicture}
}
  \caption{\label{fig:dag:recompression:ordered}}
    \end{subfigure}
     \end{minipage}

\end{figure}
\end{landscape}
}

\begin{proposition}\label{prop:dag:recompression:correct}
Algorithm \ref{algo:dag:recompression} correctly computes $\red(\mathcal{F})$.
\end{proposition}
\begin{proof}
Starting from the leaves, we examine all vertices of same height in $\sdag$. Those with same children (with respect to $\ast$) are merged into a single vertex. The algorithm stops when at some height $h$, we cannot find any vertices to be merged. Vertices that are merged in the algorithm represents isomorphic subtrees, so it suffices to prove that the algorithm stops at the right time. Let $h$ be the first height for which $\ntm(h)=\emptyset$.

Suppose by contradiction that some vertices were to be merged at some height $h'>h$. They represent isomorphic subtrees, so that their respective children should also be merged together, and all of their descendants by induction. As any vertex of height $h''+1$ admits at least one child of height $h''$, $\ntm(h)$ would not be empty, which is absurd.
\end{proof}

\begin{samepage} 
\begin{proposition}\label{prop:dag:recompression:complexity}
Algorithm \ref{algo:dag:recompression} has time-complexity:
\begin{enumerate}
\item $\O(\#\sdag \degr(\mathcal{F}) (\log\degr(\mathcal{F}) + \height(\sdag)))$ for unordered trees;
\item $\O(\#\sdag \degr(\mathcal{F}) \height(\sdag))$ for ordered trees.
\end{enumerate}
\end{proposition}
\nopagebreak[4]
\begin{proof}
The proof lies in Appendix \ref{app:proof:dag:recompression}.
\end{proof}
\end{samepage}

\begin{remark}
One might also want to treat online data, but without recompressing the whole data set when adding a single entry in the forest. Let $\red(\mathcal{F})$ be the already recompressed forest and $D$ a new DAG to be introduced in the data. It suffices to place $D$ has the rightmost child of the artificial root of $\red(\mathcal{F})$ to get $\mathbb{D}_{\mathcal{F}\cup D}$, then run Algorithm~\ref{algo:dag:recompression} to obtain $\red(\mathcal{F}\cup D)$.
\end{remark}

\subsection{DAG annotation and kernel computation} \label{ss:dag:data set}

We consider a data set composed of two parts: the train data set $\data{train} = (T_1, \dots,T_n)$ and the data set to predict $\data{pred} = (T_{n+1},\dots,T_N)$. In the train data set, the classes of the data are assumed to be known. Our aim is to compute two Gram matrices $G=\left[ K(T_i,T_j) \right]_{i,j}$, where:
\begin{itemize}
\item $(i,j)\in \data{train} \times \data{train}$ for the training matrix $G_{\text{train}}$;
\item $(i,j)\in \data{pred} \times \data{train}$ for the prediction matrix $G_{\text{pred}}$.
\end{itemize}

SVM algorithms will use $G_{\text{train}}$ to learn their classifying rule, and $G_{\text{pred}}$ to make predictions \citep[Section 6.1]{cristianini2000introduction}. Other algorithms, such as kernel PCA, would also require to compute a Gram matrix before processing \citep[Section 14.2]{Scholkopf:2001:LKS:559923}. We denote by $\Delta = \red(\data{train}\cup\data{pred})$ the DAG reduction of the data set and, for any $1\leq i\leq N$, $D_i=\red(T_i)$. DAG computation of the subtree kernel requires to annotate the DAG with different pieces of information.

\paragraph{Origins} In order to compute the subtree kernel, it will be necessary to retrieve from the vertices of $\Delta$ their origin in the data set, that is, from which tree they come from. For any vertex $\nu$ in $\Delta\setminus \roottree(\Delta)$, the origin of $\nu$ is defined as
$$\origin(\nu) =\big\lbrace i \in \lbrace1,\dots,n,n+1,\dots, N \rbrace:  D_i \ni \nu\big\rbrace.$$

Assuming that $(D_1,\dots, D_N)$ are children of the root of $\Delta$ in this order (which is achieved if  $\Delta$ had been constructed following the ideas developed in Subsection~\ref{ss:dag:forest}) leads to the following proposition.

\begin{proposition}\label{prop:origin}
Origins can be calculated using the recursive formula,
$$\forall\,\nu\in \Delta\setminus\roottree(\Delta),\;\origin(\nu)=\begin{cases}\hfil \lbrace i \rbrace & \text{if $\nu$ is the $i^{\text{th}}$ child of $\roottree(\Delta)$,}\\ \bigcup\limits_{{p} \in \parents(\nu)} \origin({p}) & \text{otherwise.}\end{cases}$$
\end{proposition}
\begin{proof}
Using the assumption, origins are correct for the children of $\roottree(\Delta)$. If $D_i \ni \nu$ for some $i\in \lbrace 1,\dots,N\rbrace$ and  $\nu\in \Delta$, then $D_i \supseteq \des(\nu)$. The statement follows by induction.
\end{proof}

\paragraph{Frequency vectors} Remember that in \eqref{eq:K:subtrees} $N_\tau(T)$ counts the number of subtrees of a tree $T$ that are $\ast$-isomorphic to the tree $\tau$. To compute the kernel, we need to know this value, and we claim that we can compute it using only $\Delta$. We associate to each vertex $\nu\in \Delta\setminus \roottree(\Delta)$ a frequency vector $\presence_\nu$ where, for any $1\leq i\leq N$, $\textstyle\presence_\nu(i) =  N_{(\red)^{-1}(\Delta[\nu])}(T_i)$.

\begin{proposition} \label{prop:presence}
Frequency vectors can be calculated using the recursive formula,
$$\forall\,\nu\in \Delta\setminus\roottree(\Delta),~\textstyle\presence_\nu=\begin{cases} \hfil(\mathbb{1}_{\{i \,\in\, \origin(\nu)\}})_{i \in \lbrace1,\dots, N \rbrace}  & \text{if $\nu \in \child(\roottree(\Delta))$,}\\ \sum\limits_{{p} \in \parents(\nu)}L(p,\nu)\presence_p  & \text{otherwise,}\end{cases}$$
where either $L(p,\nu)=1$ if $\ast=\text{\normalfont ordered}$, or $L(p,v)$ is the label on the edge between $p$ and $\nu$ in $\Delta$ if $\ast=\text{\normalfont unordered}$.
\end{proposition}
\begin{proof}
Let $\nu$ be in $\Delta\setminus\roottree(\Delta)$. If $\nu\in \child(\roottree(\Delta))$, then $\nu$ represents the root of a tree $T_i$ (possibly several trees if there are repetitions in the data set), and therefore $\presence_\nu(i) = N_{T_i}(T_i) = 1$. Otherwise, suppose by induction that $\presence_p(i)$ is correct for all $p\in\parents(\nu)$, and any $i$. We fix $p\in \parents(\nu)$. $\nu$ appears $L(p,\nu)$ times as a child of $p$, so if $(\red)^{-1}(\Delta[p])$ appears $\presence_p(i)$ times in $T_i$, then the number of occurrences of $(\red)^{-1}(\Delta[\nu])$ in $T_i$ as a child of $(\red)^{-1}(\Delta[p])$ is $L(p,\nu)\presence_p(i)$. Summing over all $p\in\parents(\nu)$ leads $\presence_\nu(i)$ to be correct as well. 
\end{proof}

\paragraph{DAG weighting} The last thing that we lack to compute the kernel is the weight function. Remember that it is defined for trees as a function $w : \mathcal{T} \to \mathbb{R}^+$. As we only need to know the weights of the subtrees associated to vertices of $\Delta$, we define the weight function for DAG as, for any $\nu\in\Delta$, $\omega_\nu = w_{(\red)^{-1}(\Delta[\nu])}$.

\begin{remark}
In light of Propositions~\ref{prop:origin} and \ref{prop:presence}, it should be noted that both $o$ and $\presence$ can be calculated in one exploration of $\Delta$. By definition, this is also true for $\omega$.
\end{remark}

\paragraph{DAG computation of the subtree kernel} We introduce the matching subtrees function $\mathcal{M}$ as
\begin{align*}
  \mathcal{M} \colon \lbrace 1,\dots,N\rbrace^2 &\to 2^\Delta\\
  (i,j) &\mapsto \lbrace \nu \in \Delta : \lbrace i,j\rbrace \subseteq \origin(\nu)\rbrace
\end{align*}
where $2^\Delta$ is the powerset of the vertices of $\Delta$. Note that $\mathcal{M}$ is symmetric. This leads us to the following proposition.
\begin{proposition}
For any $T_i, T_j \in  \data{train}\cup \data{pred}$, we have
$$K(T_i,T_j) = \sum_{\nu \in \mathcal{M}(i,j)} \omega_\nu \textstyle\presence_\nu(i) \textstyle\presence_\nu(j). $$
\end{proposition}

\begin{proof}
By construction, it suffices to show that $\red(\subtrees(T_i)\cap\subtrees(T_j)) = \mathcal{M}(i,j)$. Let $\tau \in \subtrees(T_i)\cap\subtrees(T_j)$. Then $\red(\tau) \in \red(T_i)$ and $\red(\tau) \in \red(T_j)$. Necessarily, $\red(\tau)\in \Delta$ and $\lbrace i,j \rbrace \subseteq \origin(\red(\tau))$. So $\red(\tau) \in \mathcal{M}(i,j)$. Reciprocally, let $\nu \in \mathcal{M}(i,j)$. We denote $\tau = (\red)^{-1}(\nu)$. As $\lbrace i,j \rbrace \subseteq \origin(\nu)$, then $\tau \in \subtrees(T_i)\cap\subtrees(T_j)$.
\end{proof}

\begin{remark}
\label{rem:compl:sub}
$\mathcal{M}$ can be created in $\O(N^2\#\Delta)$ within one exploration of $\Delta$ and allows afterward computations of the subtree kernel $K(T_i,T_j)$ in $\O(\#\mathcal{M}(i,j))=\O(\min(\#D_i,\#D_j))$, which is more efficient than the $O(\#T_i+\#T_j)$ algorithm proposed by \cite{vishwanathan2002fast} (the time-complexity is announced by \citet[Section 1]{kimura2011subpath}). However, since the whole process through Algorithm~\ref{algo:dag:recompression} is costly, the global method that we propose in this paper is not faster than existing algorithms. Nonetheless, our algorithm is particularly adapted to repeated computations from the same data, e.g., for tuning parameters. Indeed, once $\mathcal{M}$ and $\Delta$ have been created, they can be stored and are ready to use. An illustration of this property is provided from experimental data in Fig.\,\ref{fig:3species:time}.
\end{remark}

\begin{remark}
The DAG computation of the subtree kernel investigated in this section relies on \cite{aiolli2006fast,da2009kernel}. Our work and the aforementioned papers are different and complementary. First, our framework is valid for both ordered and unordered trees, while these papers focus only on ordered trees. In addition, the method developed by \cite{aiolli2006fast,da2009kernel} is only adapted to exponential weights (see equations (3.12) and (6.2) from \cite{da2009kernel}). Thus, even if this algorithm is also based on DAG reduction of trees, it is less general than ours since the weight function is not constrained (see in particular Section~\ref{sec:weight} where the weight function is learned from the data). Finally, in \citet[Section~4]{aiolli2006fast}, the time-complexities are studied only from a numerical point of view, while we state theoretical results.
\end{remark}

\section{Discriminance weight function}\label{sec:weight}

For a given probability level and a given classification error, and under the stochastic model of Subsection~\ref{ss:stomodel}, we state in Subsection~\ref{ss:weight:leaves} that the sufficient size of the training data set is minimum when the weight of leaves is $0$. In other words, counting the leaves, which are the only subtrees that appear in both classes, does not provide a relevant information to the classification problem associated to this model. As mentioned in Remark~\ref{rem:link:model:discr}, we conjecture that, in a more general model, this result would be true for any subtree present in both classes. In this section, we propose to rely on this idea by defining a new weight function, learned from the data and called discriminance weight that assigns a large weight to subtrees, that help to discriminate the classes, i.e., that are present or absent in exactly one class, and a low weight otherwise.

The training data set is divided into two parts: $\data{weight}=(T_1,\dots,T_m)$ to learn the weight function, and $\data{class}=(T_{m+1},\dots,T_n)$ to estimate the Gram matrix. For the sake of readability, $\Delta$ denotes the DAG reduction of the whole data set, including $\data{weight}$, $\data{class}$ and $\data{pred}$. In addition, we assume that the data are divided into $K$ classes numbered from $1$ to $K$.

For any vertex $\nu \in \Delta\setminus \roottree(\Delta)$, we define the vector $\rho_\nu$ of length $K$ as,
$$\forall\,1\leq k\leq K,~\rho_\nu(k) =\frac{1}{\#C_k}\sum_{T_i\in C_k} \mathbb{1}_{\{i\,\in\,\origin(\nu)\}},$$
where $(C_k)_{1\leq k\leq K}$ forms a partition of $\data{weight}$ such that $T_i\in C_k$ if and only if $T_i$ is in class $k$. In other words, $\rho_\nu(k)$ is the proportion of data in class $k$ that contain the subtree $(\red)^{-1}(\Delta[\nu])$. Therefore, $\rho_\nu$ belongs to the $K$-dimensional hypercube. It should be noticed that $\rho_\nu$ is a vector of zeros as soon as $(\red)^{-1}(\Delta[\nu])$ is not a subtree of a tree of $\data{weight}$.

For any $1\leq k\leq K$, let $e_k$ ($\overline{e}_k$, respectively) be the vector of zeros with a unique $1$ in position $k$ (vector of ones with a unique $0$ in position $k$, respectively). If $\rho_\nu = e_k$, the vertex $\nu$ corresponds to the subtree $(\red)^{-1}(\Delta[\nu])$, which only appears in class $k$: $\nu$ is thus a good discriminator of this class. Otherwise, if $\rho_\nu = \overline{e}_k$, the vertex $\nu$ appears in all the classes except class $k$ and is still a good discriminator of the class. For any vertex $\nu$, $\delta_\nu$ measures the distance between $\rho_\nu$ and its nearest point of interest $e_k$ or $\overline{e}_k$,
$$\delta_\nu = \min_{k=1}^K\,\min(|\rho_\nu - e_k| , |\rho_\nu - \overline{e}_k|).$$

It should be noted that the maximum value of $\delta_\nu$ depends on the number of classes and can be larger than $1$.
If $\delta_\nu$ is small, then $\rho_\nu$ is close to a point of interest. Consequently, since $\nu$ tends to discriminate a class, its weight should be large. In light of this remark, the discriminance weight of a vertex $\nu$ is defined as $\omega_\nu = f(1 - \delta_\nu)$, where $f:(-\infty,1]\to[0,1]$ is increasing with $f(x)=0$ for $x\leq0$ and $f(1)=1$.
Fig.\,\ref{fig:smoothstep} illustrates some usual choices for $f$. In the sequel, we chose $\omega_\nu = \smstp(1-\delta_\nu)$ with the smoothstep function $\smstp:x\mapsto3x^2-2x^3$. We borrowed the smoothstep function from computer graphics \citep[p.\,30]{ebert2003texturing}, where it is mostly used to have smooth transition in a threshold function. 

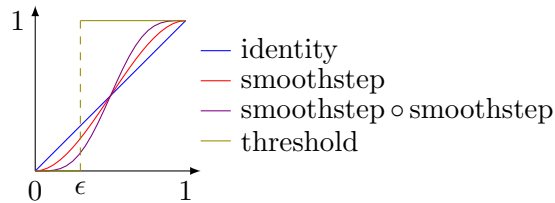
\begin{figure}[h]
\centering
\begin{minipage}[b]{0.46\textwidth}
\caption{The discriminance weight is defined by $\omega_\tau = f(1-\delta_\tau)$ where $f:(-\infty,1]\to[0,1]$ is increasing with $f(0)=0$ and $f(1)=1$. This figure presents some usual choices for $f$.}
\label{fig:smoothstep}
\end{minipage} \hfil
\begin{minipage}[b]{0.5\textwidth}
\begin{tikzpicture}[xscale=2,yscale=2]
\draw[-latex] (0,0) -- (0,1.1) node[above] {};
\draw[-latex] (0,0) -- (1.1,0) node[right] {};
\draw (0,0) node[below] {0};
\draw (0,1) node[left] {1};
\draw (1,0) node[below] {1};
\draw [domain=0:1,blue]plot (\x, \x);
\draw [domain=0:1,red]plot (\x,{ 3*\x^2-2*\x^3});
\draw [domain=0:1,violet]plot (\x,{3*(3*\x^2-2*\x^3)^2-2*(3*\x^2-2*\x^3)^3});
\draw [domain=0:0.3,olive]plot (\x, 0);
\draw [domain=0.3:1,olive]plot (\x, 1);
\draw[-,>=latex,olive,dashed] (0.3,1) -- (0.3,0) node[below,black] {$\epsilon$};
\draw[-,>=latex,blue] (1.1,0.8) -- (1.3,0.8) node[right,black] {identity};
\draw[-,>=latex,red] (1.1,0.6) -- (1.3,0.6) node[right,black] {smoothstep};
\draw[-,>=latex,violet] (1.1,0.4) -- (1.3,0.4) node[right,black] {smoothstep\,$\circ$\,smoothstep};
\draw[-,>=latex,olive] (1.1,0.2) -- (1.3,0.2) node[right,black] {threshold};
\end{tikzpicture}
\end{minipage}
\end{figure}

Since leaves appear in all the trees of the training data set, $\rho_\bullet$ is a vector of ones and thus $\delta_\bullet = 1$, which implies $\omega_\bullet = 0$. This is consistent with the result developed in Subsection~\ref{ss:weight:leaves} on the stochastic model. As aforementioned, the discriminance weight is inspired from the theoretical results established in Subsection~\ref{ss:weight:leaves} and the conjecture presented in Remark~\ref{rem:link:model:discr}. The relevance in practice of this weight function will be investigated in the sequel of the paper through two applications.

\begin{remark}\label{rem:discrim}
	The discriminance weight is defined from the proportion of data in each class that contain a given subtree, for all the subtrees appearing in the data set. It is thus required to enumerate all these subtrees. This is done, without redundancy, via the DAG reduction $\Delta$ of the data set defined and investigated in Section~\ref{sec:dag}.
	As the $m$ trees of the training data set dedicated to learning the discriminance weight are partitioned into $K$ classes, computing one $\rho_\nu$ vector is of complexity $\O(m)$. Therefore, computing all of them is in $\O(\#\Delta m)$. In addition, computing all values of $\delta_\nu$ is in $\O(\#\Delta K^2)$, as there are $2K$ Euclidean distances to be computed for each vector of length $K$. All gathered, computing the discriminance weight function has an overall complexity of $\O(\#\Delta(N+K^2))$.
\end{remark}


\section{Real data analysis}\label{sec:data}

This section is dedicated to the application of the methodology developed in the paper to eight real data sets with various characteristics in order to show its strengths and weaknesses. The related questions are supervised classification problems.
As mentioned in Subsection~\ref{ss:dag:data set}, our approach consists in computing the Gram matrices of the subtree kernel via DAG reduction and with a new weight function called the discriminance (see Section~\ref{sec:weight}).
In particular, we aim to compare the usual exponential weight of the literature and the latter in terms of prediction capability.
In all the sequel, the Gram matrices are used as inputs to SVM algorithms in order to tackle these classification problems. We emphasize that this approach is not restricted to SVM but can be applied with other prediction algorithms.

\subsection{Preliminaries} \label{ss:prelim}

In this subsection, we introduce (i) the protocol that we have followed to investigate several data sets, together with a description of (ii) the classification metrics that we use to assess the quality of our results, (iii) an extension of DAG reduction to take into account discrete labels on vertices of trees, and (iv) the standard method to convert a markup document into a tree. It should be already noted that all the data sets presented in the sequel are composed of trees (that can be ordered or unordered, labeled or not) together with their class.

\paragraph{Protocol} For each data set, we have followed the same presentation and procedure. First, a description of the data is made notably via histograms describing the size, outdegree, height and class repartition of trees. Given the dispersion of some of these quantities, we have binned together the values that does not fit inside the interval $[Q_1-1.5 \cdot IQR; Q_3+1.5\cdot IQR]$ where $IQR = Q_3-Q_1$ is the interquartile range. Therefore, the flattened-large bins that appears in some histograms represents those outliers bins. The objective of this part is to show the wide range of data sets considered in the paper.

In a second time, we evaluated the performance of the subtree kernel on a classification task via two methods: (i) for exponential weights $\tau \mapsto \lambda^{\height(\tau)}$ we randomly split the data in thirds, two for training a SVM, and one for prediction; (ii) for discriminance weight, we also randomly split the data in thirds, one for training the discriminance weight, one for training a SVM, and the last one for prediction. We repeated $50$ times this random split for discriminance, and for different values of $\lambda$. The classification results are assessed by some metrics defined in the upcoming paragraph, and gathered in boxplots. The first application example, presented in Subsection~\ref{ss:wikipedia}, is slightly different since (i) we have worked with $50$ distinct databases, and (ii) the results have been completed with a deeper analysis of the discriminance weights, in relation with the usual weighting scheme of the literature.

\paragraph{Classification metrics} To quantify the quality of a prediction, we use four standard metrics that are accuracy, precision, recall and F-score. For a class $k$, one can have true positives $TP_k$, false positives $FP_k$, true negatives $TN_k$ and false negatives $FN_k$. In a binary classification problem, those metrics are defined as,
\begin{eqnarray*}
\text{Accuracy}(k) &=& \displaystyle\frac{TP_k+TN_k}{TP_k+FP_k+FN_k+TN_k} ,\\
\text{Precision}(k) &=& \displaystyle\frac{TP_k}{TP_k+FP_k}, \\
\text{Recall}(k) &=& \displaystyle\frac{TP_k}{TP_k+FN_k} ,\\
\text{F-score}(k) &=& \displaystyle\frac{2\,\text{Precision}(k)\, \text{Recall}(k)}{\text{Precision}(k)+ \text{Recall}(k)}.
\end{eqnarray*}
For a problem with $K>2$ classes, we adopt the macro-average approach, that is, $$\text{Metric} = \frac{1}{K}\displaystyle\sum_{k=1}^K \text{Metric}(k).$$ We used the implementation available in the \verb=scikit-learn= Python library, via the two functions \verb=accuracy_score= and \verb=precision_recall_fscore_support=.

\paragraph{DAG reduction with labels} In the sequel, some of the presented data sets are composed of labeled trees, that are trees which each vertex possesses a label. Labels are supposed to take only a finite number of different values. Two labeled $\ast$-trees are said isomorphic if (i) they are $\ast$-isomorphic, and (ii) the underlying one-to-one correspondence mapping vertices of $T_1$ into vertices of $T_2$ is such that $\forall\,v\in T_1$, $v$ and $\Phi(v)$ have the same label. The set of labeled $\ast$-trees is the quotient set of rooted trees by this equivalence relation. It should be noticed that the subtree kernel as well as DAG reduction are defined through only the concept of isomorphic subtrees. As a consequence, they can be straightforwardly extended to labeled $\ast$-trees. This formalization is an extension of the definition introduced by the authors of \cite{aiolli2006fast,da2009kernel}, as they consider only ordered labeled trees, whereas we can consider unordered labeled trees as well.

\paragraph{From a markup document to a tree} Some of the data sets come from markup documents (\verb+XML+ or \verb+HTML+ files). From such a document, one can extract a tree structure, identifying each couple of opening and closing tags as a vertex, which children are the inner tags. It should be noticed that, during this transcription, semantic data is forgotten: the tree only describes the topology of the document. Fig.\,\ref{fig:htmltree} illustrates the conversion from \verb+HTML+ to tree on a small example. Such a tree is ordered but can be considered as unordered. Finally, a tag can also be chosen as a label for the corresponding vertex in the tree.
\begin{figure}[h]
\centering

\begin{minipage}[t]{0.5\textwidth}
\centering
\definecolor{cyan}{rgb}{0.0, 1.0, 1.0}
\definecolor{red}{rgb}{1.0, 0.0, 0.0}
\definecolor{yellow}{rgb}{1.0, 1.0, 0.0}
\definecolor{olive}{rgb}{0.5019607843137255, 0.5019607843137255, 0.0}
\definecolor{blue}{rgb}{0.0, 0.0, 1.0}
\definecolor{orange}{rgb}{1.0, 0.6470588235294118, 0.0}
\definecolor{magenta}{rgb}{1.0, 0.0, 1.0}
\definecolor{green}{rgb}{0.0, 0.5019607843137255, 0.0}

\includegraphics[width =\textwidth]{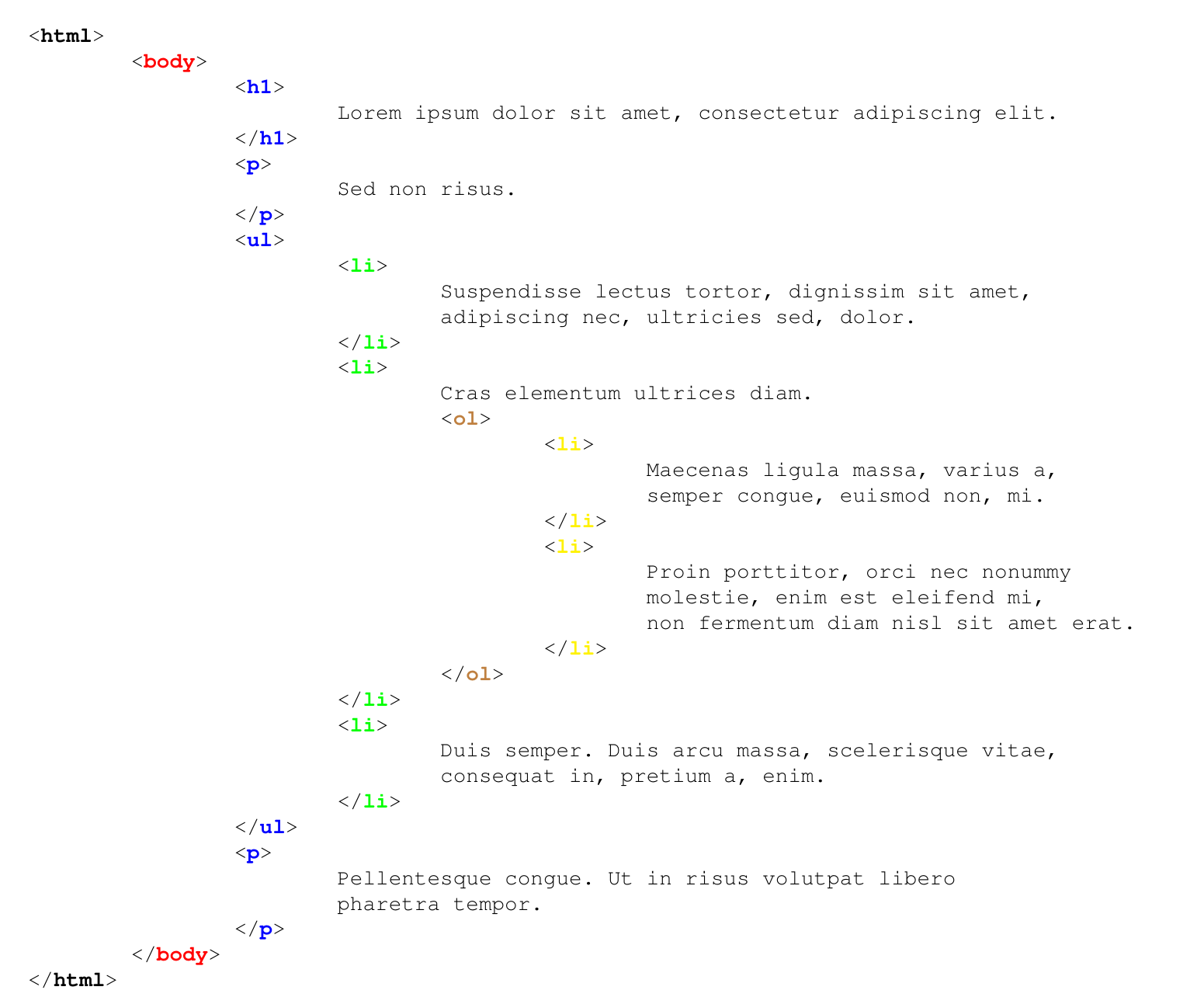}
\end{minipage}~
\begin{minipage}[t]{0.45\textwidth}
\centering

	\begin{tikzpicture}[xscale=0.1,yscale=0.2]
		\tikzstyle{fleche}=[-,>=latex,thick]
		\tikzstyle{noeud}=[circle,draw,scale=0.8]
		\tikzstyle{blanc}=[fill=white,circle,draw=none,scale=0.8]
		\tikzstyle{etiquette}=[midway]
		
		\def\DistanceInterNiveaux{6}
		\def\DistanceInterFeuilles{4}
		
		\def\NiveauA{(-5)*\DistanceInterNiveaux}
		\def\NiveauB{(-4)*\DistanceInterNiveaux}
		\def\NiveauC{(-3)*\DistanceInterNiveaux}
		\def\NiveauD{(-2)*\DistanceInterNiveaux}
		\def\NiveauE{(-1)*\DistanceInterNiveaux}
		\def\NiveauF{(0)*\DistanceInterNiveaux}
		\def\InterFeuilles{(1)*\DistanceInterFeuilles}
		
		\node[noeud,fill=black] (A) at ({(0)*\InterFeuilles},{\NiveauA}) {};
		
		\node[noeud,fill=red] (B) at ({(0)*\InterFeuilles},{\NiveauB}) {};
		\node[noeud,fill=blue] (C1) at ({(-3)*\InterFeuilles},{\NiveauC}) {};
		\node[noeud,fill=blue] (C2) at ({(-1)*\InterFeuilles},{\NiveauC}) {};
		\node[noeud,fill=blue] (C3) at ({(1)*\InterFeuilles},{\NiveauC}) {};
		\node[noeud,fill=blue] (C4) at ({(3)*\InterFeuilles},{\NiveauC}) {};
		
		\node[noeud,fill=black!30!green] (D1) at ({(0)*\InterFeuilles},{\NiveauD}) {};
		\node[noeud,fill=black!30!green] (D2) at ({(1)*\InterFeuilles},{\NiveauD}) {};
		\node[noeud,fill=black!30!green] (D3) at ({(2)*\InterFeuilles},{\NiveauD}) {};
		
		\node[noeud,fill=brown] (E) at ({(1)*\InterFeuilles},{\NiveauE}) {};
		
		\node[noeud,fill=yellow] (F1) at ({(0.5)*\InterFeuilles},{\NiveauF}) {};
		\node[noeud,fill=yellow] (F2) at ({(1.5)*\InterFeuilles},{\NiveauF}) {};
		
		\draw[fleche] (A)--(B) node[etiquette] {};
		\draw[fleche] (B)--(C1) node[etiquette] {};
		\draw[fleche] (B)--(C2) node[etiquette] {};
		\draw[fleche] (B)--(C3) node[etiquette] {};
		\draw[fleche] (B)--(C4) node[etiquette] {};
		\draw[fleche] (C3)--(D1) node[etiquette] {};
		\draw[fleche] (C3)--(D2) node[etiquette] {};
		\draw[fleche] (C3)--(D3) node[etiquette] {};
		\draw[fleche] (D2)--(E) node[etiquette] {};
		\draw[fleche] (E)--(F1) node[etiquette] {};
		\draw[fleche] (E)--(F2) node[etiquette] {};
		\end{tikzpicture}
\end{minipage}

\cprotect\caption{Underlying ordered tree structure (right) present in a \verb+HTML+ document (left). Each level in the tree is colored in the same way as the corresponding tags in the document. Natural order from top to bottom in the \verb|HTML| document corresponds to left-to-right order in the tree.}
	\label{fig:htmltree}
\end{figure}
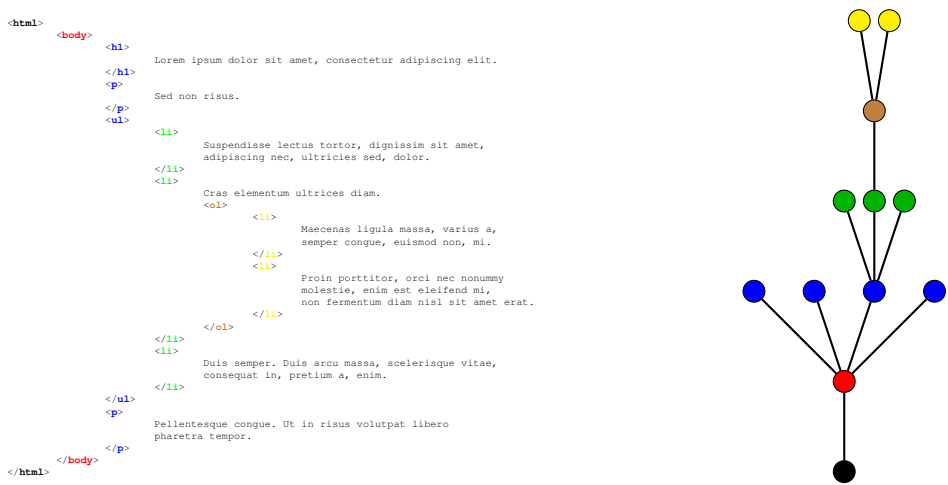

\subsection{Prediction of the language of a Wikipedia article from its topology}\label{ss:wikipedia}

\paragraph{Classification problem and results} Wikipedia pages are encoded in \verb+HTML+ and, as aforementioned, can therefore be converted into trees. In this context, we are interested in the following question: does the (ordered or unordered) topology of a Wikipedia article (as an \verb+HTML+ page) contain the information of the language in which it has been written? This can be formulated as a supervised classification problem: given a training data set composed of the tree structures of Wikipedia articles labeled with their language, is a prediction algorithm able to predict the language of a new data only from its topology? The interest of this question is discussed in Remark~\ref{rem:interest:wiki}.

In order to tackle this problem, we have built $50$ databases of $480$ trees each, converted from Wikipedia articles as follows. Each of the databases is composed of $4$ data sets:
\begin{itemize}
\item a data set to predict $\data{pred}$ made of $120$ trees;
\item a small train data set $\data[small]{train}$ made of $40$ trees;
\item a medium train data set $\data[medium]{train}$ made of $120$ trees;
\item and a large train data set $\data[large]{train}$ made of $200$ trees.
\end{itemize}

For each data set, and each language, we picked Wikipedia articles at random using the Wikipedia API\footnote{\url{https://www.mediawiki.org/wiki/API:Random} (last accessed in April 2020)}, and converted them into unlabeled trees. It should be noted that the probability to have the same article in at least two different languages is extremely low. For each database, we aim at predicting the language of the trees in $\data{pred}$ using a SVM algorithm based on the subtree kernel for ordered and unordered trees, and trained with $\data[size]{train}$ where $\text{size}\in\lbrace \text{small}, \text{medium}, \text{large}\rbrace$. Fig.\,\ref{fig:wikipedia:repartition} provides the description of one typical database. All trees seem to share common characteristics, regardless of their class. 

\begin{figure}[h]
    \centering
    \includegraphics[width = 0.48\textwidth]{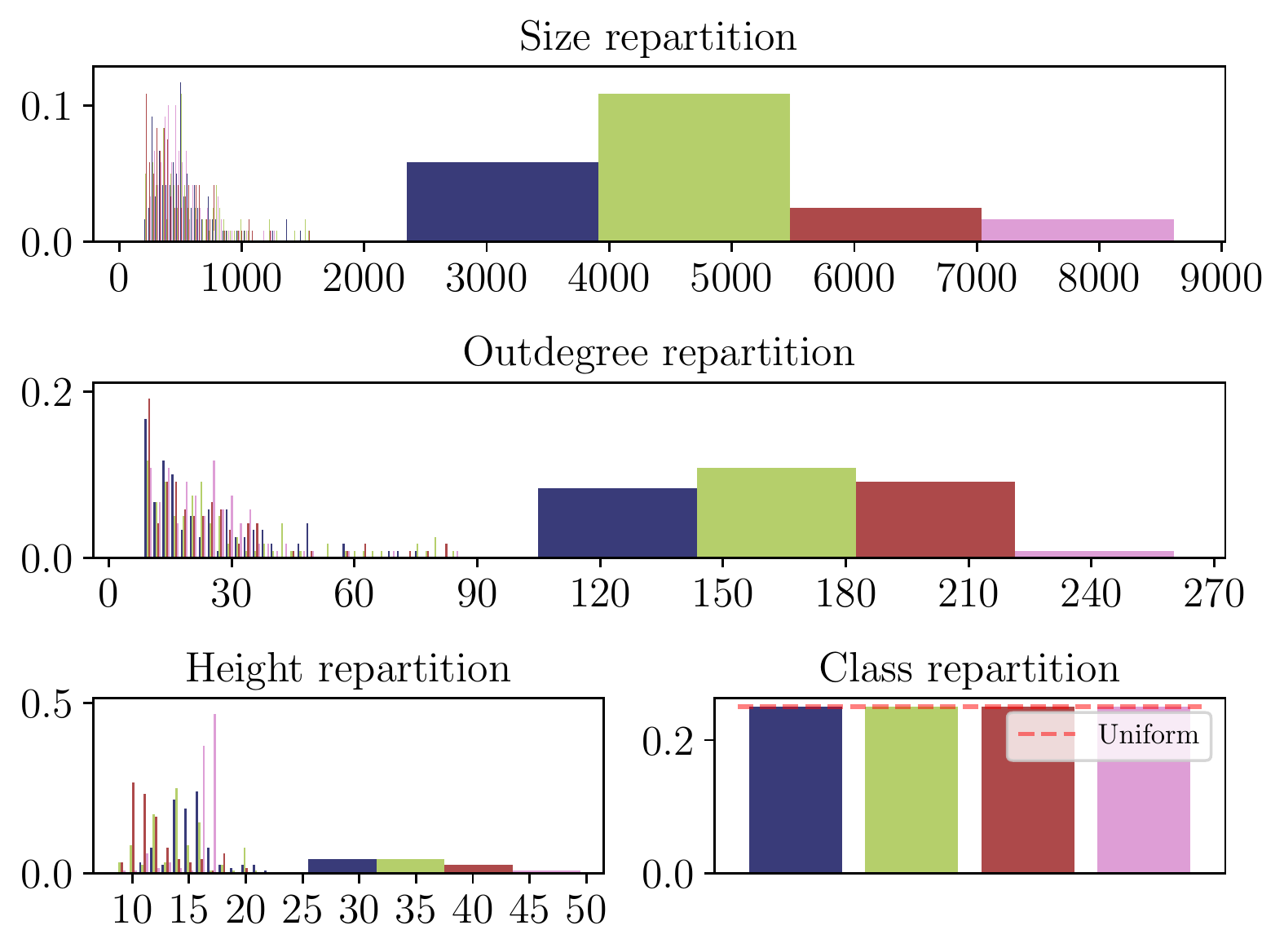}
    \caption{Description of a Wikipedia data set (480 trees).}
    \label{fig:wikipedia:repartition}
\end{figure}

Classification results over the $50$ databases are displayed in Fig.\,\ref{fig:wikipedia:results}. Discriminance weighting achie\-ves highly better results than exponential weighting, with all metrics greater than $90\%$ on average from only $200$ training data. This points out that the language information exists in the structure of Wikipedia pages, whether they are considered as ordered or unordered trees, unlike what intuition as well as subtree kernel with exponential weighting suggest. It should be added that the variance of all metrics seem to decrease with the size of the training data set when using discriminance.

\begin{figure}[h]
    \centering
        \includegraphics[width = 0.48\textwidth]{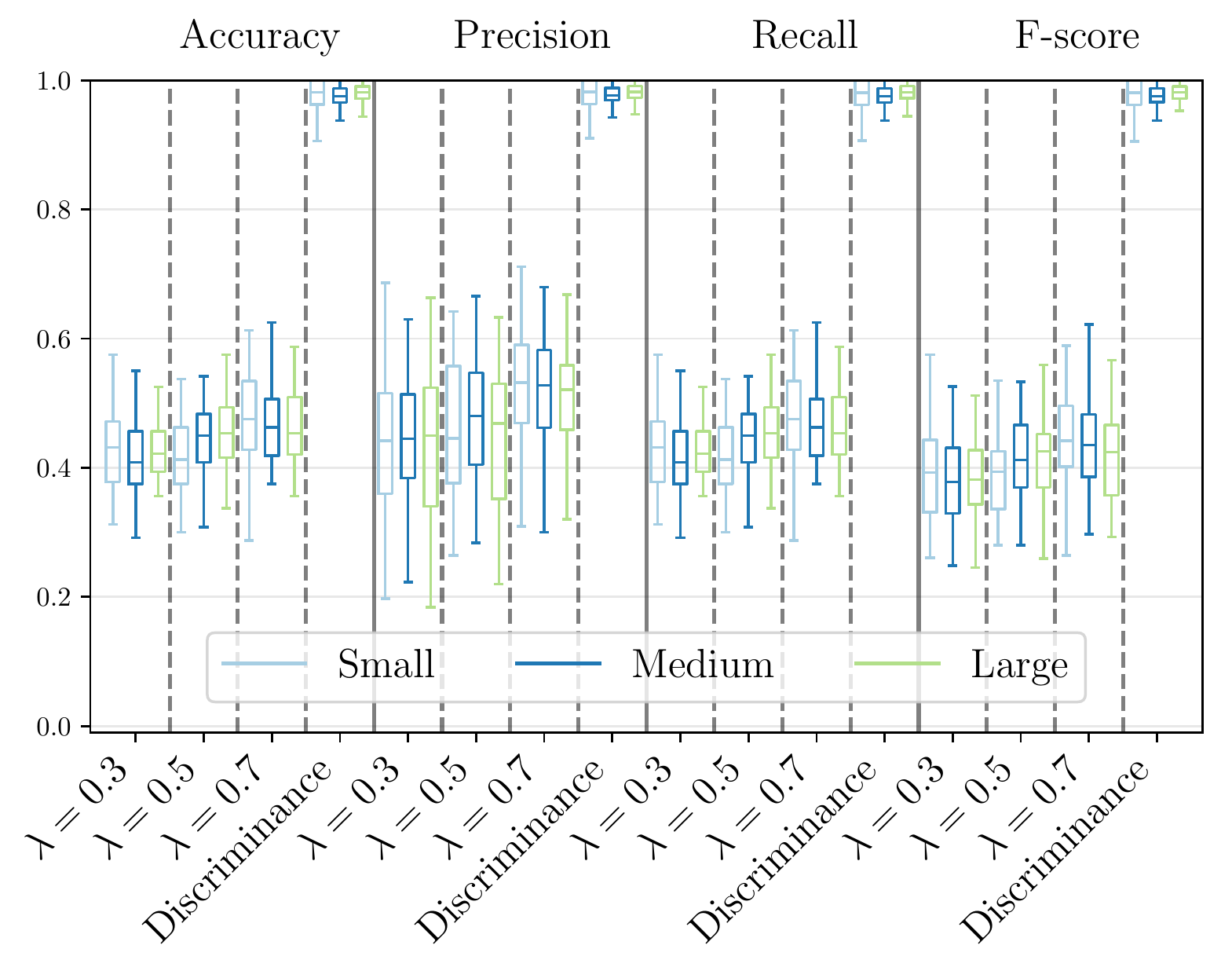}\quad
        \includegraphics[width = 0.48\textwidth]{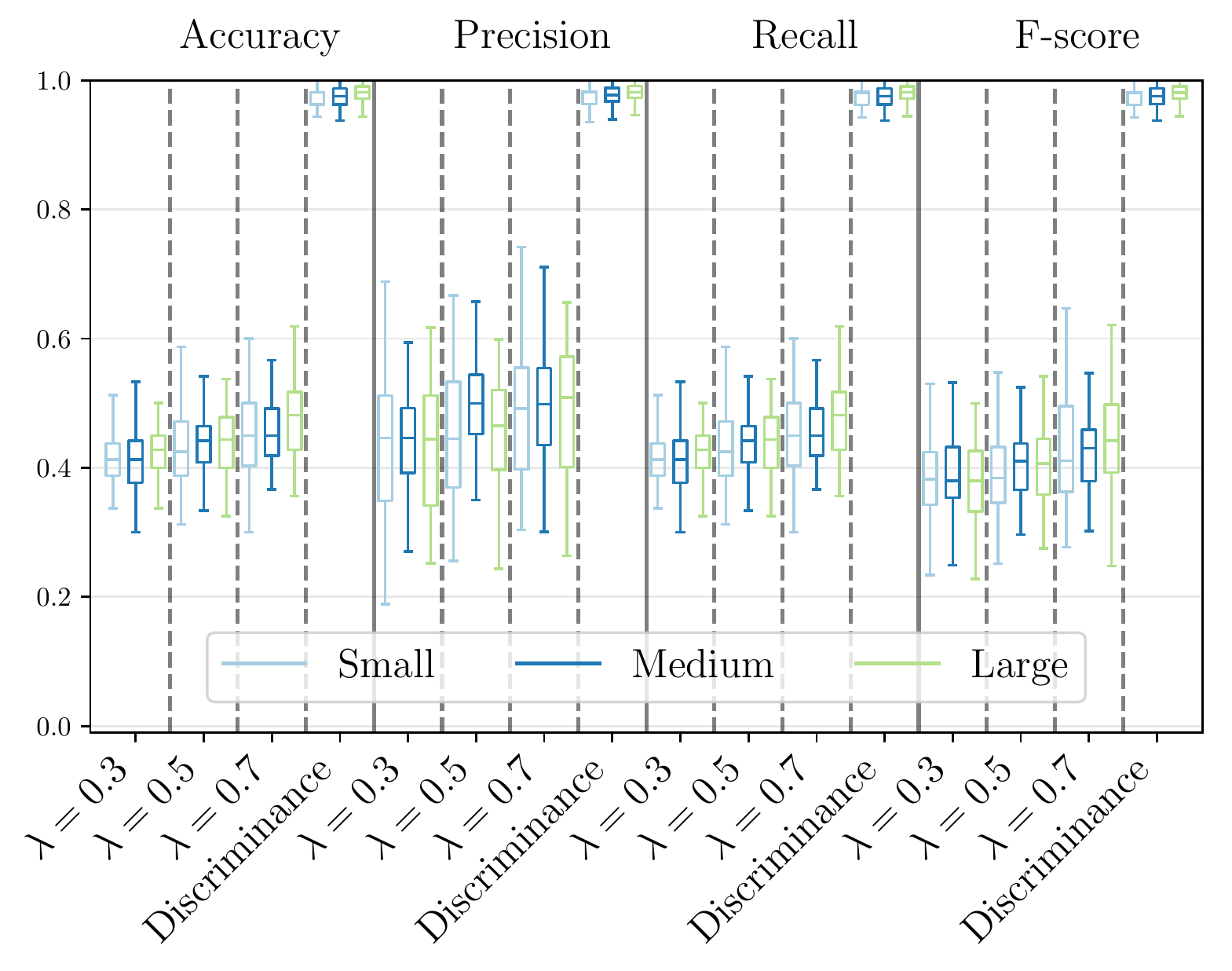}
    \caption{Classification results for the $50$ Wikipedia databases as ordered (left) and unordered (right) trees. $\lambda$ values stands for exponential decay weight of the form $\tau\mapsto\lambda^{\height(\tau)}$. The colors of the boxplot indicates, for each $\text{size}\in\lbrace \text{small}, \text{medium}, \text{large}\rbrace$, the results obtained for the classification of $\data{pred}$ from $\data[size]{train}$.}
    \label{fig:wikipedia:results}
\end{figure}

These numerical results show the great interest of the discriminance weight, in particular with respect to an exponential weight decay. Nevertheless, it should be compelling in this context to understand the classification rule learned by the algorithm. Indeed, this could lead to explain how the information of the language is present in the topology of the article.

\paragraph{Comprehensive learning and data visualization} When a learning algorithm is efficient for a given prediction problem, it is interesting to understand which features are significant. In the subtree kernel, the features are the subtrees appearing in all the trees of all the classes.
Looking at \eqref{eq:K:subtrees}, the contribution of any subtree $\tau$ to the subtree kernel with discriminance weighting is the product of two terms: the discriminance weight $w_\tau$ quantifies the ability of $\tau$ to discriminate a class, while $\kappa(\num_\tau(T_1),\num_\tau(T_2))$ evaluates the similarity between $T_1$ and $T_2$ with respect to $\tau$ through the kernel $\kappa$. As explained in Section~\ref{sec:weight}, if $w_\tau$ is close to $1$, $\tau$ is an important feature in the prediction problem.

As shown in Section~\ref{sec:dag}, DAG reduction provides a tool to compress a data set without loss. We recall that each vertex of the DAG represents a subtree appearing in the data. Consequently, we propose to visualize the important features on the DAG of the data set where the radius of the vertices is an increasing function of the discriminance weight. In addition, each vertex of the DAG can be colored as the class that it helps to discrimine, either positively (the vertex of the DAG corresponds to a subtree that is present almost only in the trees of this class), or negatively. This provides a visualization at a glance of the whole data set that highlights the significant features for the underlying classification problem. We refer the reader to Fig.\,\ref{fig:wikipedia:visu} for an application to one of our data sets. Thanks to this tool, we have remarked that the subtree corresponding to the License at the bottom of any article highly depends on the language, and thus helps to predict the class.

\afterpage{%
\begin{landscape}
\begin{figure}[h]
\centering
\includegraphics[width =1.27\textheight]{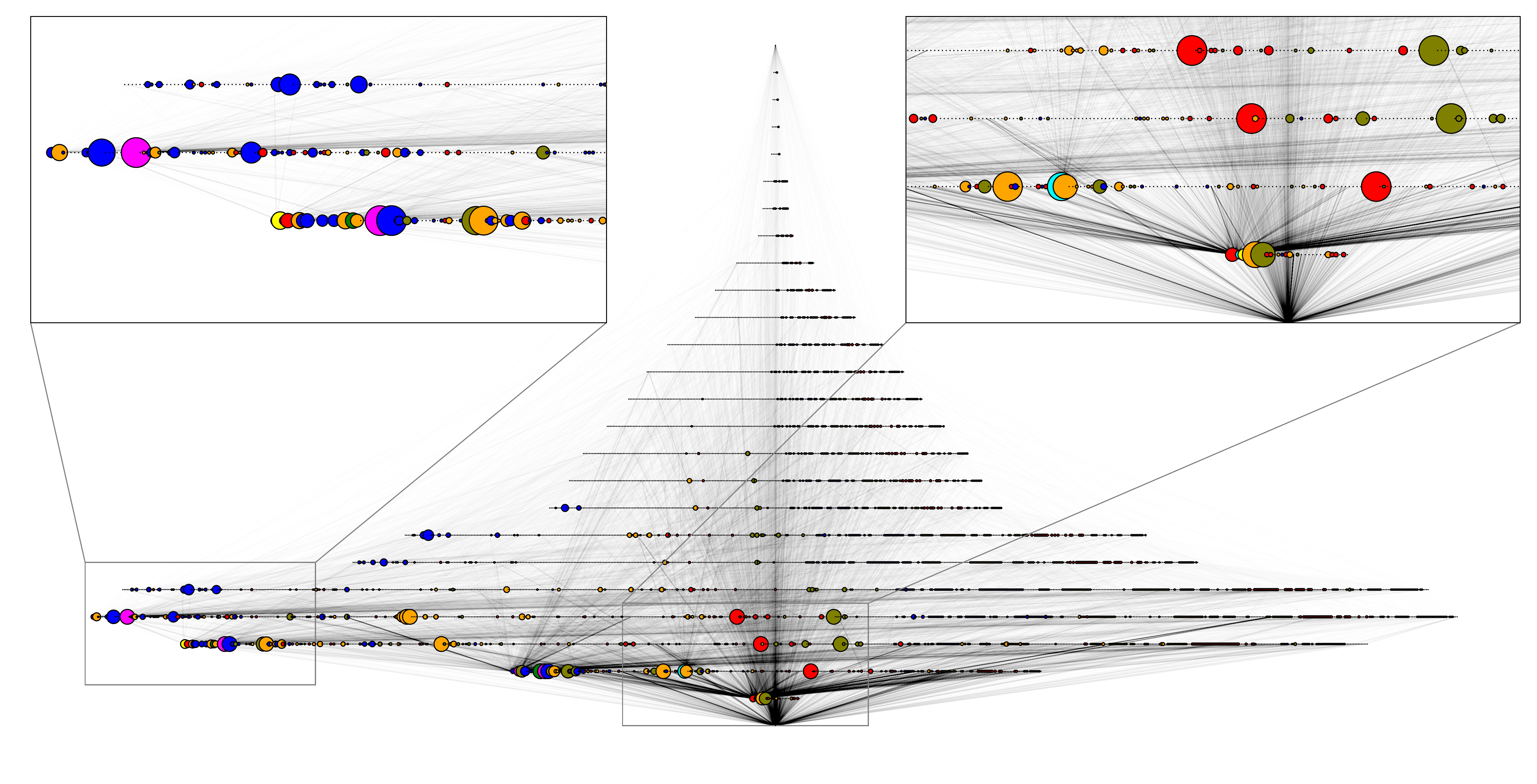}
\begin{tikzpicture}

\definecolor{cyan}{rgb}{0.0, 1.0, 1.0}
\definecolor{red}{rgb}{1.0, 0.0, 0.0}
\definecolor{yellow}{rgb}{1.0, 1.0, 0.0}
\definecolor{olive}{rgb}{0.5019607843137255, 0.5019607843137255, 0.0}
\definecolor{blue}{rgb}{0.0, 0.0, 1.0}
\definecolor{orange}{rgb}{1.0, 0.6470588235294118, 0.0}
\definecolor{magenta}{rgb}{1.0, 0.0, 1.0}
\definecolor{green}{rgb}{0.0, 0.5019607843137255, 0.0}

\def\x{3}

\draw[fill=cyan]  (0,0) circle (5pt) node[right,xshift=5pt] {en: absence};
\draw[fill=blue]  (0,-0.7) circle (5pt) node[right,xshift=5pt] {en: presence};

\draw[fill=green]  (\x,0) circle (5pt) node[right,xshift=5pt] {es: absence};
\draw[fill=olive]  (\x,-0.7) circle (5pt) node[right,xshift=5pt] {es: presence};

\draw[fill=magenta]  (2*\x,0) circle (5pt) node[right,xshift=5pt] {de: absence};
\draw[fill=red]  (2*\x,-0.7) circle (5pt) node[right,xshift=5pt] {de: presence};

\draw[fill=yellow]  (3*\x,0) circle (5pt) node[right,xshift=5pt] {fr: absence};
\draw[fill=orange]  (3*\x,-0.7) circle (5pt) node[right,xshift=5pt] {fr: presence};
\end{tikzpicture}
\caption{Visualization of one data set $\data{} = \data[medium]{train}\cup\data{pred}$ of unordered trees among the $30$ Wikipedia databases. Each vertex $\nu\in\red(\data{})$ is scaled according to $\smstp(1-\delta_\nu)$ so that the largest vertices are those that best discriminate the different classes. For each $\nu$, we find the class $k$ such that $\rho_\nu$ has minimal distance to either $e_k$ or $\overline{e}_k$. If it is $e_k$, we say that $\nu$ discriminates by its presence, and if it is $\overline{e}_k$, $\nu$ discriminates by its absence. We color $\nu$ following this distinction according to the legend, where ``en'' is for English language, ``de'' for German, ``fr'' for French, and ``es'' for Spanish.}
\label{fig:wikipedia:visu}
\end{figure}
\end{landscape}
}

\paragraph{Distribution of discriminance weights} To provide a better understanding of our results, we have analyzed in Fig.\,\ref{fig:wikipedia:analysis} the distribution of discriminance weights of one of our large training data sets. It shows that the discriminance weight behaves on average as a shifted exponential. Considering the great performance achieved by the discriminance weight, this illustrates that exponential weighting presented in the literature is indeed a good idea, when setting $w_\bullet=0$ as shown in Subsection~\ref{ss:weight:leaves} or suggested in \citep[6 Experimental results]{vishwanathan2002fast}. However, a closer look to the distribution in Fig.\,\ref{fig:wikipedia:analysis} (left) reveals that important features in the kernel are actually outliers: relevant information is both far from the average behavior and scarce. To a certain extent and regarding these results, discriminance weight is the second order of the exponential weight.

\begin{figure}[h]
\centering
\includegraphics[width = 0.9\textwidth]{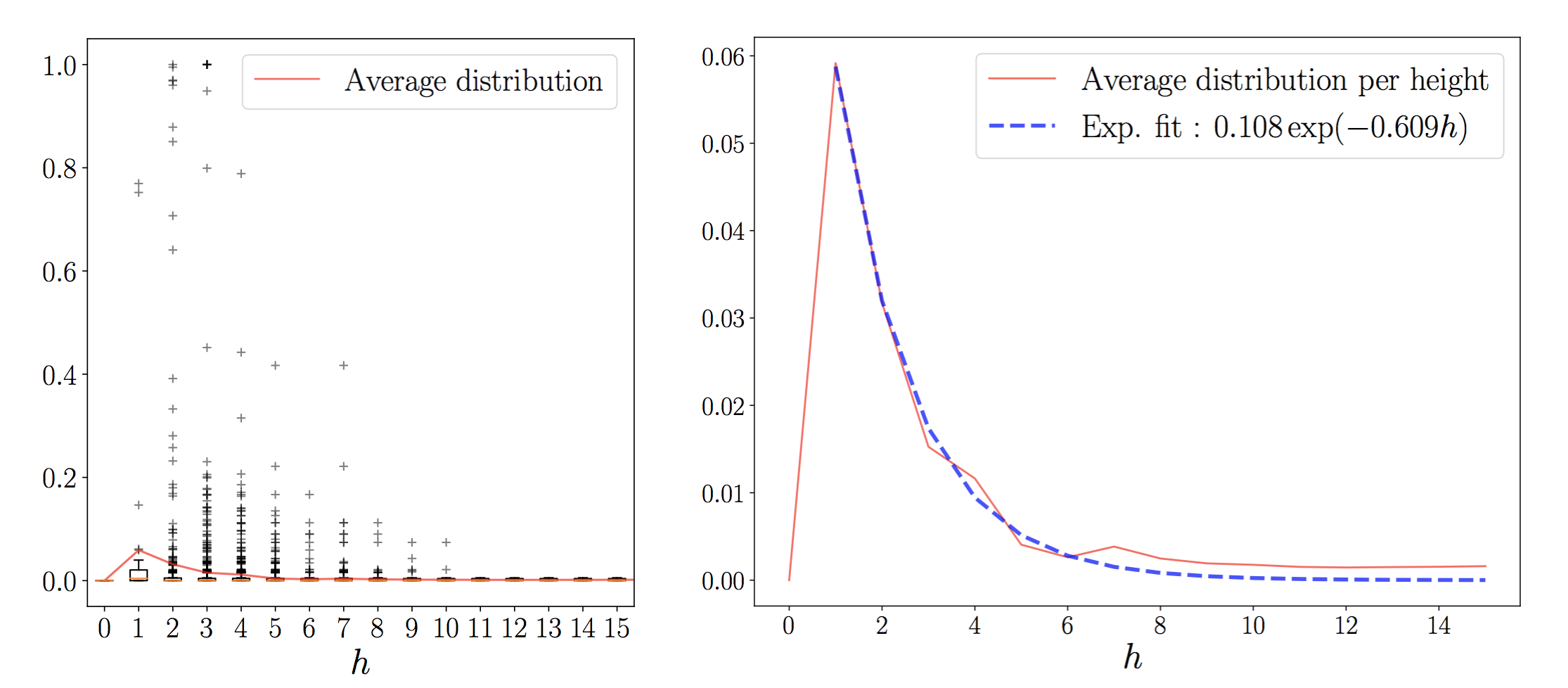}
\caption{Estimation of the distribution of the discriminance weight function $h \mapsto \lbrace w_\nu : \height(\nu) = h, \nu \in \red(\data{}) \rbrace$ from one large training Wikipedia data set of unordered trees (left) and fit of its average behavior (in \textcolor{red}{red}) to an exponential function (in \textcolor{blue}{blue}).
All ordered and unordered data sets show a similar behavior.}
\label{fig:wikipedia:analysis} 
\end{figure}

\begin{remark}\label{rem:interest:wiki}
The classification problem considered in this subsection may seem unrealistic as ignoring the text information is obviously counterproductive in the prediction of the language of an article. Nevertheless, this application example is of interest for two main reasons. First, this prediction problem is difficult as shown by the bad results obtained from the subtree kernel with exponential weights (see Fig.\,\ref{fig:wikipedia:results}). As highlighted in Fig.\,\ref{fig:wikipedia:visu} and \ref{fig:wikipedia:analysis} (left), the subtrees that can discriminate the classes are very unfrequent and diverse (in terms of size and structure), so difficult to be identified. On a different level, as Wikipedia has a very large corpus of pages, it provides a practical tool to test our algorithms and investigate the properties of our approach. Indeed, we can virtually create as many different data sets as we want by randomly picking articles, ensuring that we avoid overfitting.
\end{remark}

\subsection{Markup documents data sets}\label{ss:markup}

We present and analyze in this subsection three data sets obtained from markup documents. 

\paragraph{INEX 2005 and 2006} These data sets originate from the INEX competition \citep{denoyer2007report}. There are \verb=XML= documents, that we have been considering as ordered and unordered in our experiments. INEX 2005 is made of 9\,630 documents arranged in 11 classes, whereas INEX 2006 has 18 classes for 12\,107 documents. For INEX 2005, classes can be split into two groups of trees with similar characteristics, as shown in Fig.\,\ref{fig:inex:repartition} (left). However, inside each group, all trees are alike. In the case of INEX 2006, no special group seems to emerge from topological characteristics of the data, as pointed out in Fig.\,\ref{fig:inex:repartition} (right).

The classification results are depicted in Fig.\,\ref{fig:inex:results}, for both data sets, and with trees considered successively as ordered and unordered. For INEX 2005, both exponential decay and discriminance achieve similar good performance. However, for INEX 2006, neither of them are able to achieve significant results. Actually, discriminance performs slightly worse than exponential decay. From these results we deduce that subtrees do not seem to form the appropriate substructure to capture the information needed to properly classify the data.

\begin{figure}[h]
    \centering
    \includegraphics[width = 0.48\textwidth]{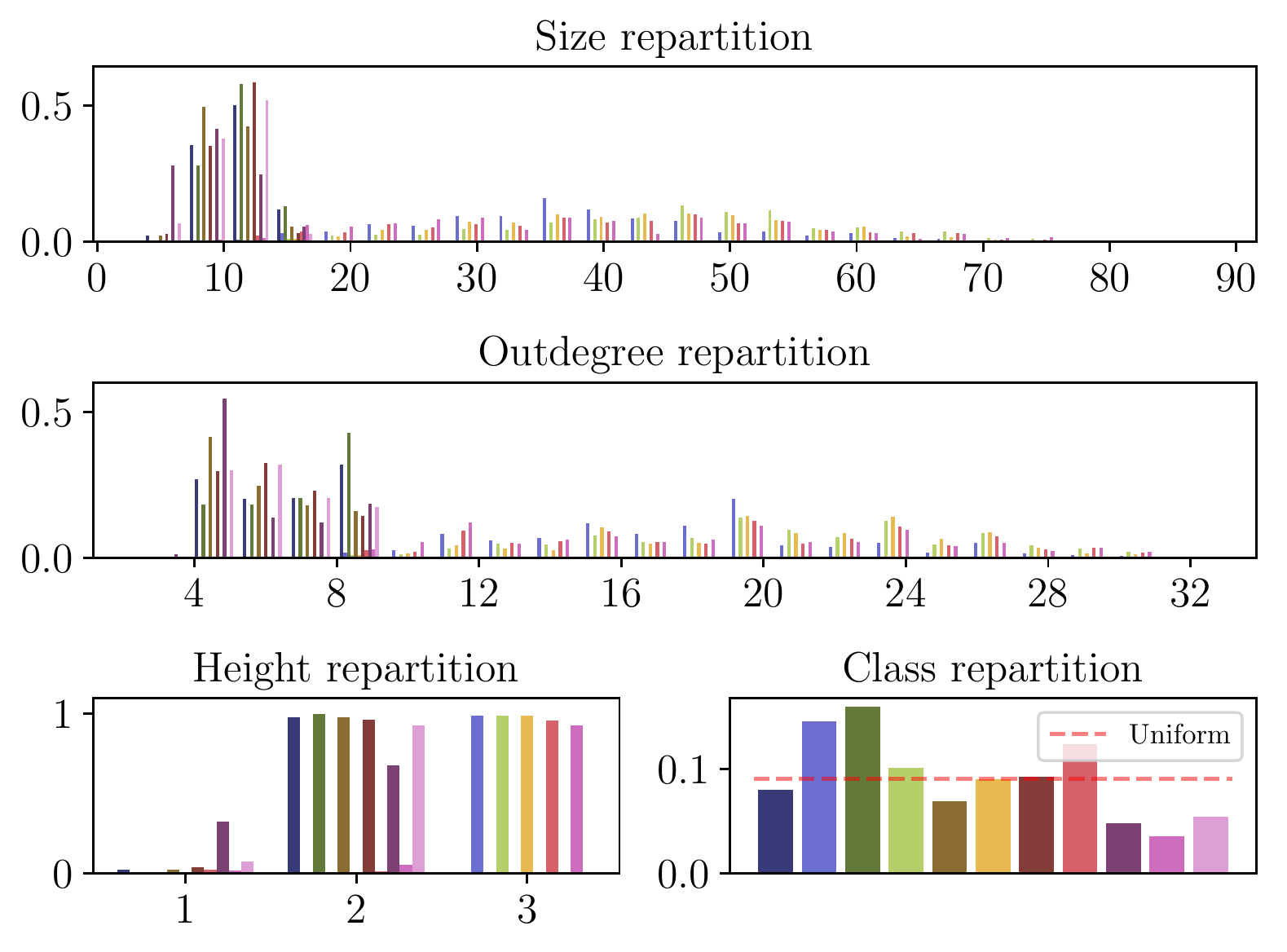}\quad
    \includegraphics[width = 0.48\textwidth]{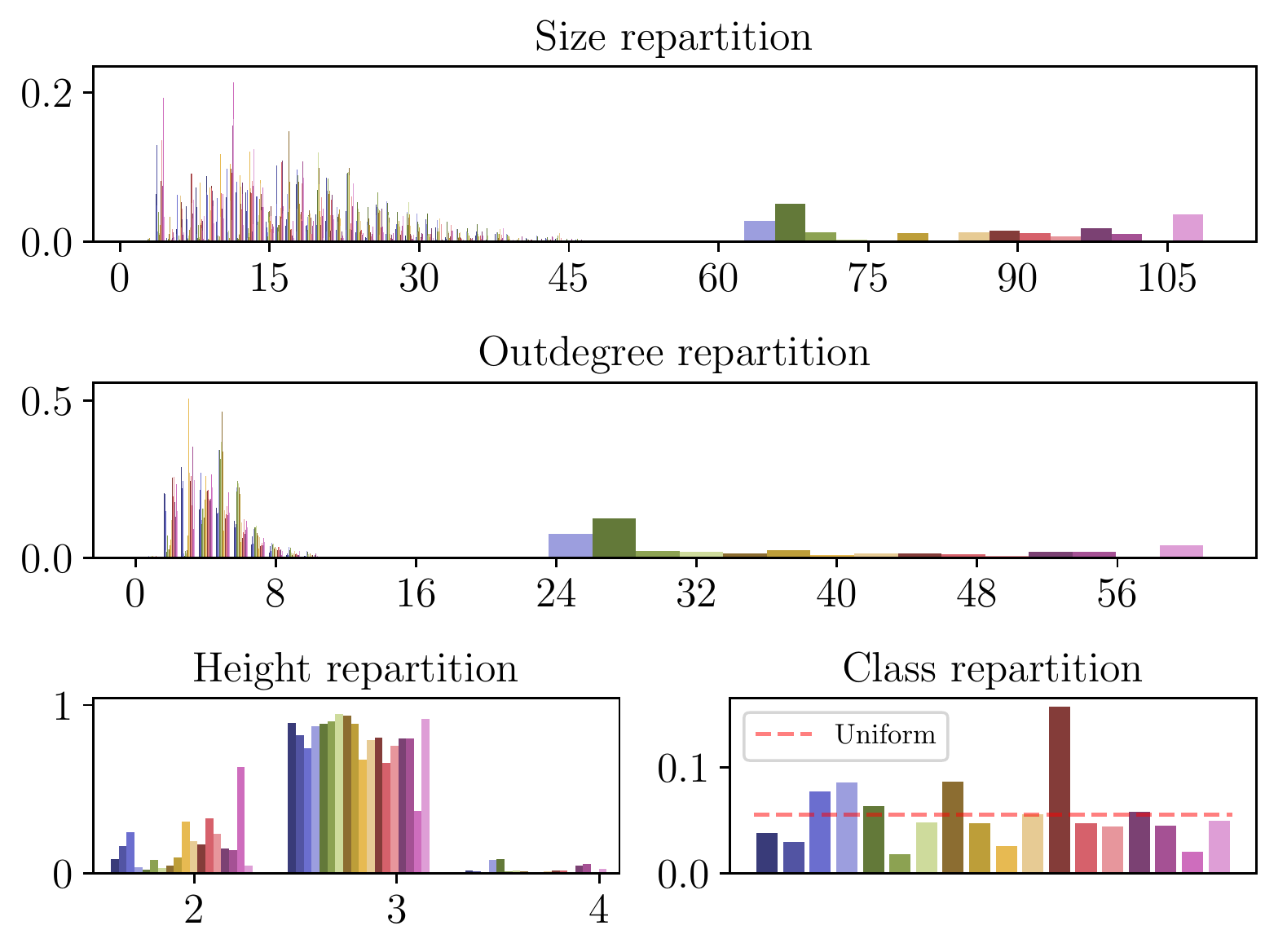}
    \caption{Description of INEX 2005 (9\,630 trees, left) and INEX 2006 (12\,107 trees, right) data sets.}
    \label{fig:inex:repartition}
\end{figure}

\begin{figure}[h]
    \centering
        \includegraphics[width = 0.48\textwidth]{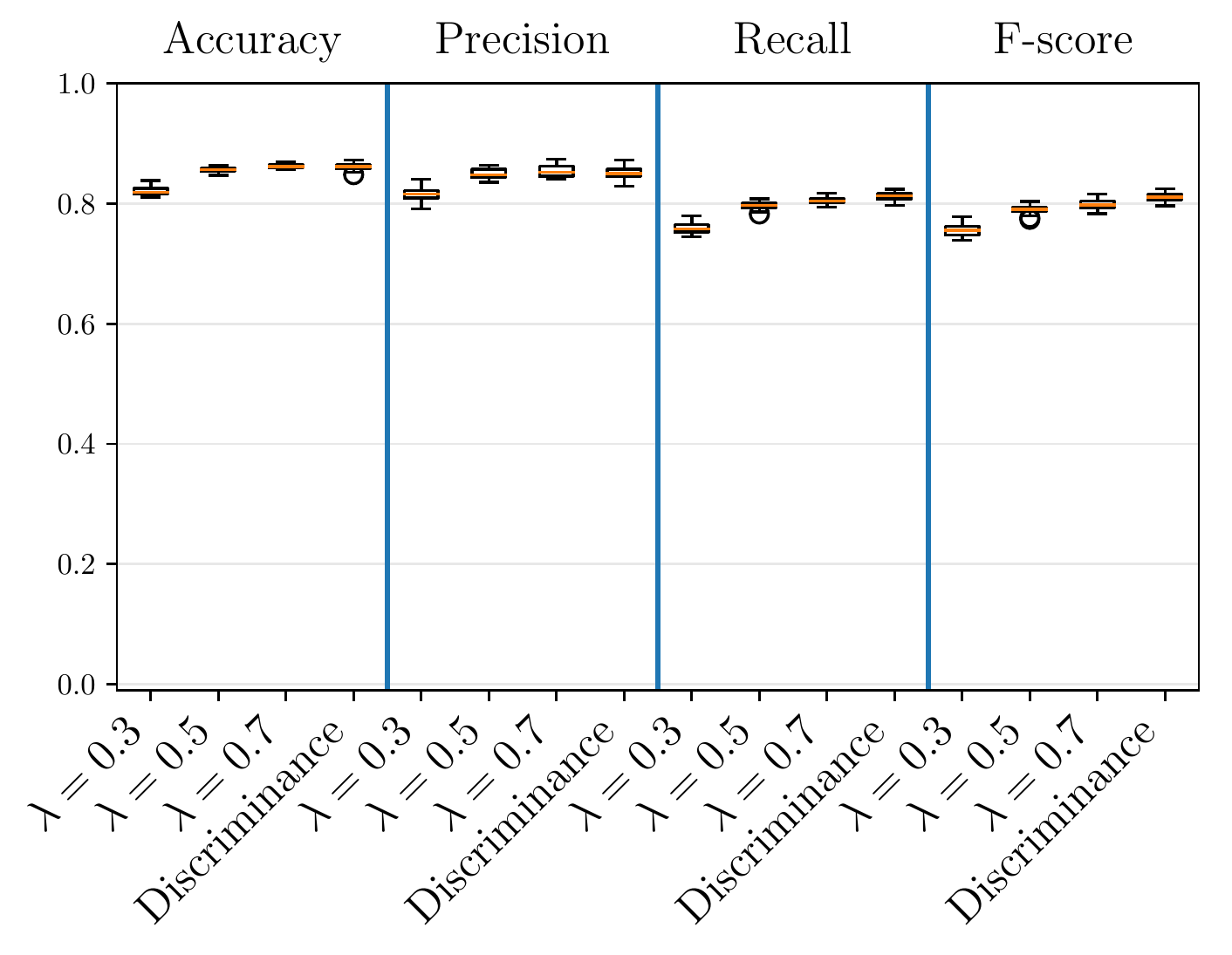}\quad
        \includegraphics[width = 0.48\textwidth]{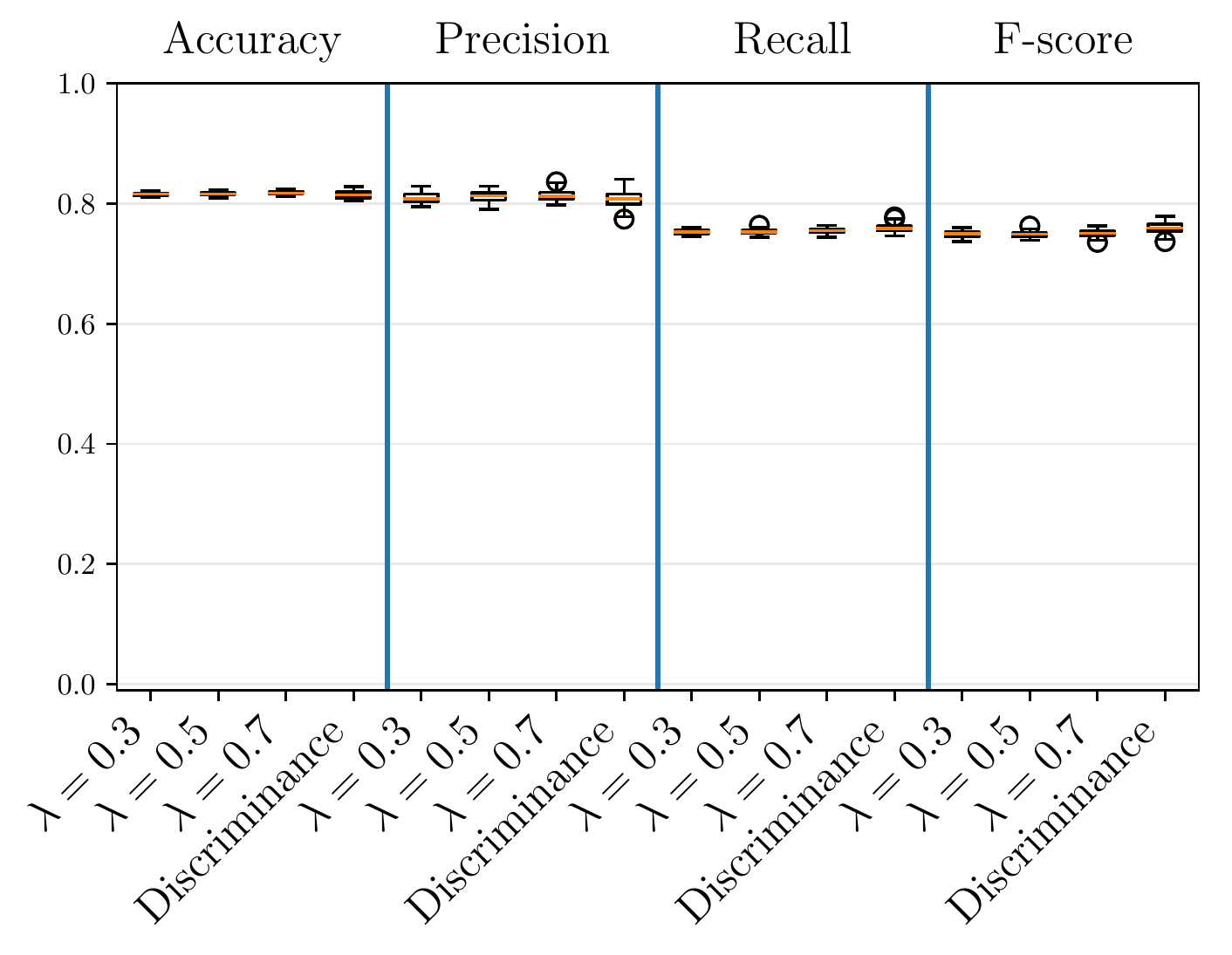}\\
        \includegraphics[width = 0.48\textwidth]{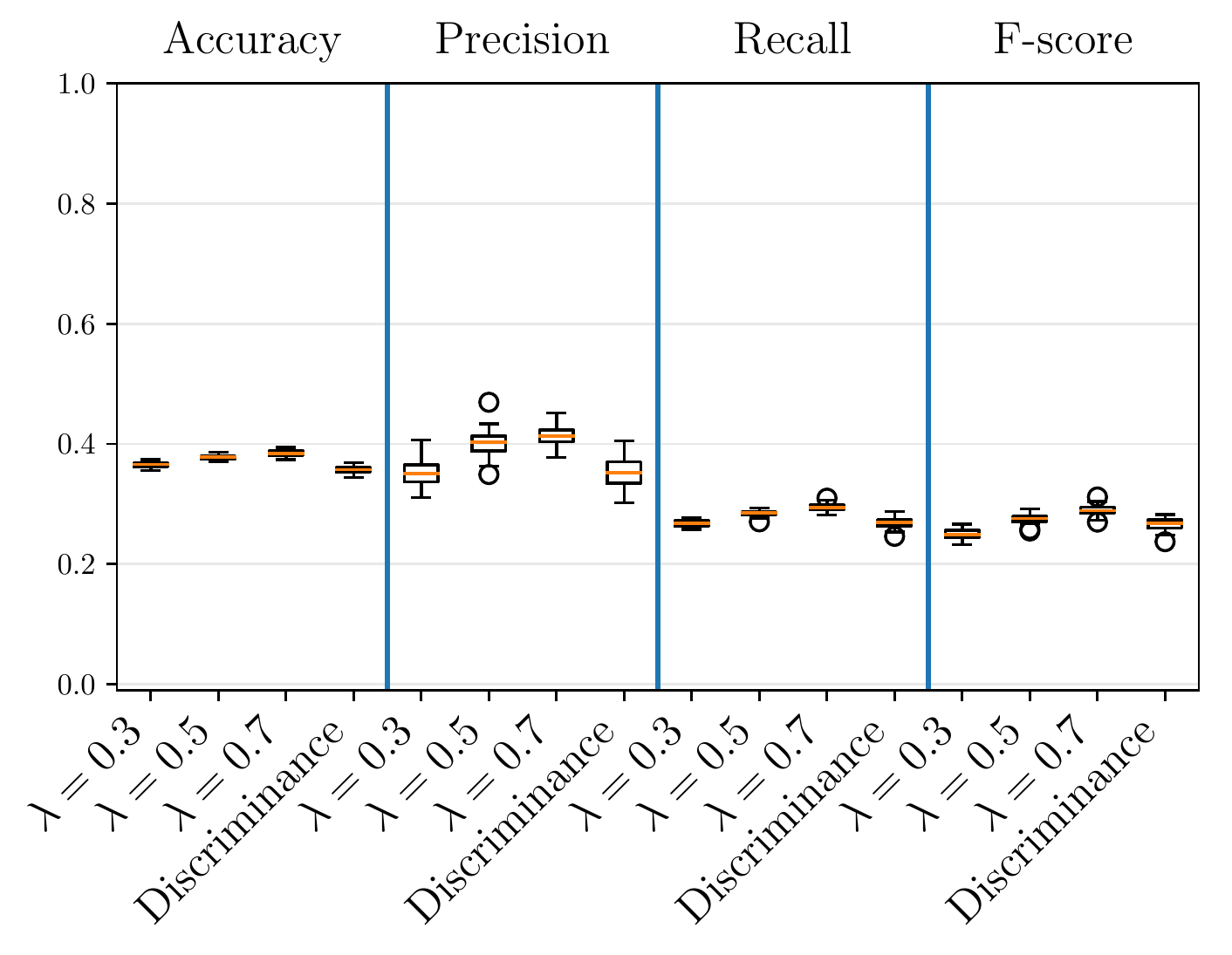}\quad
        \includegraphics[width = 0.48\textwidth]{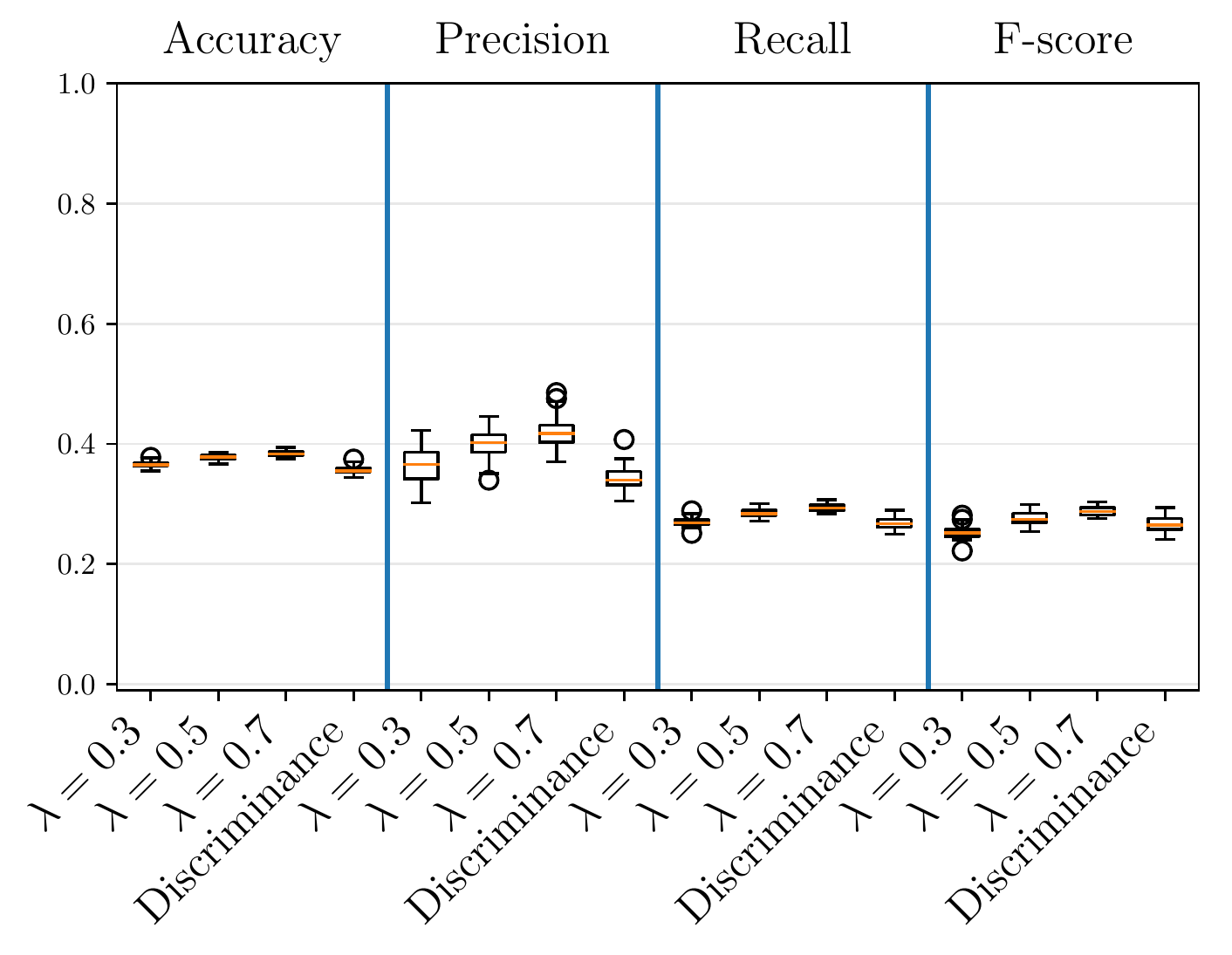}
     \caption{Classification results for INEX 2005 (top) and INEX 2006 (bottom) as ordered (left) and unordered (right) trees.}
    \label{fig:inex:results}
\end{figure}

\paragraph{Videogame sellers} We manually collected, for two major websites selling videogames\footnote{\url{https://store.steampowered.com} and \url{https://www.gog.com} (last accessed in April 2020)}, the URLs of the top 100 best-selling games, and converted them into ordered labeled trees. As webpages might seem similar to some extent, the trees are actually very different, as highlighted in Fig.\,\ref{fig:videogames:repartition}. We found that the subtree kernel retrieves this information as, for both exponential decay and discriminance weights, we achieved 100\% of correct classifications in all our tests.

\begin{figure}[h]
    \centering
    \includegraphics[width = 0.48\textwidth]{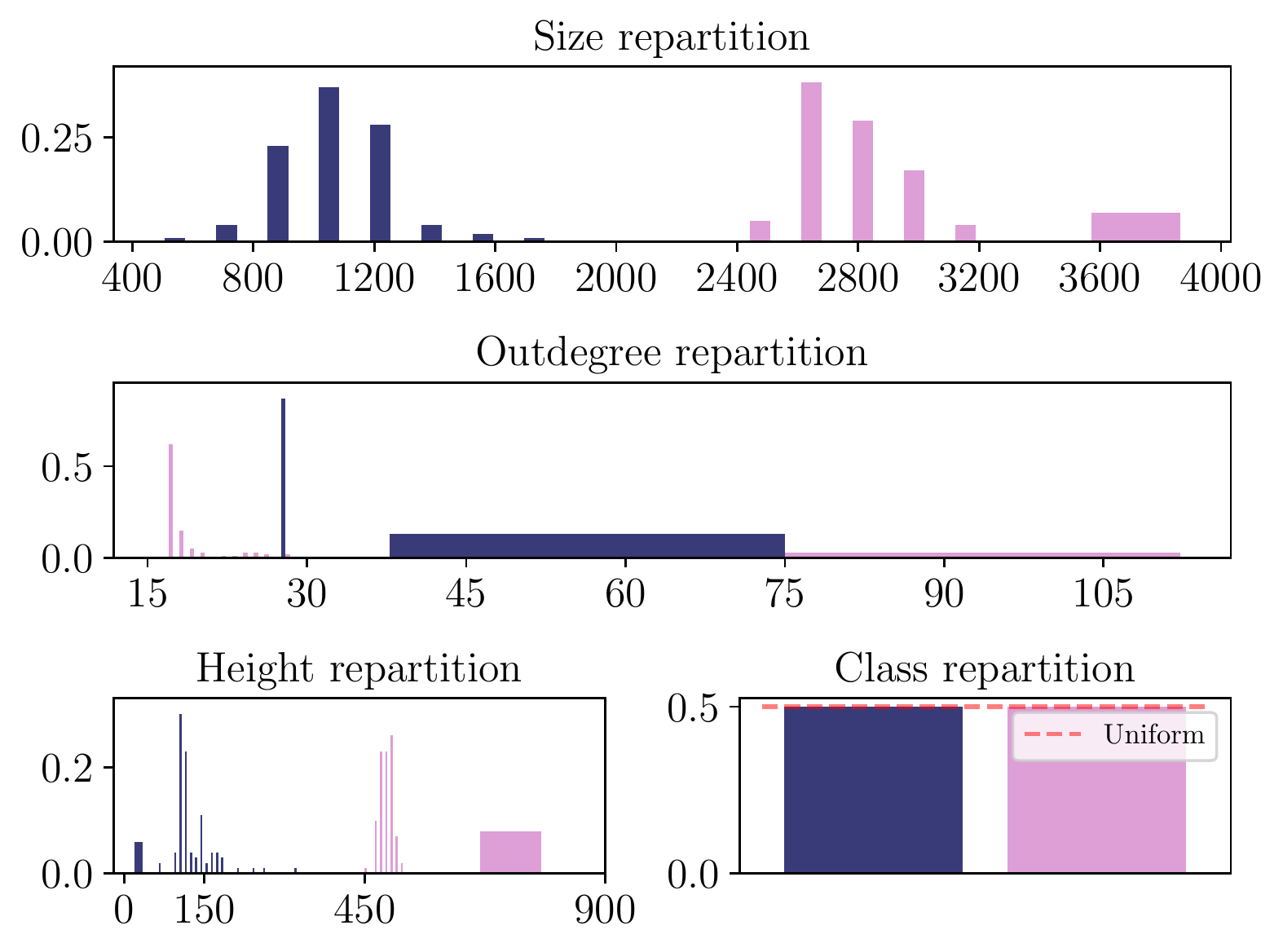}
    \caption{Description of the videogame sellers data set (200 trees).}
    \label{fig:videogames:repartition}
\end{figure}

\subsection{Biological data sets}\label{ss:biol}

In this subsection, three data sets from the literature are analyzed, all related to biological topics.

\paragraph{Vascusynth} The Vascusynth data set from \cite{cmig2010,ij2011} is composed of 120 unordered trees that represent blood vasculatures with different bifurcations numbers. In a tree, each vertex has a continuous label describing the radius $r$ of the corresponding vessel. We have discretized these continuous labels in three categories: small if $r<0.02$, medium if $0.02\leq r<0.04$ and large if $r\geq0.04$ (all values are in arbitrary unit). We split up the trees into three classes, based on their bifurcation number. Based on Fig.\,\ref{fig:vascu} (left), we can distinguish between the three classes by looking only at the size of trees. Contrary to the videogame sellers data set that had the same property, the classification does not achieve 100\% of good classification, as depicted in Fig.\,\ref{fig:vascu} (right). On average, discriminance performs better than the other weights, despite having a larger variance. This is probably due to the small size of the data set, as the discriminance is learned only with around $13$ trees per class.

\begin{figure}[h]
        \centering
        \includegraphics[width = 0.48\textwidth]{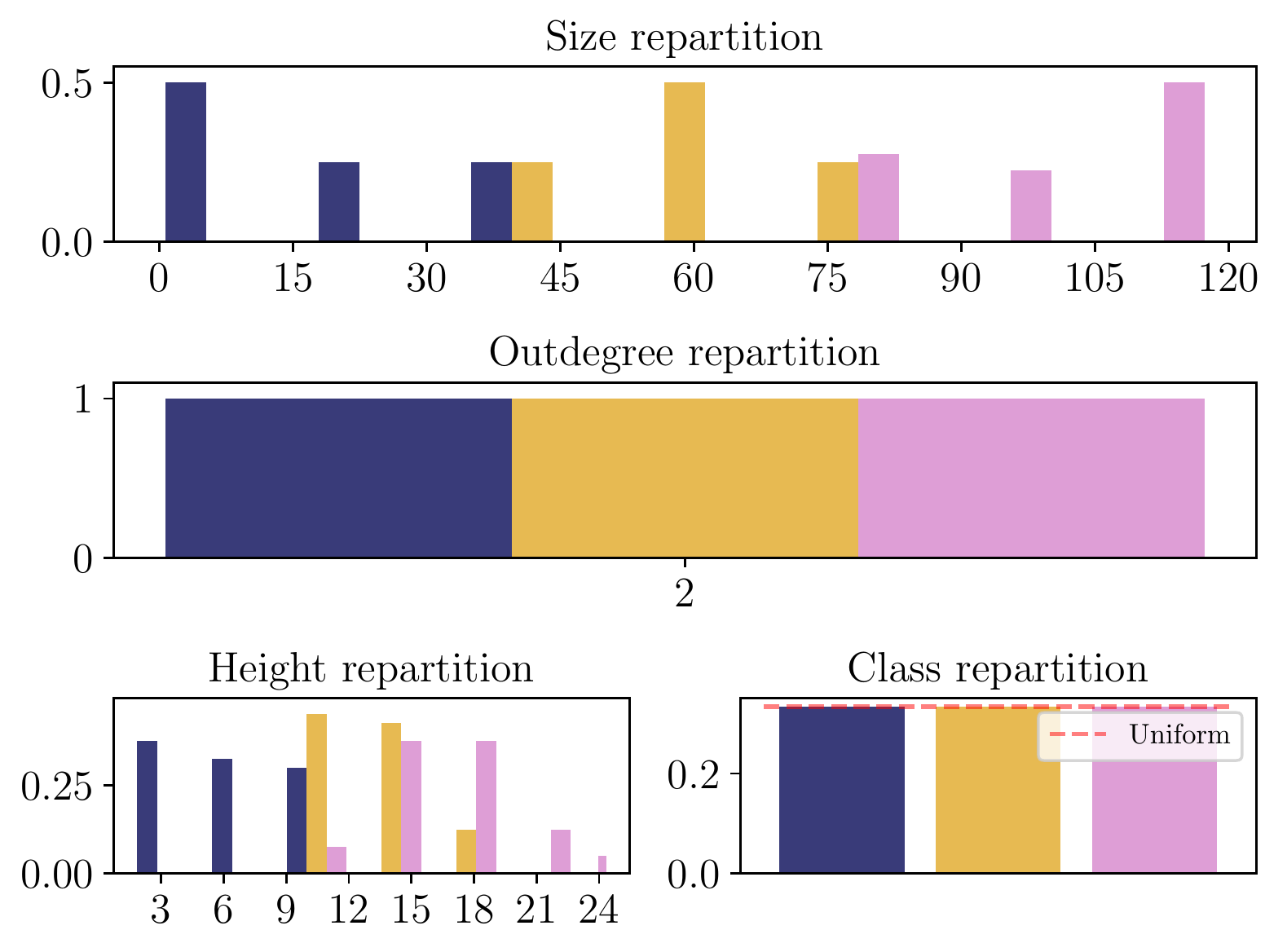}\quad
        \includegraphics[width = 0.48\textwidth]{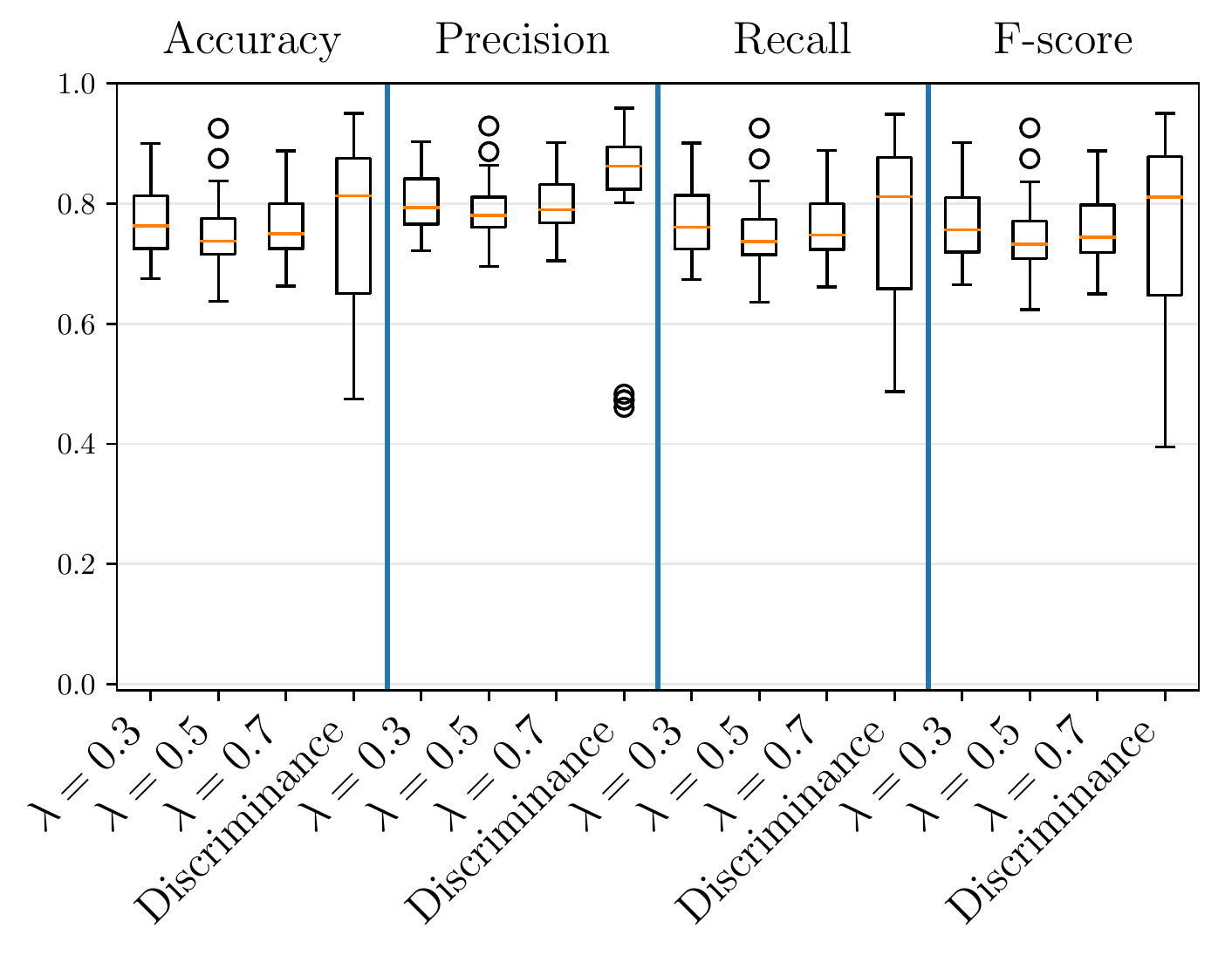}
            \caption{Description of the Vascusynth data set (120 trees, left) and classification results (right).}
    \label{fig:vascu}
\end{figure}

\paragraph{Hicks et al. cell lineage trees} Across cellular division, tracking the lineage of a single cell naturally defines a tree. In a recent article,  \citet{hicks2019maps} have been investigating the variability inside cell lineages trees of three different species. From the encoding of the data that they have provided as a supplementary material\footnote{\url{https://doi.org/10.1101/267450} (last accessed in April 2020)}, we have extracted ordered unlabeled trees that are presented in Fig.\,\ref{fig:maps} (left). The data set contains, for two classes, trees of outdegree 0 (i.e., isolated leaves) that can be considered as noise. With respect to the exponential weight, the value of the kernel between such trees will be identical, whether they belong to the same class or to two different classes. They therefore contribute to reducing the kernel's ability to effectively discriminate between these two classes. On the other hand, the discriminance weight will assign them a zero value, ``de-noising'', in a way, the data. This observation may explain why discriminance weight achieves better results than exponential weight.

\begin{figure}[h]
        \centering
        \includegraphics[width = 0.48\textwidth]{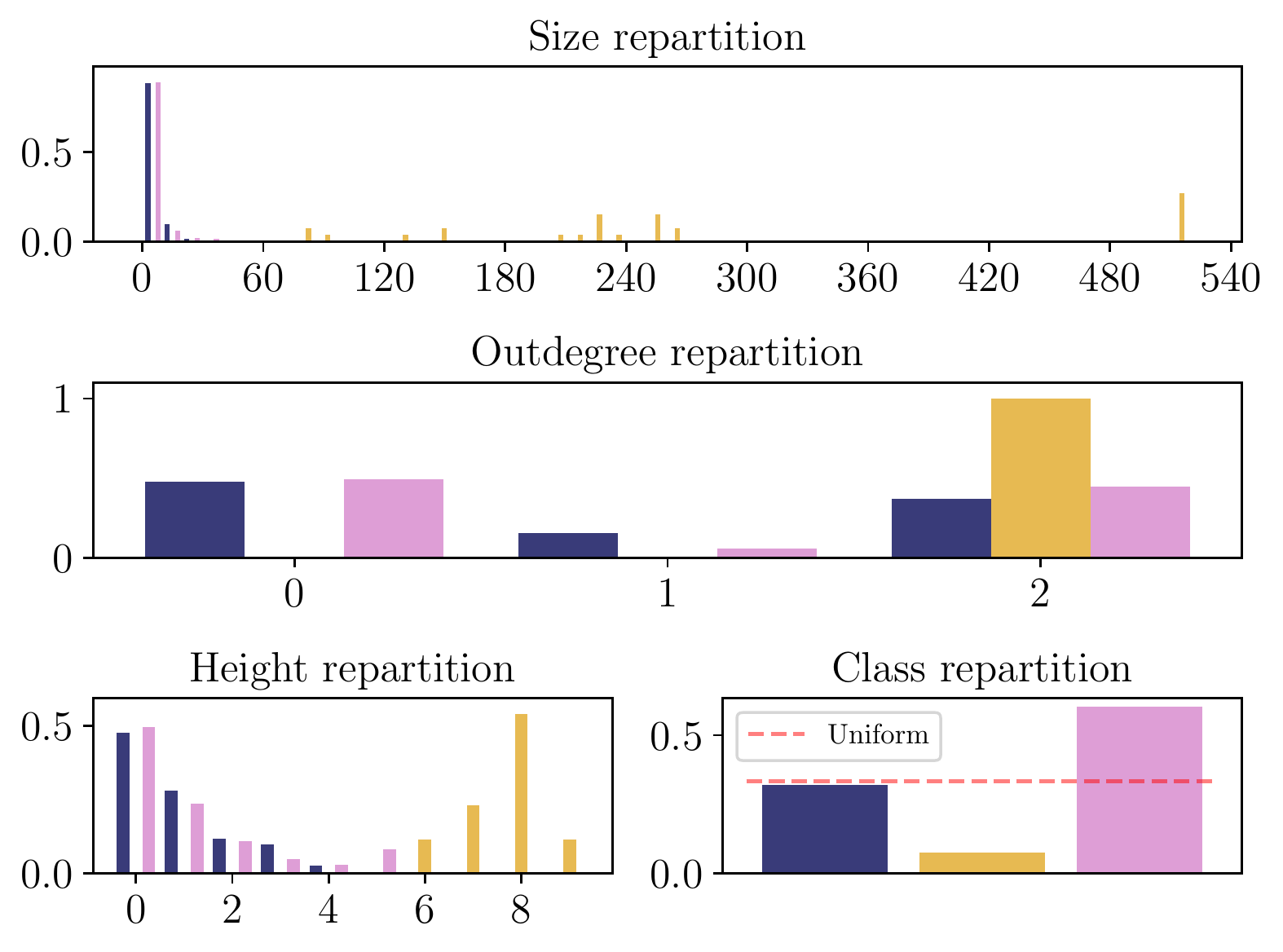}\quad
        \includegraphics[width = 0.48\textwidth]{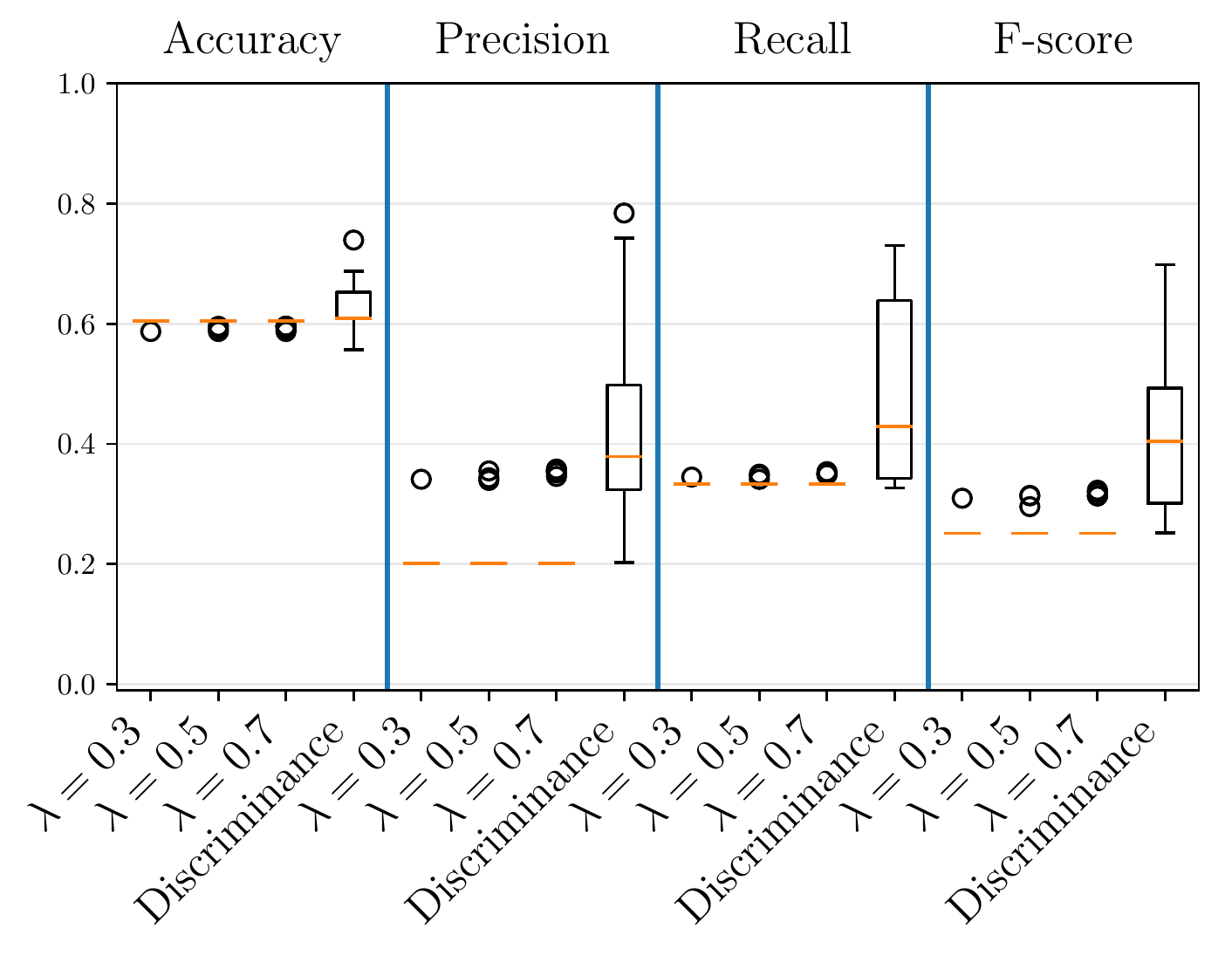}
            \caption{Description of the Hicks et al. data set (345 trees, left) and classification results (right).}
    \label{fig:maps}
\end{figure}

\paragraph{Faure et al. cell lineage trees}  \citet{faure2015algorithmic} have developed a method to construct cell lineage trees from microscopy and provided their data online\footnote{\url{https://bioemergences.eu/bioemergences/openworkflow-datasets.php} (last accessed in April 2020)}. We extracted 300 unordered and unlabeled trees, divided between three classes. It seems from Fig.\,\ref{fig:3species} (left) that one class among the three can be distinguished from the two others. Classification results can be found in Fig.\,\ref{fig:3species} (right): the discriminance weight performs better than the exponential weight, whatever the value of the parameter.

\begin{figure}[h]
    \centering
        \includegraphics[width = 0.48\textwidth]{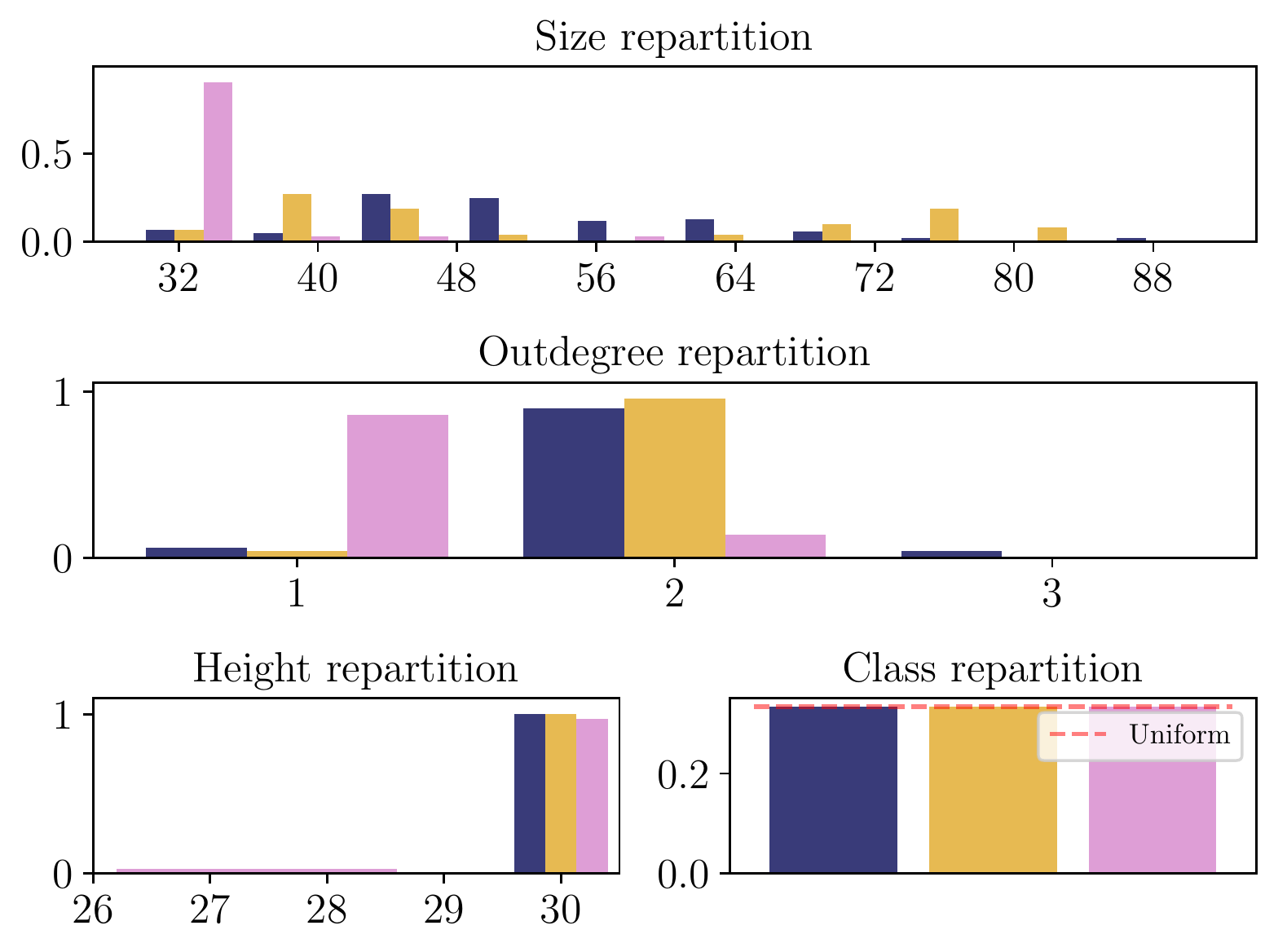}\quad
        \includegraphics[width = 0.48\textwidth]{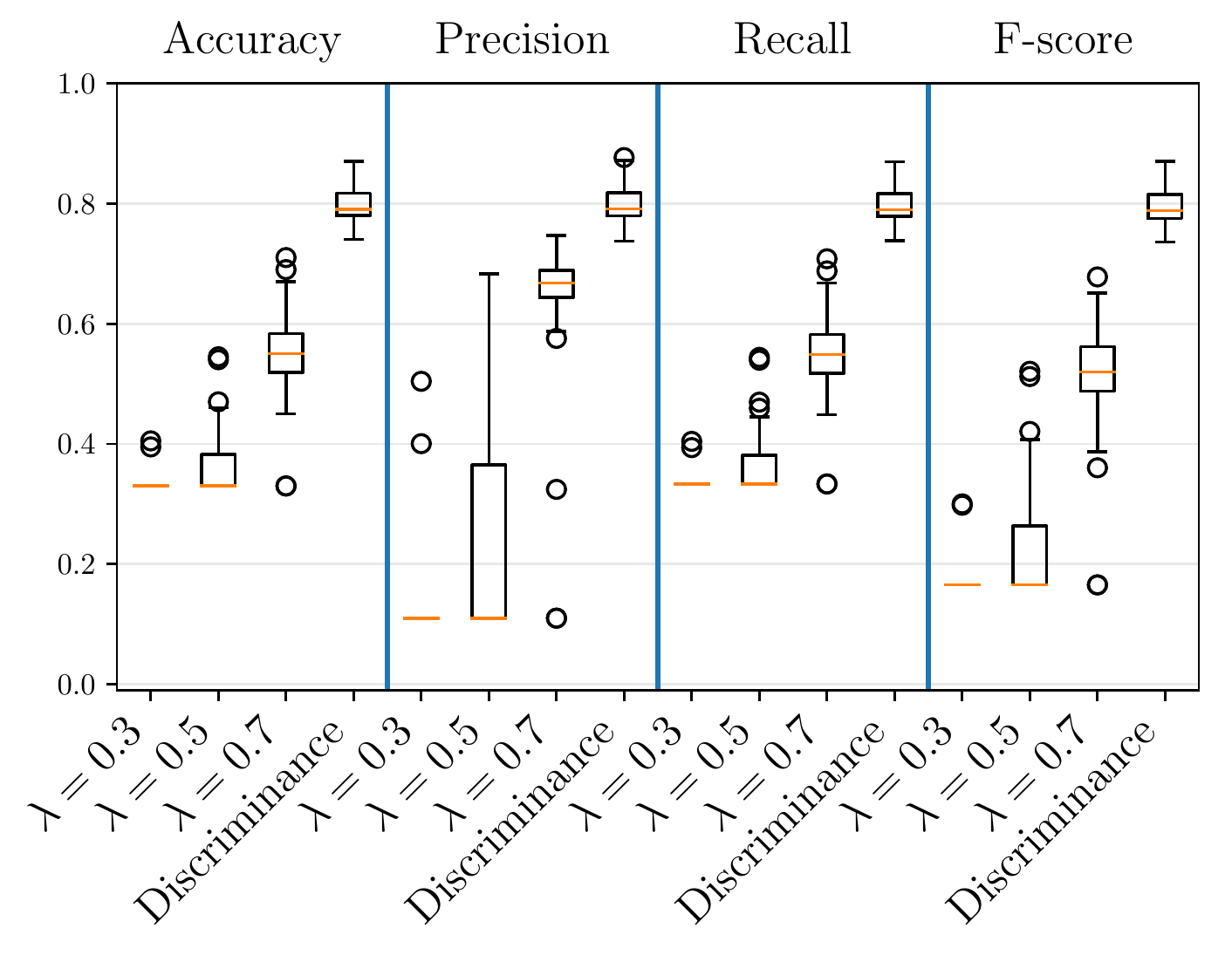}
            \caption{Description of the Faure et al. data set (300 trees, left) and classification results (right).}
    \label{fig:3species}
\end{figure}

\subsection{LOGML}\label{ss:logml}

The LOGML data set is made of user sessions on an academic website, namely the Rensselaer Polytechnic Institute Computer Science Department website\footnote{\url{https://science.rpi.edu/computer-science} (last accessed in April 2020)}, that registered the navigation of users across the website. 23\,111 unordered labeled trees are present, divided into two classes. The trees are very alike, as shown in Fig.\,\ref{fig:logml} (left), and the classification results of Fig.\,\ref{fig:logml} (right) are very similar to INEX 2005, where all weight functions behave similarly, without any advantage for the discriminance weight in terms of prediction.

\begin{figure}[h]
    \centering
        \includegraphics[width = 0.48\textwidth]{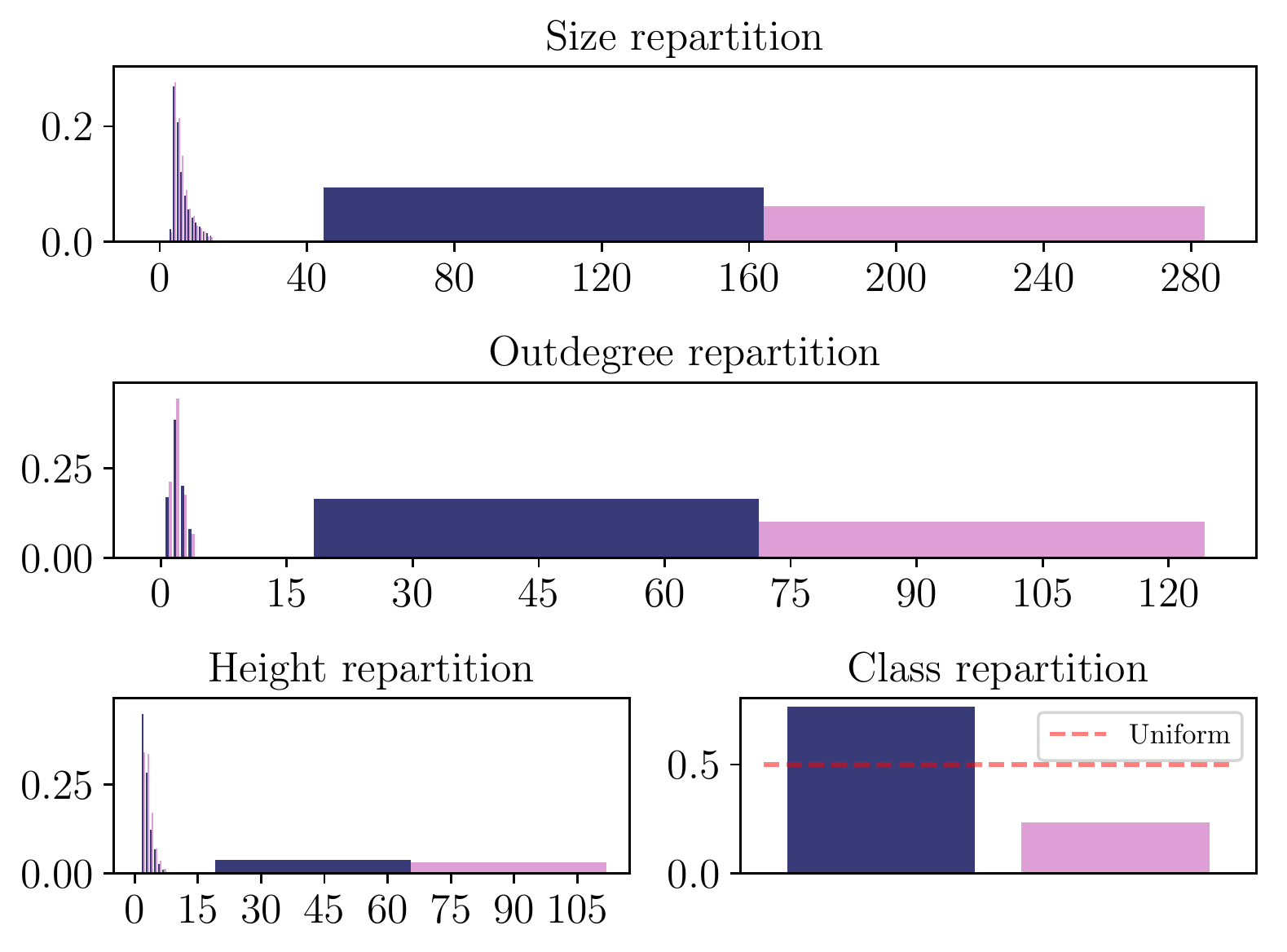}\quad
        \includegraphics[width = 0.48\textwidth]{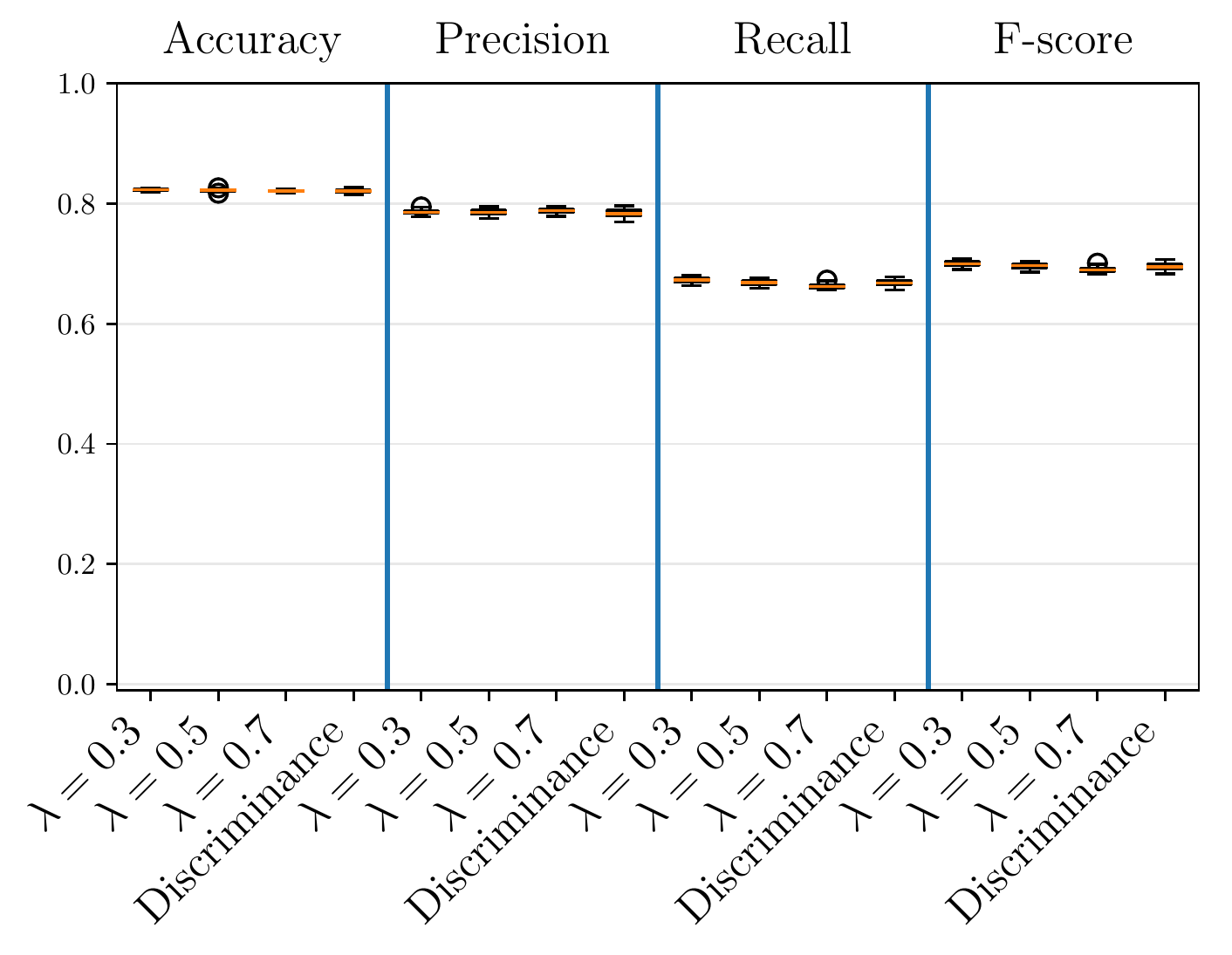}
            \caption{Description of the LOGML data set (23\,111 trees, left) and classification results (right).}
    \label{fig:logml}
\end{figure}


\section{Concluding remarks}\label{s:concl}

\subsection{Main interest of the DAG approach: learning the weight function}

In Section~\ref{sec:theory}, we have shown on a $2$-classes stochastic model that the efficiency of the subtree kernel is improved by imposing that the weight of leaves is null. As explained in Remark~\ref{rem:link:model:discr}, we conjecture that the weight of any subtree present in two different classes should be $0$. The main interest of the DAG approach developed in Section~\ref{sec:dag} is that it allows to learn the weight function from the data, as developed in Section~\ref{sec:weight} with the discriminance weight function. Our method has been implemented and tested in Section~\ref{sec:data} on eight real data sets with very different characteristics that are summed up in Table~\ref{tab:data:summary}.

As a conclusion of our experiments, we have analyzed the relative improvement in prediction obtained with the discriminance weight against the best exponential weight in order to show both the importance of the weight function and the relevance of the method developed in this paper. More precisely, for each data set and each classification metric, we have calculated
$$RI = \frac{\text{Metric}_{\text{discr}} - \max(\text{Metric}_\lambda)}{\max(\text{Metric}_\lambda)},$$
from the average values of the different metrics. The results are presented in Fig.\,\ref{fig:improvement}. We have found that, except in one case, discriminance behaves as good as exponential weight decay and even performs better in most of the data sets. Furthermore, one can observe a kind of trend, where the relative improvement decreases when the number of trees in the training data set is increasing, which proves the great interest of the discriminance to handle small data sets, provided that (i) the problem is difficult enough that the exponential weights are not already high performing, as it is the case in the Videogames sellers data set, and (ii) the data set is not too small, as for Vascusynth. Indeed, as the discriminance is learned independently from the SVM, one must have enough training data to divide them efficiently. Nevertheless, it should be noted that, in the framework of the DAG approach, results from the discriminance weight can be obtained much faster due to the fact that the Gram matrices are estimated from one half of the training data set, while learning the discrimance is very fast as it can be done in one traversal of the DAG (see time-complexity presented in Remark~\ref{rem:discrim}). Finally, we have investigated on a single example some properties of the discriminance, discovering that it can be interpreted as a second-order exponential weight, as well as a method for visualizing the important features in the data.

\begin{table}[h]
\centering
\begin{tabular}{c|cccccccc}
data set &\rot{Wikipedia}& \rot{Videogames}&\rot{INEX 2005}&\rot{INEX 2006}& \rot{Vascusynth}&\rot{Hicks et al.}&\rot{Faure et al.}&\rot{LOGML}\\
Ord. / Unord. & Both& Ord. & Both& Both& Unord.&Ord.&Unord.  &Unord. \\
labeled & \xmark & \checkmark& \checkmark&\checkmark&\checkmark&\xmark&\xmark&\checkmark\\
Number of trees & 160\,--\,320 &200&9\,630 &12\,107&120&345&300&23\,111 \\
Number of classes & 4 &2&11&18&3&3&3&2
\end{tabular}
\caption{Summary of the 8 data sets.}
\label{tab:data:summary}
\end{table}

\begin{figure}[h]
    \centering
    \includegraphics[width = 0.8\textwidth]{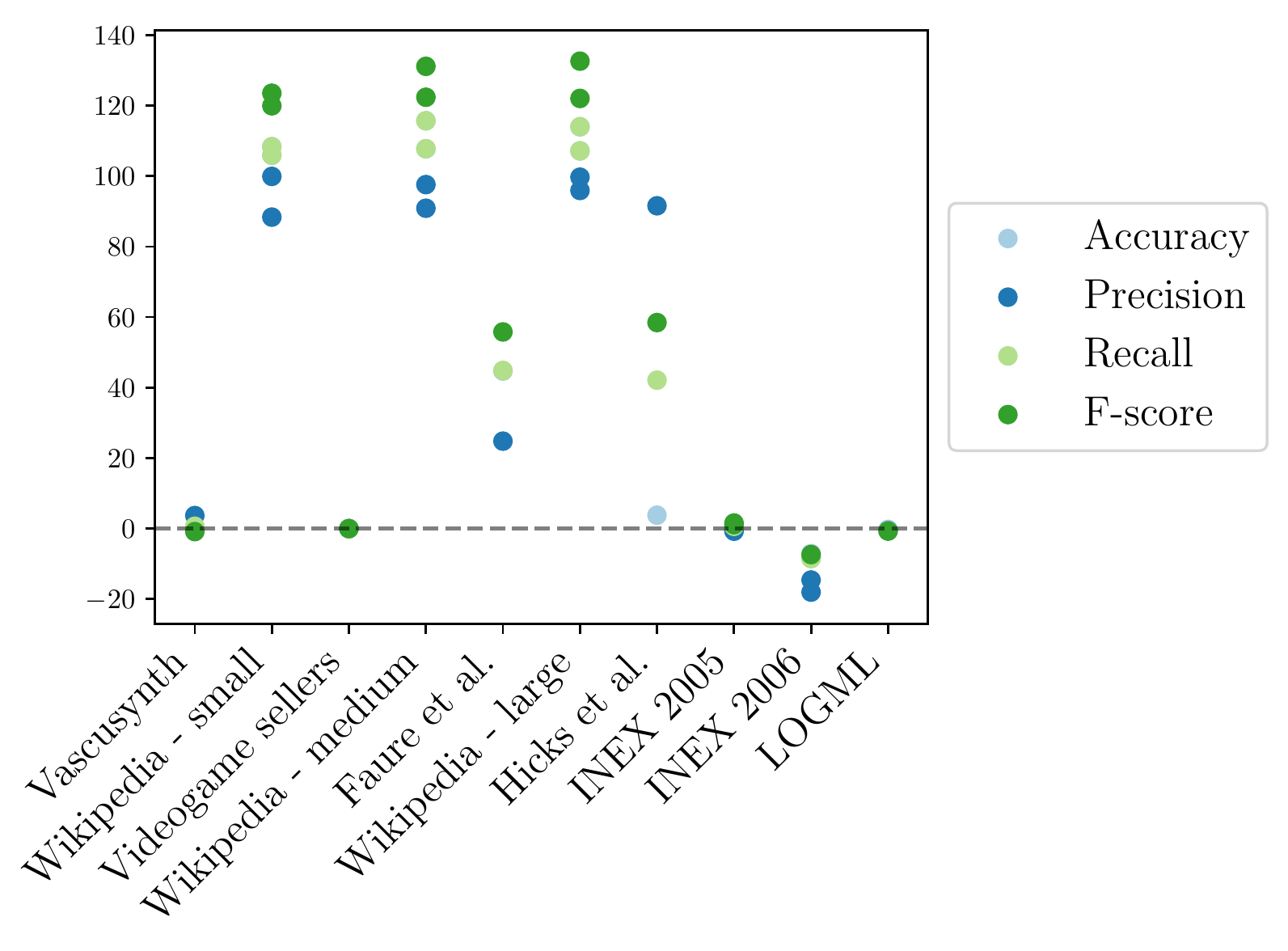}
    \caption{Relative improvement $RI$ (in percentage) of the discriminance against the best value of $\lambda$ for all data sets (sorted by increasing number of trees in the training data set) and all metrics.}   
    \label{fig:improvement}
\end{figure}

\subsection{Interest of the DAG approach in terms of computation time}

As shown in Fig.\,\ref{fig:3species} (right), the exponential decay classification results for the Faure et al. data set are very dependent on the value chosen for the parameter $\lambda$. In this case, it can be interesting to tune this parameter and estimate its best value with respect to a prediction score. This requires to compute the Gram matrices from different weight functions. We present in Fig.\,\ref{fig:3species:time} the computation time required to compute the Gram matrices from a given number of values of the parameter. As expected from the theoretical results, we observe a linear dependency: the intercept corresponds to the computation time required to compute and annotate the DAG reduction, while the slope is associated with the time required to compute the Gram matrices, which is proportional to the average of $\O(\min(\#T_i, \#T_j))$ (see Remark~\ref{rem:compl:sub}). This can be compared to the time-complexity of the algorithm developed in \cite{vishwanathan2002fast} which is the average of $\O(\#T_i+\#T_j)$. Consequently, the corresponding computation times should be proportional to at least twice the slope that we observe with the DAG approach. This shows another interest of our method that is not related to the discriminance weight function. It should be faster to compute several repetitions of the subtree kernel from the DAG approach than from the previous algorithm \citep{vishwanathan2002fast} provided that the number of repetitions is large enough.

\begin{figure}[h]
    \centering
    \begin{minipage}[c]{0.48\textwidth}
 \includegraphics[width = \textwidth]{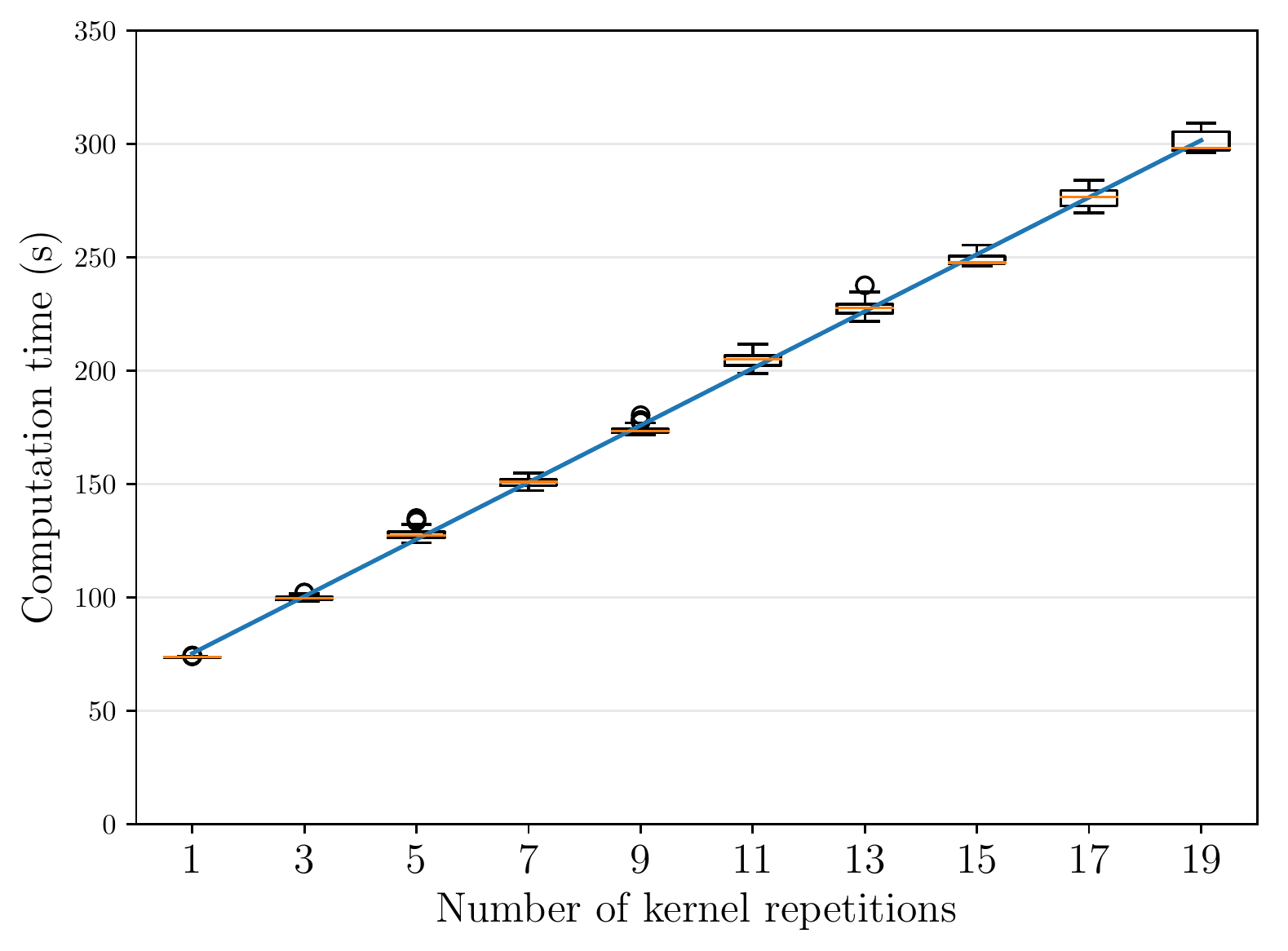}
\end{minipage}\hfil
\begin{minipage}[c]{0.48\textwidth}
            \caption{Computation time required to compute several repetitions of the kernel on the Faure et al. data set. All those calculations have been repeated $50$ times for each number of repetitions. The intercept corresponds to the DAG compression of the data set, which is independent on the number of repetitions. The blue curve is a linear fitting of all the measurement points.}
    \label{fig:3species:time}
\end{minipage}
\end{figure}

\subsection{Implementation and reproducibility}

The \verb+treex+ library for Python \citep{azais2019treex} is designed to manipulate rooted trees, with a lot of diversity (ordered or not, labeled or not). It offers options for random generation, visualization, edit operations, conversions to other formats, and various algorithms. We implemented the subtree kernel as a module of \verb+treex+ so that the interested reader can manipulate the concepts discussed in this paper in a ready-to-use manner.

Basically, the \verb+subtree_kernel+ module allows the computation of formula (\ref{eq:K:subtrees}) with options for choosing (i) $\kappa$ among some classic choices of kernels \citep[Section 2.3]{Scholkopf:2001:LKS:559923} and (ii) the weight function among exponential decay or discriminance. Resorting to dependencies to \verb+scikit-learn+, tools for processing databases and compute SVM are also provided for the sake of self-containedness. Finally, visualization tools are also made available to perform the comprehensive learning approach discussed in Subsection~\ref{ss:wikipedia}.

Installing instructions and the documentation of \verb+treex+ can be found from \cite{azais2019treex}. For the sake of reproducibility, the databases used in Section~\ref{sec:data}, as well as the scripts that were designed to create them and process them, can be made available upon request to the authors.

\acks{This work has been supported by the European Union's H2020 project ROMI. The authors would like to thank Dr. Bruno Leggio as well as Dr. Giovanni Da San Martino for helping them to access some of the data sets presented in the paper: \textit{grazie}. Last but not least, the authors are grateful to three anonymous reviewers for their valuable comments on a first version of the manuscript.}

\appendix

\section{Proof of Proposition~\ref{prop:ker:sep}}
\label{app:proof:ker:sep}

The proof is mainly based on the following technical lemma, which statement requires the following notation. If $u$ is a vertex of a tree $T$, $\fami(u)$ denotes the family of $u$, i.e., the set composed of the ascendants of $u$, $u$, and the descendants of $u$ in $T$. We recall that $\des(u)$ stands for the set of descendants of $u$.

\begin{lemma}\label{lem:formula:delta}
Let $u,v\in T_i$, $i\in\{1,2\}$. One has
$$
K(T_i^u , T_i^v) = K(T_i,T_i) - \sum_{x\in \mathcal{B}_{u,v}}\omega_{T_i[x]} + K(\tau_{\height(u)},\tau_{\height(v)}),
$$
where
\begin{equation}\label{eq:def:B}
\mathcal{B}_{u,v} =
\left\{
\begin{array}{cl}
\des(u) \cup \{u\} & \text{if $u=v$,} \\
\fami(u) \cup \fami(v) & \text{else.}
\end{array}
\right.
\end{equation}

Let $u\in T_1$ and $v\in T_2$. Then,
$$K(T_1^u,T_2^v) = K(\tau_{\height(u)},\tau_{\height(v)}).$$
\end{lemma}
\begin{proof} We begin with the case $u\neq v$. The result relies on the following decomposition which is valid under the assumptions made on $T_i$ and the sequence $(\tau_h)$,
$$
\subtrees(T_i^u)\cap\subtrees(T_i^v) ~=~ \big[\subtrees(T_i)\setminus\{T_i[z]~:~z\in\fami(u)\cup\fami(v)\}\big]~ \cup~ \big[\subtrees(\tau_{\height(u)}) \cap \subtrees(\tau_{\height(v)})\big].
$$ 

Together with \eqref{eq:K:subtrees},
\begin{eqnarray*}
K(T_i^u,T_i^v) &=& \sum_{\theta\in\subtrees(T_i)\setminus\{T_i[z]~:~z\in\fami(u)\cup\fami(v)\}} w_\theta \num_\theta(T_i^u) \num_\theta(T_i^v) \\
&+& \sum_{\theta\in\subtrees(\tau_{\height(u)}) \cap \subtrees(\tau_{\height(v)})} w_\theta\num_\theta(T_i^u) \num_\theta(T_i^v).
\end{eqnarray*}

If $\theta\in\subtrees(\tau_{\height(u)}) \cap \subtrees(\tau_{\height(v)})$, then $\num_\theta(T_i^z) = \num_\theta(\tau_{\height(z)})$, $z\in\{u,v\}$, because, for any $h>0$, $\tau_h$ is not a subtree of $T_0$ nor $T_1$ by assumption. Thus,
\begin{eqnarray}\label{eq:calculint:Kvw1}
\sum_{\theta\in\subtrees(\tau_{\height(u)}) \cap \subtrees(\tau_{\height(v)})} w_\theta\num_\theta(T_i^u) \num_\theta(T_i^v) &=& \sum_{\theta\in\subtrees(\tau_{\height(u)}) \cap \subtrees(\tau_{\height(v)})} w_\theta\num_\theta(\tau_{\height(u)}) \num_\theta(\tau_{\height(v)}) \nonumber \\
&=& K(\tau_{\height(u)} , \tau_{\height(v)}),
\end{eqnarray}
in light of \eqref{eq:K:subtrees} again. Furthermore, if $\theta\in\subtrees(T_i)\setminus\{T_i[z]~:~z\in\fami(u)\cup\fami(v)\}$, then $\num_\theta(T_i^z) = \num_\theta(T_i)$, $z\in\{u,v\}$, and
\begin{eqnarray}
\sum_{\theta\in\subtrees(T_i)\setminus\{T_i[z]~:~z\in\fami(u)\cup\fami(v)\}} w_\theta \num_\theta(T_i^u) \num_\theta(T_i^v) &=& \sum_{\theta\in\subtrees(T_i)\setminus\{T_i[z]~:~z\in\fami(u)\cup\fami(v)\}} w_\theta \num_\theta(T_i) \num_\theta(T_i) \nonumber \\
&=& \sum_{\theta\in\subtrees(T_i)} w_\theta \num_\theta(T_i) \num_\theta(T_i) \nonumber  \\
&-& \sum_{\theta\in\{T_i[u]~:~u\in\fami(v)\cup\fami(w)\}} w_\theta\num_\theta(T_i) \num_\theta(T_i) \nonumber \\
&=& K(T_i,T_i) - \sum_{\theta\in\{T_i[z]~:~z\in\fami(u)\cup\fami(v)\}} w_\theta, \label{eq:calculint:Kvw2}
\end{eqnarray}
since $\num_\theta(T_i) = 1$ because of the first assumption on $T_i$. \eqref{eq:calculint:Kvw1} and \eqref{eq:calculint:Kvw2} state the first result. When $u=v$, the decomposition is slightly different,
$$
\subtrees(T_i^u) ~=~ \big[\subtrees(T_i)\setminus\{T_i[z]~:~z\in\{u\}\cup\des(u)\}\big]~ \cup~ \subtrees(\tau_{\height(u)}),
$$
but the rest of the proof is similar. Finally, the formula for $K(T_1^u,T_2^v)$ is a direct consequence of the third assumption on $T_1$, $T_2$ and the sequence $(\tau_h)$.\end{proof}

By virtue of the previous lemma, one can derive the following result on the quantity $\Delta_x^i$ defined by \eqref{eq:delta:def}.

\begin{lemma}\label{lem:delta:calcul}
Let $x\in T_i$, $i\in\{1,2\}$. One has
$$\Delta_x^i = K(T_i , T_i ) - \mathbb{E}_u\left[ \sum_{z \in \mathcal{B}_{x,u}}w_{T_i[z]}\right].$$
\end{lemma}
\begin{proof}
In light of Lemma~\ref{lem:formula:delta}, one has
$$\Delta_x^i = K(T_i , T_i ) - \mathbb{E}_u\left[ \sum_{z \in \mathcal{B}_{x,u}}w_{T_i[z]}\right] + \mathbb{E}_u\left[ K(\tau_{\height(x)} , \tau_{\height(u)})\right] - \mathbb{E}_v\left[ K(\tau_{\height(x)} , \tau_{\height(v)})\right] .$$
By assumption on the stochastic model of random trees, $\height(u)$ and $\height(v)$ have the same distribution and thus $ \mathbb{E}_u[ K(\tau_{\height(x)} , \tau_{\height(u)})] =  \mathbb{E}_v[ K(\tau_{\height(x)} , \tau_{\height(v)})]$, which states the expected result.
\end{proof}

The next decomposition is useful to prove the result of interest. If $c_h^i$ denotes the number of subtrees of height $h$ appearing in $T_i$, $h\geq0$, then the probability of picking a particular vertex $u$ is $P_\rho(\height(u))/c_{\height(u)}^i$ and thus
$$
	\mathbb{E}_u\left[ \sum_{z \in \mathcal{B}_{x,u}}w_{T_i[z]}\right] = \frac{P_\rho(\height(x))}{c_{\height(x)}^i} \sum_{z\in\{x\}\cup\des(x)} w_{T_i[z]} + \sum_{u\in T_i\setminus\{x\}}  \frac{P_\rho(\height(u))}{c_{\height(u)}^i} \sum_{z\in\mathcal{B}_{x,u}} w_{T_i[z]}.
$$

In addition, for $u\in T_i\setminus\{x\}$,
\begin{eqnarray}
\sum_{z\in\{x\}\cup\des(x)} \omega_{T_i[z]} &=& K(T_i[x],T_i[x]) \label{eq:proof:sum1}, \\
\sum_{z\in\mathcal{B}_{x,u}} \omega_{T_i[z]} &=& K(T_i,T_i) -  \sum_{z\notin\fami(x)\cup\fami(u)} \omega_{T_i[z]} . \label{eq:proof:sum2}
\end{eqnarray}
\eqref{eq:proof:sum1} and \eqref{eq:proof:sum2} together with Lemma~\ref{lem:delta:calcul} show that
$$ \Delta_x^i = \frac{P_\rho(\height(x))}{c_{\height(x)}^i} \left(K(T_i,T_i) - K(T_i[x],T_i[x]\right) + \sum_{u\in T_i\setminus\{x\}}  \frac{P_\rho(\height(u))}{c_{\height(u)}^i} \sum_{z\notin\fami(x)\cup\fami(u)} \omega_{T_i[z]}.$$

The left-hand term (and the right-hand term when $w_{T_i}>0$) is null if and only if $x=\roottree(T_i)$, which shows the first result. In addition,
$$ \Delta_x^i \geq \frac{P_\rho(\height(x))}{c_{\height(x)}^i}\left(K(T_i,T_i) - K(T_i[x],T_i[x]\right) ,$$
which states the expected formula \eqref{eq:delta:min} with $P_\rho(0)\leq P_\rho(\height(x))$ (true if $\rho>H/2$) and $c_{\height(x)}^i\leq\#\leaves(T_i)$. The conclusion comes from the fact that the probability of drawing a vertex $x$ of height greater 
than $h$ is $G_\rho(h)$.


\section{Proof of Proposition~\ref{prop:dag:recompression:complexity} }\label{app:proof:dag:recompression}

We denote by $D^h$ the set of vertices at height $h$ in any DAG $D$, and $\ast\in\lbrace \text{ordered},\text{unordered}\rbrace$ the type of isomorphism considered. From the forest $(D_1,\dots,D_N)$, we construct the DAG $\Delta$ such that (i) $D_i$ is a subDAG of $\Delta$ for all $i$, (ii) $\height(\Delta) = \max_i \height(D_i)$, (iii) all vertices in $\Delta$ have degree $\max_i \degr(D_i)$, and (iv) at each height except $0$ and $\height(D)$, $\# \Delta^h = \max_i \#D_i^h$. If $\Delta$ is placed $N$ times under an artificial root, and then recompressed by the algorithm, indeed the output contains the recompression of the original forest. Therefore, this case is the worst possible for the algorithm, and we claim that it achieves the proposed complexity.

Let $\Delta$ be now a DAG with following properties : $\# \Delta =m$, $\height(\Delta)=H$, at each height $h \notin \lbrace 0,H \rbrace$, $\#\Delta^h = n$ (so that $n(H-2)+2 = m$), and all vertices have degree $d$. $\sdag$ is the super-DAG obtained after placing $N$ copies of $\Delta$ under an artificial root. We then have $\#\sdag = 1+Nm$ so that $\O(\#\sdag)=\O(Nm)=\O(NHn)$ and $\degr(\mathcal{F}) = \degr(\Delta) = d$.

At the beginning of the algorithm, constructing the mapping $h\mapsto \sdag^h$ in one exploration of $\sdag$ has complexity $\O(\#\sdag)$. We will now examine the complexity of the further steps, with respect to $n,m,d,H$ and $N$. We introduce the following lemma :

\begin{lemma}\label{lemma:recompression:merge}
Constructing $\ntm(h)$ has time-complexity:
\begin{enumerate}
\item $\O\big(\sum_{\nu \in \sdag^h} \#\child(\nu)\log\#\child(\nu)\big)$ for unordered trees;
\item $\O\big(\sum_{\nu \in \sdag^h} \#\child(\nu)\big)$ for ordered trees.
\end{enumerate}
\end{lemma}
\begin{proof}
When sorting lists of size $L$, merge sort is known to have $\O(L\log L)$ complexity in the worst case \citep{skiena2012sorting}. Accordingly, we introduce

$$g^\ast(x) = \begin{cases} \hfil x & \text{if } \ast = \text{ordered};\\ x(1+\log x) & \text{if } \ast = \text{unordered}. \end{cases}$$

At height $h$, we construct $\ntm(h)=\lbrace f^{-1}(S)\,:\,S \in \text{Im}(f),\,\#f^{-1}(S) \geq 2\rbrace$ where $f : \nu \in \sdag^h \mapsto \child(\nu)$. Finding the preimage of $f$ requires first to construct $f$, by copying the children of each vertex in $\sdag^h$ (in the unordered case, we also need to sort them, so that we get rid of the order and can properly compare them). Then we only need to explore once the image and check whether an element has two or more antecedents. The global cost is then $\O(\sum_{\nu \in \sdag^h} g^\ast(\#\child(\nu)))$. 
\end{proof}

We reuse the notation $g^\ast$ from the proof of Lemma~\ref{lemma:recompression:merge}. With respect to $\Delta$, the complexity for constructing $\ntm(\cdot)$ is $\O(Nng^\ast(d))$. Exploring the elements of $\ntm(h)$ for (i) choosing a vertex $\nu_M$ to remain, and (ii) delete the other elements $\delta_M$ has complexity $\O(Nn)$. In addition, at height $h'>h$, exploring the children to replace them or not costs $\O(\sum_{\nu \in \sdag^{h'}} \#\child(\nu))=\O(Ndn)$.

The global complexity $C(\sdag)$ of the algorithm is then
$$C(\sdag) =\O(\#\sdag) + \sum_{h=0}^{\height(\Delta)} \O(Nng^\ast(d)) +\O(Nn)+ \sum_{h' >h} \O(Ndn).$$

Remark that $\sum_{h=0}^{\height(D)} \O(Nn) = \O(Nm)=\O(\#\sdag)$, this leads to
$$C(\sdag)=\O(\#\sdag g^\ast(\degr(\mathcal{F})))+ \O\Bigg( Ndn \sum_{h=0}^{\height(D)}\sum_{h'>h} 1\Bigg).$$

The right-hand inner sum is in $\O(H^2)$. As $$\O(NdnH^2) = \O(\#\sdag Hd) =\O(\#\sdag \height(\sdag) \degr(\mathcal{F})),$$ this leads to our statement.

\vskip 0.2in
\bibliography{main}

\end{document}